\documentclass{article}
 \pdfoutput=1

\usepackage{iclr2021_conference,times}

\usepackage{amsmath,amsfonts,bm}

\def\eqref#1{equation~\ref{#1}}

\def\1{\bm{1}}

\def\va{{\bm{a}}}
\def\vb{{\bm{b}}}
\def\vc{{\bm{c}}}

\def\vw{{\bm{w}}}
\def\vx{{\bm{x}}}
\def\vy{{\bm{y}}}

\DeclareMathAlphabet{\mathsfit}{\encodingdefault}{\sfdefault}{m}{sl}
\SetMathAlphabet{\mathsfit}{bold}{\encodingdefault}{\sfdefault}{bx}{n}

\newcommand{\E}{\mathbb{E}}

\newcommand{\R}{\mathbb{R}}

\iclrfinalcopy

\usepackage{enumitem}
\usepackage{todonotes}
\usepackage[utf8]{inputenc} 
\usepackage[T1]{fontenc}    
\usepackage{hyperref}       
\usepackage{url}            
\usepackage{booktabs}       
\usepackage{amsfonts}       
\usepackage{nicefrac}       
\usepackage{microtype}      
\usepackage{amsmath}
\usepackage{xcolor}
\usepackage{float}
\usepackage{amssymb}
\usepackage{amsthm}
\usepackage{subfigure}
\usepackage{graphicx}
\usepackage{verbatim}
\usepackage{multirow}
\usepackage{multicol}

\usepackage{titlesec}
\titlespacing\section{0pt}{0pt plus 0pt minus 0pt}{0pt plus 0pt minus 0pt}
\titlespacing\subsection{0pt}{0pt plus 0pt minus 0pt}{0pt plus 0pt minus 0pt}
\titlespacing\subsubsection{0pt}{0pt plus 0pt minus 0pt}{0pt plus 0pt minus 0pt}

\usepackage[skip=0pt]{caption}
\usepackage[ruled]{algorithm2e}
\usepackage{amsmath, amsthm, thm-restate}
\usepackage{chngcntr}

\newtheorem{definition}{Definition}
\newtheorem{lemma}{Lemma}

\newtheorem{remark}{Remark}

\def\*#1{\mathbf{#1}}
\newcommand{\bs}[1]{\ensuremath{\boldsymbol{#1}}}

\newtheoremstyle{slanted}
  {0.9\topsep}
  {0.9\topsep}
  {\it}
  {}
  {\bfseries}
  {.}
  {0.5em}
  {}

\theoremstyle{slanted}
\newtheorem{theorem}{Theorem}

\newcommand\blfootnote[1]{%
  \begingroup
  \renewcommand\thefootnote{}\footnote{#1}%
  \addtocounter{footnote}{-1}%
  \endgroup
}

\title{Rethinking Attention with Performers}

\author{Krzysztof Choromanski$^{*1}$, Valerii Likhosherstov$^{*2}$, David Dohan$^{*1}$, Xingyou Song$^{*1}$ \\ 
  \textbf{Andreea Gane$^{*1}$,  Tamas Sarlos$^{*1}$,  Peter Hawkins$^{*1}$, Jared Davis$^{*3}$, Afroz Mohiuddin$^{1}$} \\
  \textbf{Lukasz Kaiser$^{1}$, David Belanger$^{1}$, Lucy Colwell$^{1,2}$, Adrian Weller$^{2,4}$} \\
  $^1$Google $^2$University of Cambridge $^3$DeepMind $^4$Alan Turing Institute \\
}

\begin{document}
\maketitle
\vspace{-0.8cm}
\begin{abstract}
\vspace{-0.3cm}
We introduce \textit{Performers}, Transformer architectures which can estimate regular (softmax) full-rank-attention Transformers with provable accuracy, but using only linear (as opposed to quadratic) space and time complexity, without relying on any priors such as sparsity or low-rankness. To approximate softmax attention-kernels, Performers use a novel \textit{Fast Attention Via positive Orthogonal Random features} approach (FAVOR+), which may be of independent interest for scalable kernel methods. FAVOR+ can also be used to efficiently model kernelizable attention mechanisms beyond softmax. This representational power is crucial to accurately compare softmax with other kernels for the first time on large-scale tasks, beyond the reach of regular Transformers, and investigate optimal attention-kernels. Performers are linear architectures fully compatible with regular Transformers and with strong theoretical guarantees: unbiased or nearly-unbiased estimation of the attention matrix, uniform convergence and low  estimation variance. We tested Performers on a rich set of tasks stretching from pixel-prediction through text models to protein sequence modeling. We demonstrate competitive results with other examined efficient sparse and dense attention methods, showcasing effectiveness of the novel attention-learning paradigm leveraged by Performers. 

\end{abstract}

\blfootnote{$^\ast$Equal contribution. Correspondence to \texttt{\{kchoro,lcolwell\}@google.com}.}

\blfootnote{Code for Transformer models on protein data can be found in \textcolor{blue}{\url{github.com/google-research/google-research/tree/master/protein_lm}} and Performer code can be found in \textcolor{blue}{\url{github.com/google-research/google-research/tree/master/performer}}. Google AI Blog: \textcolor{blue}{\url{https://ai.googleblog.com/2020/10/rethinking-attention-with-performers.html}}}
\vspace{-1.5cm}
\section{Introduction and related work}
\label{sec:intro_related_work}

Transformers \citep{transformer, universal_t} are powerful neural network architectures that have become SOTA in several areas of machine learning including natural language processing (NLP) (e.g. speech recognition \citep{luo}), neural machine translation (NMT) \citep{nmt}, document generation/summarization, time series prediction, generative modeling (e.g. image generation \citep{parmar}), music generation \citep{simon}, and bioinformatics \citep{rives, progen, ingraham2019generative, elnaggar2019end, du2020energy}. 

Transformers rely on a trainable \textit{attention} mechanism that identifies complex dependencies between the elements of each input sequence. Unfortunately, the regular Transformer scales quadratically with the number of tokens $L$ in the input sequence, which is prohibitively expensive for large $L$ and precludes its usage in settings with limited computational resources even for moderate values of $L$. 
Several solutions have been proposed to address this issue \citep{longformer, conformer, imputer, sparsetr, image_transformer}. Most approaches restrict the attention mechanism to attend to local neighborhoods \citep{parmar} or incorporate structural priors on attention such as sparsity \citep{sparsetr}, pooling-based compression \citep{compr} clustering/binning/convolution techniques (e.g. \citep{routing_t} which applies $k$-means clustering to learn dynamic sparse attention regions, or \citep{reformer}, where locality sensitive hashing is used to group together tokens of similar embeddings), sliding windows \citep{longformer}, or truncated targeting \citep{chelba}. 
There is also a long line of research on using dense attention matrices, but defined by low-rank kernels
substituting softmax \citep{trans-rnns, shen}. Those methods critically rely on kernels admitting explicit representations as dot-products of finite positive-feature vectors.

The approaches above do not aim to approximate regular attention, but rather propose simpler and more tractable attention mechanisms, often by incorporating additional constraints (e.g. identical query and key sets as in \citep{reformer}), or by trading regular with sparse attention using more layers \citep{sparsetr}. Unfortunately, there is a lack of rigorous guarantees for the representation power produced by such methods, and sometimes the validity of sparsity patterns can only be verified empirically through trial and error by constructing special GPU operations (e.g. either writing C++ CUDA kernels \citep{sparsetr} or using TVMs \citep{longformer}). Other techniques 
which aim to reduce 
Transformers' space complexity include reversible residual layers allowing one-time activation storage in training \citep{reformer} and shared attention weights \citep{shared_weights}. These constraints may impede application 
to long-sequence problems, where approximations of the attention mechanism are not sufficient. Approximations based on truncated back-propagation \citep{transformerxl} are also unable to capture long-distance correlations since the gradients are only propagated inside a localized window.
Other methods propose biased estimation of 
regular attention but only in the non-causal setting and with large mean squared error \citep{linformer}.


In response, we introduce the first Transformer architectures, \textit{Performers}, capable of \textbf{provably} accurate and practical estimation of regular (softmax) full-rank attention, but of only linear space and time complexity and \textbf{not relying on any priors} such as sparsity or low-rankness. Performers use the \textit{Fast Attention Via positive Orthogonal Random features} (FAVOR+) mechanism, leveraging new methods for approximating softmax and Gaussian kernels, which we propose. We believe these methods are of independent interest, contributing to the theory of scalable kernel methods. 
Consequently, Performers are the first linear architectures \textbf{fully compatible} (via small amounts of fine-tuning) with regular Transformers, providing strong theoretical guarantees: unbiased or nearly-unbiased estimation of the attention matrix, uniform convergence and lower variance of the approximation.

FAVOR+ can be also applied to efficiently model other kernelizable attention mechanisms beyond softmax. This representational power is crucial to accurately compare softmax with other kernels for the first time on large-scale tasks, that are beyond the reach of regular Transformers, and find for them optimal attention-kernels.
FAVOR+ can also be applied beyond the Transformer scope as a more scalable replacement for regular attention, which itself has a wide variety of uses in computer vision \citep{attention_cvpr}, reinforcement learning \citep{relational}, training with softmax cross entropy loss, and even combinatorial optimization \citep{pointer}.

We test Performers on a rich set of tasks ranging from pixel-prediction through text models to protein sequence modeling. We demonstrate competitive results with other examined efficient sparse and dense attention methods, showcasing the effectiveness of the novel attention-learning paradigm leveraged by Performers. We emphasize that in principle, FAVOR+ can also be combined with other techniques, such as reversible layers \citep{reformer} or cluster-based attention \citep{routing_t}.

\section{FAVOR+ Mechanism \& Positive Orthogonal Random Features}
\label{sec:algorithm}
Below we describe in detail the FAVOR+ mechanism - the backbone of the $\mathrm{Performer's}$ architecture. We introduce a new method for estimating softmax (and Gaussian) kernels with $\textbf{positive}$ orthogonal random features which FAVOR+ leverages for the robust and unbiased estimation of regular (softmax) attention and show how FAVOR+ can be applied for other attention-kernels.

\subsection{Preliminaries - regular attention mechanism}

Let $L$ be the size of an input sequence of tokens. Then regular dot-product attention \citep{transformer} is a mapping which accepts matrices $\mathbf{Q}, \mathbf{K}, \mathbf{V} \in \mathbb{R}^{L \times d}$ as input where $d$ is the hidden dimension (dimension of the latent representation). Matrices $\mathbf{Q}, \mathbf{K}, \mathbf{V}$ are intermediate representations of the input and their rows can be interpreted as \textit{queries}, \textit{keys} and \textit{values} of the continuous dictionary data structure respectively. \textit{Bidirectional (or non-directional \citep{bert}) dot-product attention} has the following form, where  $\mathbf{A} \in \mathbb{R}^{L \times L}$ is the so-called \textit{attention matrix}:
\begin{equation}
\label{eq:attnorm}
    \mathrm{Att}_\leftrightarrow (\mathbf{Q}, \mathbf{K}, \mathbf{V}) = \mathbf{D}^{-1} \mathbf{A} \mathbf{V}, \quad 
    \mathbf{A} = \exp ( \mathbf{Q} \mathbf{K}^\top / \sqrt{d}), \quad \mathbf{D} = \mathrm{diag} ( \mathbf{A} \mathbf{1}_L ). 
\end{equation}    
Here $\exp (\cdot)$ is applied elementwise, $\mathbf{1}_L$ is the all-ones vector of length $L$, and $\mathrm{diag} (\cdot)$ is a diagonal matrix with the input vector as the diagonal. Time and space complexity of computing (\ref{eq:attnorm}) are $O(L^2 d)$ and $O(L^{2}+Ld)$ respectively, because $\mathbf{A}$ has to be stored explicitly. Hence, in principle, dot-product attention of type (\ref{eq:attnorm}) is incompatible with end-to-end processing of long sequences. Bidirectional attention is applied in encoder self-attention and encoder-decoder attention in Seq2Seq architectures.

Another important type of attention is \textit{unidirectional dot-product attention} which has the form:
\begin{equation}
    \mathrm{Att}_\to (\mathbf{Q}, \mathbf{K}, \mathbf{V}) = \widetilde{\mathbf{D}}^{-1} \widetilde{\mathbf{A}} \mathbf{V}, \quad 
    \widetilde{\mathbf{A}} = \mathrm{tril} (\mathbf{A}), \quad \widetilde{\mathbf{D}} = \mathrm{diag} ( \widetilde{\mathbf{A}} \mathbf{1}_L ) , \label{eq:uattnorm}
\end{equation}
where $\mathrm{tril}(\cdot)$ returns the lower-triangular part of the argument matrix including the diagonal. As discussed in \citep{transformer}, unidirectional attention is used for autoregressive generative modelling, e.g. as self-attention in generative Transformers as well as the decoder part of Seq2Seq Transformers.

We will show that attention matrix $\mathbf{A}$ can be approximated up to any precision in time $O (L d^2\log(d))$. For comparison, popular methods leveraging sparsity via Locality-Sensitive Hashing (LSH) techniques \citep{reformer} have $O(L d^2 \log L)$ time complexity. 
In the main body of the paper we will describe FAVOR+ for bidirectional attention. Completely analogous results can be obtained for the unidirectional variant via the mechanism of \textit{prefix-sums} (all details in the Appendix \ref{subsec:unidirectional}).

\subsection{Generalized Kernelizable Attention}
\label{sec:gka}

FAVOR+ works for attention blocks using matrices $\mathbf{A} \in \mathbb{R}^{L \times L}$ of the form $\mathbf{A}(i,j) = \mathrm{K}(\mathbf{q}_{i}^{\top},\mathbf{k}_{j}^{\top})$, with $\mathbf{q}_{i}/\mathbf{k}_{j}$ standing for the $i^{th}/j^{th}$ query/key row-vector in $\mathbf{Q}/\mathbf{K}$ and kernel $\mathrm{K}:\mathbb{R}^{d} \times \mathbb{R}^{d} \rightarrow \mathbb{R}_{+}$ defined for the (usually randomized) mapping: $\phi: \mathbb{R}^{d} \rightarrow \mathbb{R}_{+}^{r}$ (for some $r >0$) as:
\begin{equation}
\label{kernel-def}
\mathrm{K}(\mathbf{x}, \mathbf{y}) = \mathbb{E}[\phi(\mathbf{x})^{\top}\phi(\mathbf{y})].
\end{equation}
We call $\phi(\mathbf{u})$ a \textit{random feature map} for $\mathbf{u} \in \mathbb{R}^{d}$. 
For $\mathbf{Q}^{\prime},\mathbf{K}^{\prime} \in \mathbb{R}^{L \times r}$ with rows given as $\phi(\mathbf{q}_{i}^{\top})^{\top}$ and $\phi(\mathbf{k}_{i}^{\top})^{\top}$ respectively,
Equation \ref{kernel-def} leads directly to the efficient attention mechanism of the form:
\begin{equation}
    \widehat{\mathrm{Att}_\leftrightarrow} (\mathbf{Q}, \mathbf{K}, \mathbf{V}) = \widehat{\mathbf{D}}^{-1} (\mathbf{Q}^{\prime}((\mathbf{K}^{\prime})^{\top} \mathbf{V})), \quad 
    \quad \widehat{\mathbf{D}} = \mathrm{diag} (\mathbf{Q}^{\prime}((\mathbf{K}^{\prime})^{\top} \mathbf{1}_L) ). 
\end{equation} 
Here $\widehat{\mathrm{Att}_\leftrightarrow}$ stands for the approximate attention and brackets indicate the order of computations. It is easy to see that such a mechanism is characterized by space complexity $O(Lr + Ld + rd)$ and time complexity $O(Lrd)$ as opposed to $O(L^{2}+Ld)$ and $O(L^{2}d)$ of the regular attention (see also Fig. \ref{fig:avqk}).
\vspace{-2mm}
\begin{figure}[h]
  \centering
  \includegraphics[width=0.99\textwidth]{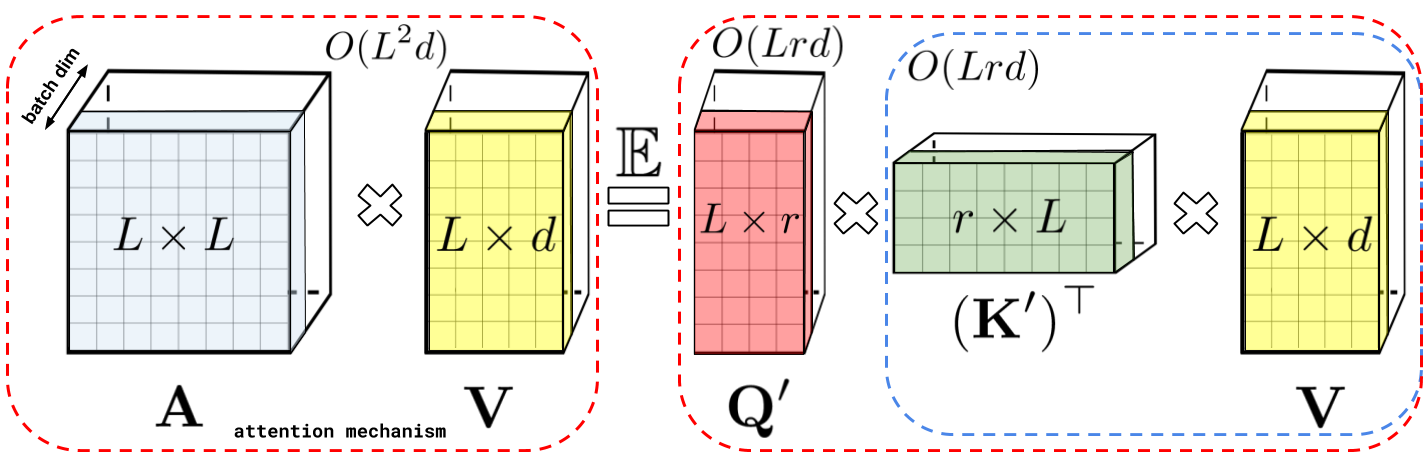}
  \caption{\small{Approximation of the regular attention mechanism $\mathbf{AV}$ (before $\mathbf{D}^{-1}$-renormalization) via (random) feature maps. Dashed-blocks indicate order of computation with corresponding time complexities attached.}}
  \label{fig:avqk}
\end{figure}

The above scheme constitutes the FA-part of the FAVOR+ mechanism. The remaining OR+ part answers the following questions: \textbf{(1)} How expressive is the attention model defined in Equation \ref{kernel-def}, and in particular, can we use it in principle to approximate regular softmax attention ? \textbf{(2)} How do we implement it robustly in practice, and in particular, can we choose $r \ll L$ for $L \gg d$ to obtain desired space and time complexity gains? We answer these questions in the next sections.

\subsection{How to and how not to approximate softmax-kernels for Attention}
\label{howto}

It turns out that by taking $\phi$ of the following form
for functions $f_{1},...,f_{l}:\mathbb{R} \rightarrow \mathbb{R}$,
function $g:\mathbb{R}^{d} \rightarrow \mathbb{R}$ and deterministic vectors $\omega_{i}$ or $\omega_{1},...,\omega_{m} \overset{\mathrm{iid}}{\sim} \mathcal{D}$ for some distribution $\mathcal{D} \in \mathcal{P}(\mathbb{R}^{d})$:
\begin{equation} \label{eq:feature}
\phi(\mathbf{x}) = \frac{h(\mathbf{x})}{\sqrt{m}}
(f_{1}(\omega_{1}^{\top}\mathbf{x}),...,f_{1}(\omega_{m}^{\top}\mathbf{x}),...,f_{l}(\omega_{1}^{\top}\mathbf{x}),...,f_{l}(\omega_{m}^{\top}\mathbf{x})),    
\end{equation}

we can model most kernels used in practice. Furthermore, in most cases $\mathcal{D}$ is isotropic (i.e. with pdf function constant on a sphere), usually Gaussian. For example, by taking $h(\mathbf{x})=1$, $l=1$ and $\mathcal{D}=\mathcal{N}(0,\mathbf{I}_{d})$ we obtain estimators of the so-called PNG-kernels \citep{unreas} (e.g. $f_{1} = \mathrm{sgn}$ corresponds to the angular kernel).
Configurations: $h(\mathbf{x})=1$, $l=2$, $f_{1}=\sin$, $f_{2}=\cos$ correspond to shift-invariant kernels, in particular $\mathcal{D}=\mathcal{N}(0, \mathbf{I}_{d})$ leads to the Gaussian kernel $\mathrm{K}_{\mathrm{gauss}}$ \citep{fourierapprox}. The \textit{softmax-kernel} which defines regular attention matrix $\mathbf{A}$ is given as:
\vspace{-0.75mm}
\begin{equation}
\mathrm{SM}(\mathbf{x}, \mathbf{y}) \overset{\mathrm{def}}{=} \exp(\mathbf{x}^{\top}\mathbf{y}).    
\end{equation}
In the above, without loss of generality, we omit $\sqrt{d}$-renormalization since we can equivalently renormalize input keys and queries. Since: $\mathrm{SM}(\mathbf{x}, \mathbf{y}) = \exp(\frac{\|\mathbf{x}\|^{2}}{2})\mathrm{K}_{\mathrm{gauss}}(\mathbf{x},\mathbf{y})\exp(\frac{\|\mathbf{y}\|^{2}}{2})$, based on what we have said, we obtain random feature map unbiased approximation of $\mathrm{SM}(\mathbf{x}, \mathbf{y})$ using trigonometric functions with: $h(\mathbf{x})=\exp(\frac{\|\mathbf{x}\|^{2}}{2})$, $l=2$, $f_{1}=\sin$, $f_{2}=\cos$. We call it $\widehat{\mathrm{SM}}_{m}^{\mathrm{trig}}(\mathbf{x}, \mathbf{y})$. 

There is however a caveat there. The attention module from (\ref{eq:attnorm}) constructs for each token, a convex combination of value-vectors with coefficients given as corresponding renormalized kernel scores. That is why kernels producing non-negative scores are used. Applying random feature maps with potentially negative dimension-values ($\sin/\cos$) leads to unstable behaviours, especially when kernel scores close to $0$ (which is the case for many 
entries of $\mathbf{A}$ corresponding to low relevance tokens) are approximated by estimators with large variance in such regions. This results in abnormal behaviours, e.g. negative-diagonal-values renormalizers $\mathbf{D}^{-1}$, and consequently either completely prevents training or leads to sub-optimal models. 
We demonstrate empirically that this is what happens for $\widehat{\mathrm{SM}}_{m}^{\mathrm{trig}}$ and provide detailed theoretical explanations showing that the variance of $\widehat{\mathrm{SM}}_{m}^{\mathrm{trig}}$ is large as approximated values tend to $0$ (see: Section \ref{sec:theory}). This is one of the main reasons why the robust random feature map mechanism for approximating regular softmax attention was never proposed.

We propose a robust mechanism in this paper. Furthermore, the variance of our new unbiased positive random feature map estimator tends to $0$ as approximated values tend to $0$ (see: Section \ref{sec:theory}).

\begin{lemma}[Positive Random Features (PRFs) for Softmax]
\label{pos_random_features_lemma}
For $\mathbf{x}, \mathbf{y} \in \mathbb{R}^{d}$, $\mathbf{z}=\mathbf{x}+\mathbf{y}$ we have:
\begin{equation}
\label{main_eq}
\mathrm{SM}(\mathbf{x},\mathbf{y}) = 
\mathbb{E}_{\omega \sim \mathcal{N}(0,\mathbf{I}_{d})}\!\Big[\mathrm{exp}\!\Big(\omega^{\top}\mathbf{x} - \frac{\|\mathbf{x}\|^{2}}{2}\Big)\mathrm{exp}\!\Big(\omega^{\top}\mathbf{y}-\frac{\|\mathbf{y}\|^{2}}{2}\Big)\Big]
=\Lambda \mathbb{E}_{\omega \sim \mathcal{N}(0,\mathbf{I}_{d})}\cosh(\omega^{\top}\mathbf{z}),
\end{equation}
where $\Lambda=\exp(-\frac{\|\mathbf{x}\|^{2}+\|\mathbf{y}\|^{2}}{2})$ and $\cosh$ is hyperbolic cosine. Consequently, softmax-kernel admits a positive random feature map unbiased approximation with  $h(\mathbf{x})=\exp(-\frac{\|\mathbf{x}\|^{2}}{2})$, $l=1$, $f_{1}=\exp$ and $\mathcal{D}=\mathcal{N}(0, \mathbf{I}_{d})$ or: $h(\mathbf{x})=\frac{1}{\sqrt{2}}\exp(-\frac{\|\mathbf{x}\|^{2}}{2})$, $l=2$, $f_{1}(u)=\exp(u)$, $f_{2}(u)=\mathrm{exp}(-u)$ and the same $\mathcal{D}$ (the latter for further variance reduction).
We call related estimators: $\widehat{\mathrm{SM}}^{+}_{m}$ and $\widehat{\mathrm{SM}}^{\mathrm{hyp+}}_{m}$.
\end{lemma}
\begin{figure}[h]
  \centering
  \includegraphics[width=0.99\textwidth]{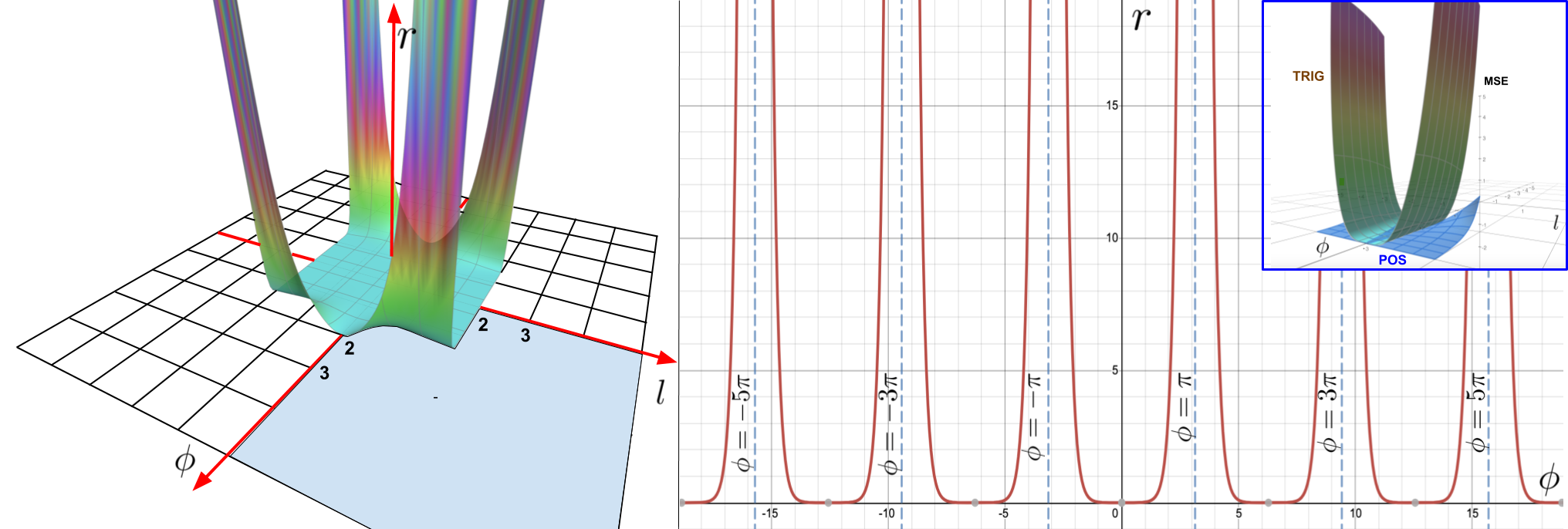}
  \caption{\small{\textbf{Left:} Symmetrized (around origin) utility function $r$ (defined as the ratio of the mean squared errors (MSEs) of estimators built on: trigonometric and positive random features) as a function of the angle $\phi$ (in radians) between input feature vectors and their lengths $l$. Larger values indicate regions of $(\phi,l)$-space with better performance of positive random features. We see that for critical regions with $\phi$ large enough (small enough softmax-kernel values) our method is arbitrarily more accurate than trigonometric random features. Plot presented for domain $[-\pi, \pi] \times [-2, 2]$.} \textbf{Right:} The slice of function $r$ for fixed $l=1$ and varying angle $\phi$. \textbf{Right Upper Corner:} Comparison of the MSEs of both the estimators in a low softmax-kernel value region.\\}
  \label{fig:tower-function}
\end{figure}
In Fig. \ref{fig:tower-function} we visualize the advantages of positive versus standard trigonometric random features. In critical regions, where kernel values are small and need careful approximation, our method outperforms its counterpart. In Section \ref{sec:experiments} we further confirm our method's advantages empirically, using positive features to efficiently train softmax-based linear Transformers.
If we replace in (\ref{main_eq}) $\omega$ with $\sqrt{d}\frac{\omega}{\|\omega\|}$, we obtain the so-called \textbf{regularized softmax-kernel} $\mathrm{SMREG}$ which we can approximate in a similar manner, simply changing $\mathcal{D}=\mathcal{N}(0,\mathbf{I}_{d})$ to $\mathcal{D} = \mathrm{Unif}(\sqrt{d}\mathcal{S}^{d-1})$, a distribution corresponding to Haar measure on the sphere of radius $\sqrt{d}$ in $\mathbb{R}^{d}$, obtaining estimator $\widehat{\mathrm{SMREG}}^{+}_{m}$. As we show in Section \ref{sec:theory}, such random features can also be used to accurately approximate regular softmax-kernel.

\subsection{Orthogonal Random Features (ORFs)}
The above constitutes the R+ part of the FAVOR+ method. It remains to explain the O-part. To further reduce the variance of the estimator (so that we can use an even smaller number of random features $r$), we entangle different random samples $\omega_{1},...,\omega_{m}$ to be \textbf{exactly} orthogonal. This can be done while maintaining unbiasedness whenever isotropic distributions $\mathcal{D}$ are used (i.e. in particular in all kernels we considered so far) by the standard Gram-Schmidt orthogonalization procedure (see \citep{unreas} for details). ORFs is a well-known method, yet it turns out that it works particularly well with our introduced PRFs for softmax. This leads to the \textbf{first theoretical results} showing that ORFs can be applied to reduce the variance of softmax/Gaussian kernel estimators \textbf{for any} dimensionality $d$ rather than just asymptotically for large enough $d$ (as is the case for previous methods, see: next section) and leads to the \textbf{first exponentially small bounds} on large deviations probabilities that are strictly smaller than for non-orthogonal methods. Positivity of random features plays a key role in these bounds. The ORF mechanism requires $m \leq d$, but this will be the case in all our experiments. The pseudocode of the entire FAVOR+ algorithm is given in Appendix \ref{appendix:main_algorithm}.

Our theoretical results are tightly aligned with experiments. We show in Section \ref{sec:experiments} that PRFs+ORFs drastically improve accuracy of the approximation of the attention matrix and enable us to reduce $r$ which results in an accurate as well as space and time efficient mechanism which we call FAVOR+.

\section{Theoretical results}
\label{sec:theory}

We present here the theory of positive orthogonal random features for softmax-kernel estimation. All these results can be applied also to the Gaussian kernel, since as explained in the previous section, one can be obtained from the other by renormalization (see: Section \ref{howto}).
All proofs and additional more general theoretical results with a discussion
are given in the Appendix.

\begin{lemma}[positive (hyperbolic) versus trigonometric random features]
\label{mse-lemma}
The following is true:
\begin{align}
\begin{split}
\mathrm{MSE}(\widehat{\mathrm{SM}}^{\mathrm{trig}}_{m}(\mathbf{x}, \mathbf{y})) =
\frac{1}{2m} \exp(\|\mathbf{x}+\mathbf{y}\|^{2})\mathrm{SM}^{-2}(\mathbf{x},\mathbf{y})
(1-\mathrm{exp}(-\|\mathbf{x}-\mathbf{y}\|^{2}))^{2}, \\ 
\mathrm{MSE}(\widehat{\mathrm{SM}}^{\mathrm{+}}_{m}(\mathbf{x}, \mathbf{y})) = \frac{1}{m}\exp(\|\mathbf{x}+\mathbf{y}\|^{2})\mathrm{SM}^{2}(\mathbf{x},\mathbf{y})(1-\exp(-\|\mathbf{x}+\mathbf{y}\|^{2})),\\
\mathrm{MSE}(\widehat{\mathrm{SM}}^{\mathrm{hyp}+}_{m}(\mathbf{x}, \mathbf{y})) =
\frac{1}{2}(1-\exp(-\|\mathbf{x}+\mathbf{y}\|^{2}))\mathrm{MSE}(\widehat{\mathrm{SM}}^{\mathrm{+}}_{m}(\mathbf{x}, \mathbf{y})),
\end{split}
\end{align}
for independent random samples $\omega_{i}$, and where $\mathrm{MSE}$ stands for the mean squared error.
\end{lemma}
Thus, for $\mathrm{SM}(\mathbf{x}, \mathbf{y}) \rightarrow 0$ we have: $\mathrm{MSE}(\widehat{\mathrm{SM}}^{\mathrm{trig}}_{m}(\mathbf{x}, \mathbf{y})) \rightarrow \infty$ and $\mathrm{MSE}(\widehat{\mathrm{SM}}^{\mathrm{+}}_{m}(\mathbf{x}, \mathbf{y})) \rightarrow 0$. Furthermore, the hyperbolic estimator provides additional accuracy improvements that are strictly better than those from $\widehat{\mathrm{SM}}_{2m}^{\mathrm{+}}(\mathbf{x}, \mathbf{y})$ with twice as many random features. The next result shows that the regularized softmax-kernel is in practice an accurate proxy of the softmax-kernel in attention.

\begin{theorem}[regularized versus softmax-kernel]
\label{reg-theorem}
Assume that the $L_{\infty}$-norm of the attention matrix for the softmax-kernel satisfies: $\|\mathbf{A}\|_{\infty} \leq C$ for some constant $C \geq 1$. Denote by $\mathbf{A}^{\mathrm{reg}}$ the corresponding attention matrix for the regularized softmax-kernel. The following holds:
\begin{equation}
\inf_{i,j} \frac{\mathbf{A}^{\mathrm{reg}}(i,j)}{\mathbf{A}(i, j)} \geq 1 - \frac{2}{d^{\frac{1}{3}}} + o\left(\frac{1}{d^{\frac{1}{3}}}\right), \textrm{ and } \sup_{i,j} \frac{\mathbf{A}^{\mathrm{reg}}(i,j)}{\mathbf{A}(i, j)} \leq 1. 
\end{equation}
Furthermore, the latter holds for $d \geq 2$ even if the $L_{\infty}$-norm condition is not satisfied, i.e. the regularized softmax-kernel is a universal lower bound for the softmax-kernel. 
\end{theorem}
Consequently, positive random features for $\mathrm{SMREG}$ can be used to approximate the softmax-kernel.
Our next result shows that orthogonality provably reduces mean squared error of the estimation with positive random features \textbf{for any dimensionality $d>0$} and we explicitly provide the gap.
\begin{theorem}
\label{var-theorem}
If $\widehat{\mathrm{SM}}_{m}^{\mathrm{ort+}}(\mathbf{x}, \mathbf{y})$ stands for the modification of
$\widehat{\mathrm{SM}}_{m}^{+}(\mathbf{x}, \mathbf{y})$ with orthogonal random features (and thus for $m \leq d$), then the following holds for any $d >0$:
\begin{equation}
\mathrm{MSE}(\widehat{\mathrm{SM}}_{m}^{\mathrm{ort+}}(\mathbf{x}, \mathbf{y})) \leq 
\mathrm{MSE}(\widehat{\mathrm{SM}}_{m}^{+}(\mathbf{x}, \mathbf{y})) - \frac{2 (m - 1)}{m(d+2)}\left(\mathrm{SM}(\mathbf{x}, \mathbf{y}) - \exp\left(- \frac{\| \*x \|^2 + \| \*y \|^2}{2}\right) \right)^2.
\end{equation}
Furthermore, completely analogous result holds for the regularized softmax-kernel $\mathrm{SMREG}$.
\end{theorem}
For the regularized softmax-kernel, orthogonal features provide additional concentration results - the first exponentially small bounds for probabilities of estimators' tails that are strictly better than for non-orthogonal variants for every $d>0$. Our next result enables us to explicitly estimate the gap.

\begin{theorem}
\label{ort-theorem}
Let $\mathbf{x}, \mathbf{y} \in \mathbb{R}^{d}$. The following holds for any $a > \mathrm{SMREG}(\mathbf{x},\mathbf{y})$, $\theta > 0$ and $m \leq d$:
\begin{gather*}
    \mathbb{P}[\widehat{\mathrm{SMREG}}^{+}_{m}(\mathbf{x}, \mathbf{y}) > a] \leq \exp(-\theta m a) M_Z (\theta)^m, \quad \mathbb{P}[\widehat{\mathrm{SMREG}}^{\mathrm{ort+}}_{m}(\mathbf{x}, \mathbf{y}) > a] \\
    \leq \exp(-\theta m a) \biggl( M_Z (\theta)^m - \exp\left(-\frac{m}{2} (\| \*x \|^2 + \| \*y \|^2)\right) \frac{\theta^4 m (m - 1)}{4 (d + 2)} \| \*x + \*y \|^4 \biggr)
\end{gather*}
where $\widehat{\mathrm{SMREG}}^{\mathrm{ort+}}_{m}(\mathbf{x}, \mathbf{y})$ stands for the modification of $\widehat{\mathrm{SMREG}}^{\mathrm{+}}_{m}(\mathbf{x}, \mathbf{y})$ with ORFs, $X = \Lambda \exp(\sqrt{d}\frac{\omega^{\top}}{\|\omega\|_{2}}(\mathbf{x}+\mathbf{y}))$, $\omega \sim \mathcal{N}(0,\mathbf{I}_{d})$, $\Lambda$ is as in Lemma \ref{pos_random_features_lemma} and
$M_{Z}$ is the moment generating function of $Z$.
\end{theorem}

We see that ORFs provide exponentially small and sharper bounds for critical regions where the softmax-kernel is small. 
Below we show that even for the $\mathrm{SM}^{\mathrm{trig}}$ mechanism with ORFs, it suffices to take $m=\Theta(d\log(d))$ random projections to accurately approximate the attention matrix (thus if not attention renormalization, PRFs would not be needed). In general, $m$ depends on the dimensionality $d$ of the embeddings, radius $R$ of the ball where all queries/keys live and precision parameter $\epsilon$ (see: Appendix~\ref{sec:ball} for additional discussion), but does not depend on input sequence length $L$.


\begin{theorem}[uniform convergence for attention approximation]
\label{thm:uniform}
Assume that $L_{2}$-norms of queries/keys are upper-bounded by $R>0$. Define $l=Rd^{-\frac{1}{4}}$ and take 
$h^{*} = \exp(\frac{l^{2}}{2})$.
Then for any $\epsilon>0$, $\delta=\frac{\epsilon}{(h^{*})^{2}}$
and the number of random projections $m=\Theta(\frac{d}{\delta^{2}}\log(\frac{4d^{\frac{3}{4}} R}{\delta}))$ the following holds for the attention approximation mechanism leveraging estimators $\widehat{\mathrm{SM}}^{\mathrm{trig}}$ with ORFs:
$\|\widehat{\mathbf{A}}-\mathbf{A}\|_{\infty} \leq \epsilon$ with any constant probability, where $\widehat{\mathbf{A}}$ approximates the attention matrix $\mathbf{A}$.
\end{theorem}

\vspace{-2.5mm}
\section{Experiments}
\label{sec:experiments}
We implemented our setup on top of pre-existing Transformer training code in Jax \citep{jax} optimized with just-in-time (\texttt{jax.jit}) compilation, and complement our theory with empirical evidence to demonstrate the practicality of FAVOR+ in multiple settings. Unless explicitly stated, a Performer replaces only the attention component with our method, while all other components are exactly the same as for the regular Transformer. For shorthand notation, we denote unidirectional/causal modelling as \textbf{(U)} and bidirectional/masked language modelling as \textbf{(B)}. 

In terms of baselines, we use other Transformer models for comparison, although some of them are restricted to only one case - e.g. Reformer \citep{reformer} is only (U), and Linformer \citep{linformer} is only (B). Furthermore, we use PG-19 \citep{compr} as an alternative (B) pretraining benchmark, as it is made for long-length sequence training compared to the (now publicly unavailable) BookCorpus \citep{bookcorpus} + Wikipedia dataset used in BERT \citep{bert} and Linformer. All model and tokenization hyperparameters are shown in Appendix \ref{appendix:hyperparameters}.

\begin{figure}[h]
  \centering
  \includegraphics[width=1.0\textwidth]{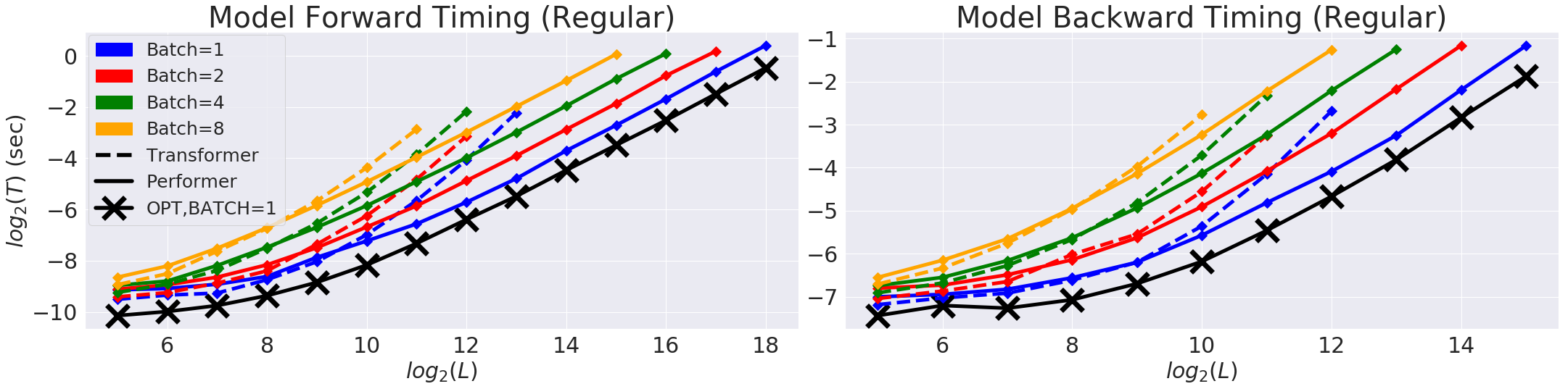}
  \caption{\small{Comparison of Transformer and Performer in terms of forward and backward pass speed and maximum $L$ allowed. "X" (OPT) denotes the maximum possible speedup achievable, when attention simply returns the $\mathbf{V}$-matrix. Plots shown up to when a model produces an out of memory error on a V100 GPU with 16GB. Vocabulary size used was 256. Best in color.}}
  \label{fig:backward_pass}
\end{figure}

\subsection{Computational costs}
We compared speed-wise the backward pass of the Transformer and the Performer in (B) setting, as it is one of the main computational bottlenecks during training, when using the regular default size $(n_{heads}, n_{layers}, d_{ff}, d) = (8,6,2048,512)$, where $d_{ff}$ denotes the width of the MLP layers. We observed (Fig. \ref{fig:backward_pass}) that in terms of $L$, the Performer reaches nearly linear time and sub-quadratic memory consumption (since the explicit $O(L^{2})$ attention matrix is not stored). In fact, the Performer achieves nearly optimal speedup and memory efficiency possible, depicted by the "X"-line when attention is replaced with the "identity function" simply returning the $\mathbf{V}$-matrix. The combination of both memory and backward pass efficiencies for large $L$ allows respectively, large batch training and lower wall clock time per gradient step. Extensive additional results are demonstrated in Appendix \ref{appendix:computation_costs_bidirectional} by varying layers, raw attention, and architecture sizes.

\subsection{Softmax attention approximation error}
\label{subsec:approx_error_compatibility}

We further examined the approximation error via FAVOR+ in Fig. \ref{fig:approx}.
We demonstrate that \textbf{1.} Orthogonal features produce lower error than unstructured (IID) features, \textbf{2.} Positive features produce lower error than trigonometric $\sin$/$\cos$ features. These two empirically validate the PORF mechanism.
\begin{figure}[H]
  \includegraphics[width=0.498\textwidth]{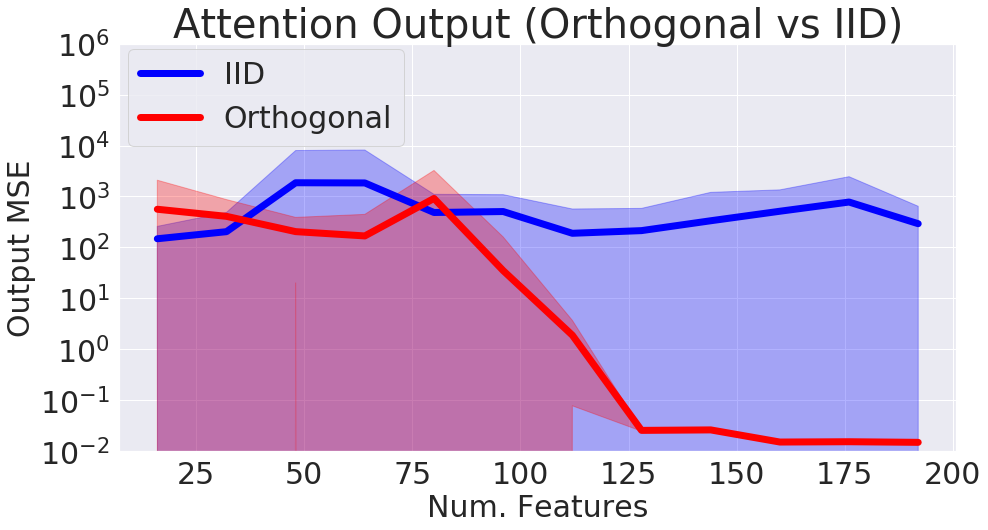}
  \includegraphics[width=0.498\textwidth]{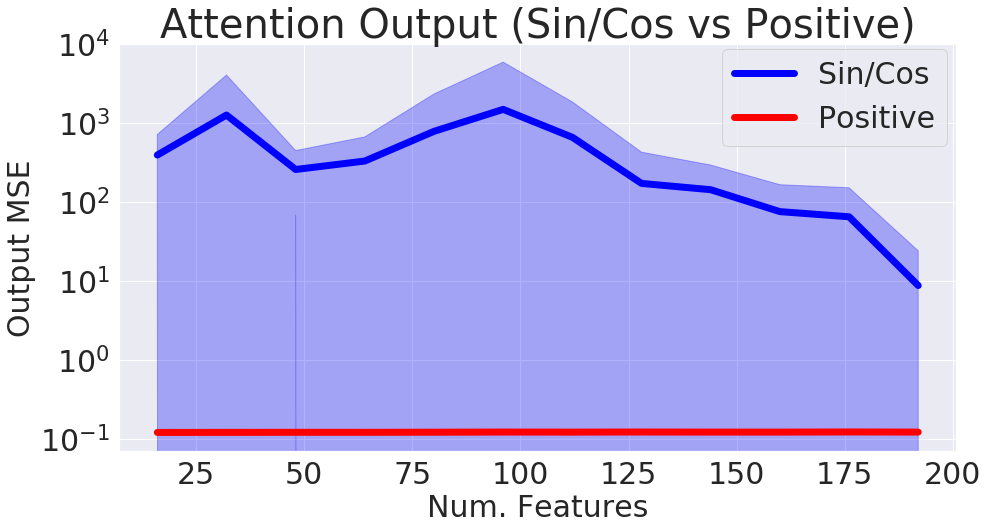}
  \caption{\small{MSE of the approximation output when comparing Orthogonal vs IID features and trigonometric $\sin$/$\cos$ vs positive features. We took $L = 4096, d = 16$, and varied the number of random samples $m$. Standard deviations shown across 15 samples of appropriately normalized random matrix input data.}}
  \label{fig:approx}
\end{figure}

To further improve overall approximation of attention blocks across multiple iterations which further improves training, random samples should be periodically redrawn (Fig. \ref{fig:backward_compatibility}, right). This is a cheap procedure, but can be further optimized (Appendix \ref{subsec:extensions}).

\subsection{Softmax approximation on Transformers}
\label{subsec:softmax_approx_transformer}

Even if the approximation of the attention mechanism is tight, small errors can easily propagate throughout multiple Transformer layers (e.g. MLPs, multiple heads), as we show in Fig. \ref{fig:appendix_approx} (Appendix). In other words, the model's \textit{Lipschitz constant} can easily scale up small attention approximation error, which means that very tight approximations may sometimes be needed. Thus, when applying FAVOR(+)'s softmax approximations on a Transformer model (i.e. "Performer-X-SOFTMAX"), we demonstrate that: 

\textbf{1.} Backwards compatibility with pretrained models is available as a benefit from softmax approximation, via small finetuning (required due to error propagation) even for trigonometric features (Fig. \ref{fig:backward_compatibility}, left) on the LM1B dataset \citep{lm1b}.
However, when on larger dataset PG-19, \textbf{2.} Positive (POS) softmax features (with redrawing) become crucial for achieving performance matching regular Transformers (Fig. \ref{fig:backward_compatibility}, right). 

\begin{figure}[H]
  \centering
  \includegraphics[width=0.99\textwidth]{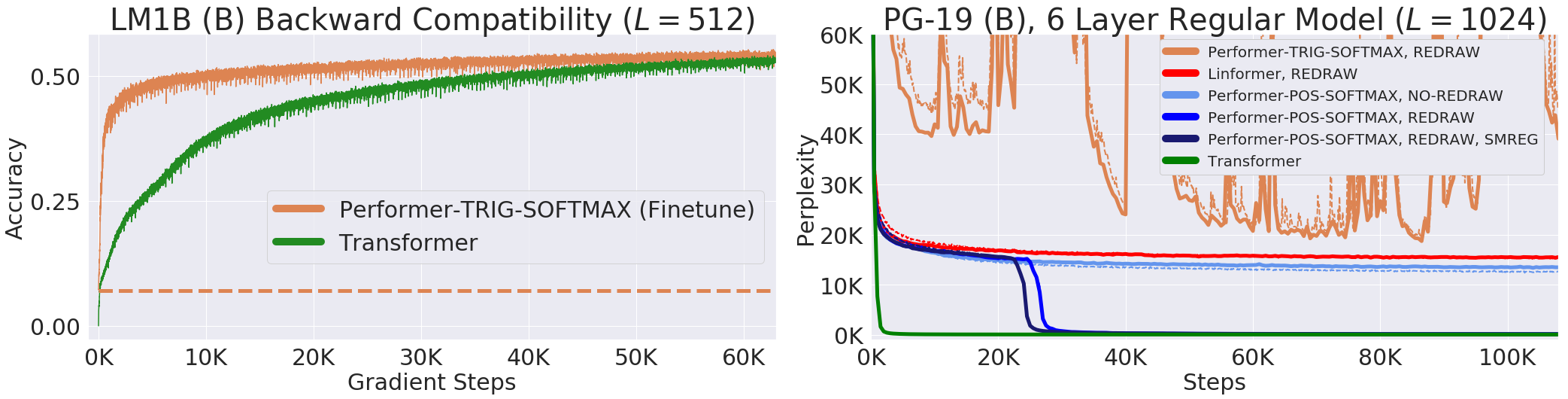}
  \caption{\small We transferred the original pretrained Transformer's weights into the Performer, which produces an initial non-zero 0.07 accuracy (dotted orange line), but quickly recovers accuracy in a small fraction of the original number of gradient steps. However on PG-19, Trigonometric (TRIG) softmax approximation becomes highly unstable (full curve in Appendix \ref{subsec:unstable_trig}), while positive features (POS) (without redrawing) and Linformer (which also approximates softmax) \textit{even with redrawn projections}, plateau at the same perplexity. Positive softmax with feature redrawing is necessary to match the Transformer, with SMREG (regularization from Sec. \ref{ort-theorem}) allowing faster convergence. Additional ablation studies over many attention kernels, showing also that trigonometric random features lead even to NaN values in training are given in Appendix \ref{subsec:appendix_generalized_attention}.}
  \label{fig:backward_compatibility}
\end{figure}

\subsection{Multiple layer training for proteins}
We further benchmark the Performer on both (U) and (B) cases by training a 36-layer model using protein sequences from the Jan. 2019 release of TrEMBL \citep{uniprot2019uniprot}, similar to \citep{progen}. In Fig. \ref{fig:big_benchmarking}, the Reformer and Linformer \textit{significantly drop in accuracy} on the protein dataset. Furthermore, the usefulness of generalized attention is evidenced by Performer-RELU (taking $f=\mathrm{ReLU}$ in Equation \ref{eq:feature}) achieving the highest accuracy in both (U) and (B) cases. Our proposed softmax approximation is also shown to be tight, achieving the same accuracy as the exact-softmax Transformer and confirming our theoretical claims from Section \ref{sec:theory}. 

\begin{figure}[H]
  \centering
  \includegraphics[width=1.0\textwidth]{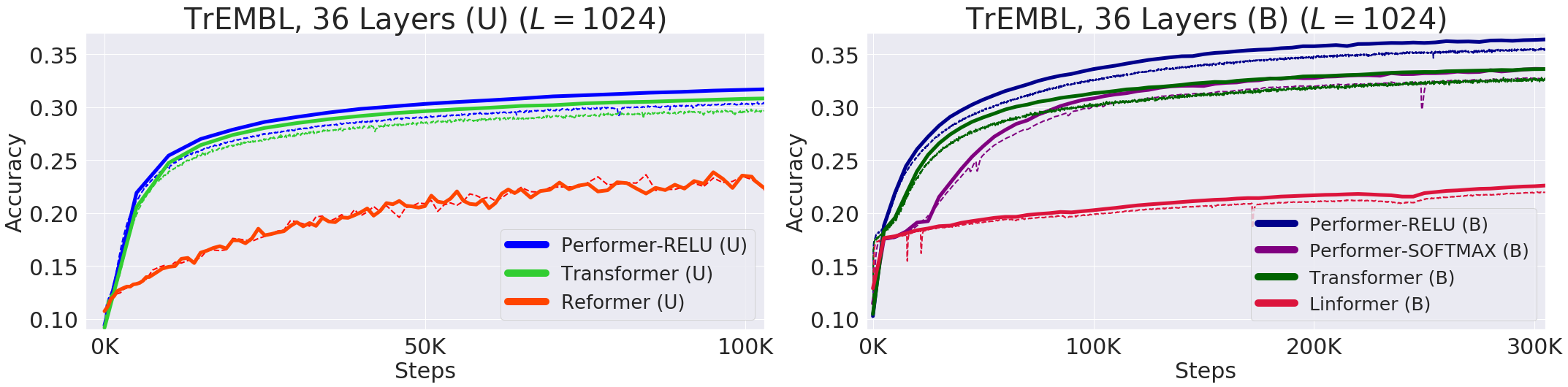}
  \caption{\small Train = Dashed, Validation = Solid. For TrEMBL, we used the exact same model parameters $(n_{heads}, n_{layers}, d_{ff}, d)  = (8, 36, 1024, 512)$ from \protect\citep{progen} for all runs. For fairness, all TrEMBL experiments used 16x16 TPU-v2's. Batch sizes were maximized for each separate run given the compute constraints. Hyperparameters can be found in Appendix \ref{appendix:hyperparameters}. Extended results including dataset statistics, out of distribution evaluations, and visualizations, can be found in Appendix \ref{appendix:protein_extended}.}
  \label{fig:big_benchmarking}
\end{figure}
\vspace{-2mm}
\subsection{Large length training - Common datasets}
On the standard (U) ImageNet64 benchmark from \citep{parmar} with $L = 12288$ which is unfeasible for regular Transformers, we set all models to use the same $(n_{heads}, d_{ff}, d)$ but varying $n_{layers}$. Performer/6-layers matches the Reformer/12-layers, while the Performer/12-layers matches the Reformer/24-layers (Fig. \ref{fig:im64_protein_8192}: left). Depending on hardware (TPU or GPU), we also found that the Performer can be 2x faster than the Reformer via Jax optimizations for the (U) setting. 

For a proof of principle study, we also create an initial protein benchmark for predicting interactions among groups of proteins by concatenating protein sequences to length $L = 8192$ from TrEMBL, long enough to model protein interaction networks without the large sequence alignments required by existing methods \citep{cong2019protein}. In this setting, a regular Transformer overloads memory even at a batch size of $1$ per chip, by a wide margin. Thus as a baseline, we were forced to use a significantly smaller variant, reducing to $(n_{heads}, n_{layers}, d_{ff}, d) = (8, \{1,2,3\}, 256, 256)$. Meanwhile, the Performer trains efficiently at a batch size of 8 per chip using the standard $(8, 6, 2048, 512)$ architecture. We see in Fig. \ref{fig:im64_protein_8192} (right subfigure) that the smaller Transformer ($n_{layer} = 3$) is quickly bounded at $\approx 19\%$, while the Performer is able to train continuously to $\approx 24\%$.

\begin{figure}[H]
  \centering
  \includegraphics[width=1.0\textwidth]{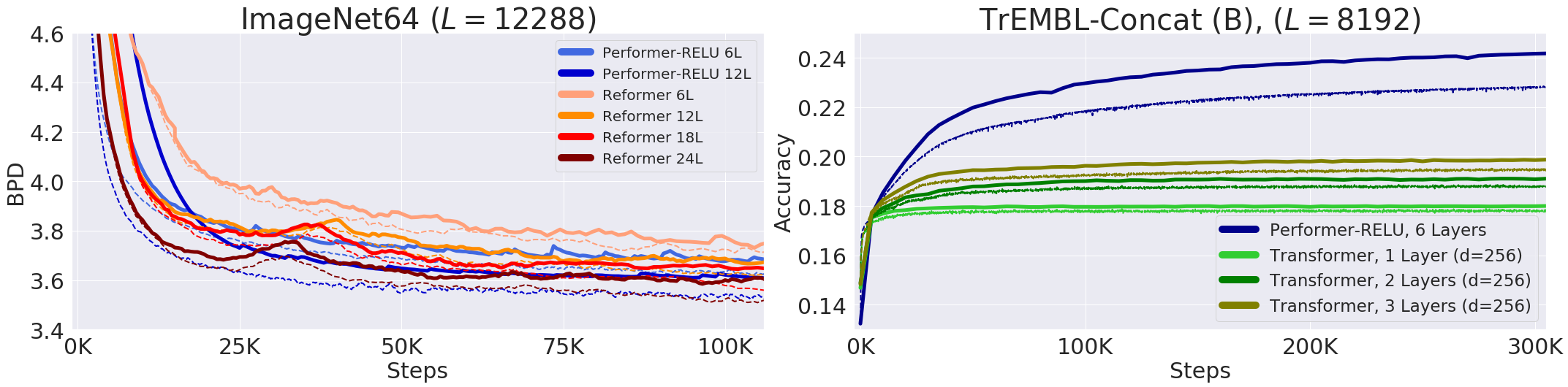}
  \caption{\small Train = Dashed, Validation = Solid. For ImageNet64, all models used the standard $(n_{heads}, d_{ff}, d) = (8, 2048, 512)$. We further show that our positive softmax approximation achieves the same performance as ReLU in Appendix \ref{subsec:softmax_unidirectional}. For concatenated TrEMBL, we varied $n_{layers} \in \{1,2,3\}$ for the smaller Transformer. Hyperparameters can be found in Appendix \ref{appendix:hyperparameters}.}
  \label{fig:im64_protein_8192}
\end{figure}

\vspace{-3mm}
\section{Conclusion}
\label{sec:conclusion}
We presented $\mathrm{Performer}$, a new type of Transformer, relying on our Fast Attention Via positive Orthogonal Random features (FAVOR+) mechanism to significantly improve space and time complexity of regular Transformers. Our mechanism provides to our knowledge the first effective unbiased estimation of the original softmax-based Transformer with linear space and time complexity and opens new avenues in the research on Transformers and the role of non-sparsifying attention mechanisms.



\newpage
\section{Broader impact}
\label{sec:broader_impact}
We believe that the presented algorithm can be impactful in various ways:

\textbf{Biology and Medicine:} Our method has the potential to directly impact research on biological sequence analysis by enabling the Transformer to be applied to much longer sequences without constraints on the structure of the attention matrix. The initial application that we consider is the prediction of interactions between proteins on the proteome scale. Recently published approaches require large evolutionary sequence alignments, a bottleneck for applications to mammalian genomes \citep{cong2019protein}. The potentially broad translational impact of applying these approaches to biological sequences was one of the main motivations of this work. We believe that modern bioinformatics can immensely benefit from new machine learning techniques with Transformers being among the most promising. Scaling up these methods to train faster more accurate language models opens the door to the ability to design sets of molecules with pre-specified interaction properties. These approaches could be used to augment existing physics-based design strategies that are of critical importance for example in the development of new nanoparticle vaccines \citep{marcandalli2019induction}.  

\textbf{Environment:} As we have shown, Performers with FAVOR+ are characterized by much lower compute costs and substantially lower space complexity which can be directly translated to $\mathrm{CO}_{2}$ emission reduction \citep{co2} and lower energy consumption \citep{energy}, as regular Transformers require very large computational resources.

\textbf{Research on Transformers:} We believe that our results can shape research on efficient Transformers architectures, guiding the field towards methods with strong mathematical foundations. Our research may also hopefully extend Transformers also beyond their standard scope (e.g. by considering the Generalized Attention mechanism and connections with kernels). Exploring scalable Transformer architectures that can handle $L$ of the order of magnitude few thousands and more, preserving accuracy of the baseline at the same time, is a gateway to new breakthroughs in bio-informatics, e.g. language modeling for proteins, as we explained in the paper. Our presented method can be potentially a first step.

\textbf{Backward Compatibility:} Our Performer can be used on the top of a regular pre-trained Transformer as opposed to other Transformer variants. Even if up-training is not required, FAVOR+ can still be used for fast inference with no loss of accuracy. We think about this backward compatibility as a very important additional feature of the presented techniques that might be particularly attractive for practitioners.

\textbf{Attention Beyond Transformers:} Finally, FAVOR+ can be applied to approximate exact attention also outside the scope of Transformers. This opens a large volume of new potential applications including: hierarchical attention networks (HANS) \citep{hans}, graph attention networks \citep{gran}, image processing \citep{attention_cvpr}, and reinforcement learning/robotics \citep{tang}.

\section{Acknowledgements}
We thank Nikita Kitaev and Wojciech Gajewski for multiple discussions on the Reformer, and also thank Aurko Roy and Ashish Vaswani for multiple discussions on the Routing Transformer. We further thank Joshua Meier, John Platt, and Tom Weingarten for many fruitful discussions on biological data and useful comments on this draft. We lastly thank Yi Tay and Mostafa Dehghani for discussions on comparing baselines.

Valerii Likhosherstov acknowledges support from the Cambridge Trust and DeepMind. Lucy Colwell acknowledges support from the Simons Foundation. Adrian Weller acknowledges support from 
a Turing AI Fellowship under grant EP/V025379/1, The Alan Turing Institute under EPSRC grant EP/N510129/1 and U/B/000074, and the Leverhulme Trust via CFI. 
\newpage
\bibliographystyle{iclr2021_conference}
\bibliography{performers}

\begin{thebibliography}{60}
\providecommand{\natexlab}[1]{#1}
\providecommand{\url}[1]{\texttt{#1}}
\expandafter\ifx\csname urlstyle\endcsname\relax
  \providecommand{\doi}[1]{doi: #1}\else
  \providecommand{\doi}{doi: \begingroup \urlstyle{rm}\Url}\fi

\bibitem[Bello et~al.(2019)Bello, Zoph, Vaswani, Shlens, and
  Le]{image_transformer}
Irwan Bello, Barret Zoph, Ashish Vaswani, Jonathon Shlens, and Quoc~V. Le.
\newblock Attention augmented convolutional networks.
\newblock \emph{CoRR}, abs/1904.09925, 2019.
\newblock URL \url{http://arxiv.org/abs/1904.09925}.

\bibitem[Beltagy et~al.(2020)Beltagy, Peters, and Cohan]{longformer}
Iz~Beltagy, Matthew~E. Peters, and Arman Cohan.
\newblock Longformer: The long-document transformer.
\newblock \emph{CoRR}, abs/2004.05150, 2020.
\newblock URL \url{https://arxiv.org/abs/2004.05150}.

\bibitem[Chan et~al.(2020)Chan, Saharia, Hinton, Norouzi, and Jaitly]{imputer}
William Chan, Chitwan Saharia, Geoffrey~E. Hinton, Mohammad Norouzi, and
  Navdeep Jaitly.
\newblock Imputer: Sequence modelling via imputation and dynamic programming.
\newblock \emph{CoRR}, abs/2002.08926, 2020.
\newblock URL \url{https://arxiv.org/abs/2002.08926}.

\bibitem[Chelba et~al.(2014)Chelba, Mikolov, Schuster, Ge, Brants, Koehn, and
  Robinson]{lm1b}
Ciprian Chelba, Tomas Mikolov, Mike Schuster, Qi~Ge, Thorsten Brants, Phillipp
  Koehn, and Tony Robinson.
\newblock One billion word benchmark for measuring progress in statistical
  language modeling.
\newblock In \emph{{INTERSPEECH} 2014, 15th Annual Conference of the
  International Speech Communication Association, Singapore, September 14-18,
  2014}, pp.\  2635--2639, 2014.

\bibitem[Chelba et~al.(2020)Chelba, Chen, Bapna, and Shazeer]{chelba}
Ciprian Chelba, Mia~Xu Chen, Ankur Bapna, and Noam Shazeer.
\newblock Faster transformer decoding: N-gram masked self-attention.
\newblock \emph{CoRR}, abs/2001.04589, 2020.
\newblock URL \url{https://arxiv.org/abs/2001.04589}.

\bibitem[Chen et~al.(2018)Chen, Firat, Bapna, Johnson, Macherey, Foster, Jones,
  Schuster, Shazeer, Parmar, Vaswani, Uszkoreit, Kaiser, Chen, Wu, and
  Hughes]{nmt}
Mia~Xu Chen, Orhan Firat, Ankur Bapna, Melvin Johnson, Wolfgang Macherey,
  George~F. Foster, Llion Jones, Mike Schuster, Noam Shazeer, Niki Parmar,
  Ashish Vaswani, Jakob Uszkoreit, Lukasz Kaiser, Zhifeng Chen, Yonghui Wu, and
  Macduff Hughes.
\newblock The best of both worlds: Combining recent advances in neural machine
  translation.
\newblock In \emph{Proceedings of the 56th Annual Meeting of the Association
  for Computational Linguistics, {ACL} 2018, Melbourne, Australia, July 15-20,
  2018, Volume 1: Long Papers}, pp.\  76--86. Association for Computational
  Linguistics, 2018.
\newblock \doi{10.18653/v1/P18-1008}.
\newblock URL \url{https://www.aclweb.org/anthology/P18-1008/}.

\bibitem[Child et~al.(2019)Child, Gray, Radford, and Sutskever]{sparsetr}
Rewon Child, Scott Gray, Alec Radford, and Ilya Sutskever.
\newblock Generating long sequences with sparse transformers.
\newblock \emph{CoRR}, abs/1904.10509, 2019.
\newblock URL \url{http://arxiv.org/abs/1904.10509}.

\bibitem[Choromanski et~al.(2018{\natexlab{a}})Choromanski, Downey, and
  Boots]{psrnn}
Krzysztof Choromanski, Carlton Downey, and Byron Boots.
\newblock Initialization matters: Orthogonal predictive state recurrent neural
  networks.
\newblock In \emph{6th International Conference on Learning Representations,
  {ICLR} 2018, Vancouver, BC, Canada, April 30 - May 3, 2018, Conference Track
  Proceedings}. OpenReview.net, 2018{\natexlab{a}}.
\newblock URL \url{https://openreview.net/forum?id=HJJ23bW0b}.

\bibitem[Choromanski et~al.(2018{\natexlab{b}})Choromanski, Rowland,
  Sarl{\'{o}}s, Sindhwani, Turner, and Weller]{geom}
Krzysztof Choromanski, Mark Rowland, Tam{\'{a}}s Sarl{\'{o}}s, Vikas Sindhwani,
  Richard~E. Turner, and Adrian Weller.
\newblock The geometry of random features.
\newblock In \emph{International Conference on Artificial Intelligence and
  Statistics, {AISTATS} 2018, 9-11 April 2018, Playa Blanca, Lanzarote, Canary
  Islands, Spain}, volume~84 of \emph{Proceedings of Machine Learning
  Research}, pp.\  1--9. {PMLR}, 2018{\natexlab{b}}.
\newblock URL \url{http://proceedings.mlr.press/v84/choromanski18a.html}.

\bibitem[Choromanski et~al.(2019{\natexlab{a}})Choromanski, Pacchiano,
  Pennington, and Tang]{kama}
Krzysztof Choromanski, Aldo Pacchiano, Jeffrey Pennington, and Yunhao Tang.
\newblock {KAMA-NN}s: Low-dimensional rotation based neural networks.
\newblock In \emph{The 22nd International Conference on Artificial Intelligence
  and Statistics, {AISTATS} 2019, 16-18 April 2019, Naha, Okinawa, Japan},
  volume~89 of \emph{Proceedings of Machine Learning Research}, pp.\  236--245.
  {PMLR}, 2019{\natexlab{a}}.
\newblock URL \url{http://proceedings.mlr.press/v89/choromanski19a.html}.

\bibitem[Choromanski et~al.(2019{\natexlab{b}})Choromanski, Rowland, Chen, and
  Weller]{uni}
Krzysztof Choromanski, Mark Rowland, Wenyu Chen, and Adrian Weller.
\newblock Unifying orthogonal {Monte Carlo} methods.
\newblock In \emph{Proceedings of the 36th International Conference on Machine
  Learning, {ICML} 2019, 9-15 June 2019, Long Beach, California, {USA}},
  volume~97 of \emph{Proceedings of Machine Learning Research}, pp.\
  1203--1212. {PMLR}, 2019{\natexlab{b}}.
\newblock URL \url{http://proceedings.mlr.press/v97/choromanski19a.html}.

\bibitem[Choromanski et~al.(2017)Choromanski, Rowland, and Weller]{unreas}
Krzysztof~Marcin Choromanski, Mark Rowland, and Adrian Weller.
\newblock The unreasonable effectiveness of structured random orthogonal
  embeddings.
\newblock In \emph{Advances in Neural Information Processing Systems 30: Annual
  Conference on Neural Information Processing Systems 2017, 4-9 December 2017,
  Long Beach, CA, {USA}}, pp.\  219--228, 2017.

\bibitem[Clevert et~al.(2016)Clevert, Unterthiner, and Hochreiter]{elu}
Djork{-}Arn{\'{e}} Clevert, Thomas Unterthiner, and Sepp Hochreiter.
\newblock Fast and accurate deep network learning by exponential linear units
  (elus).
\newblock In \emph{4th International Conference on Learning Representations,
  {ICLR} 2016, San Juan, Puerto Rico, May 2-4, 2016, Conference Track
  Proceedings}, 2016.
\newblock URL \url{http://arxiv.org/abs/1511.07289}.

\bibitem[Cong et~al.(2019)Cong, Anishchenko, Ovchinnikov, and
  Baker]{cong2019protein}
Qian Cong, Ivan Anishchenko, Sergey Ovchinnikov, and David Baker.
\newblock Protein interaction networks revealed by proteome coevolution.
\newblock \emph{Science}, 365\penalty0 (6449):\penalty0 185--189, 2019.

\bibitem[Consortium(2019)]{uniprot2019uniprot}
UniProt Consortium.
\newblock Uniprot: a worldwide hub of protein knowledge.
\newblock \emph{Nucleic acids research}, 47\penalty0 (D1):\penalty0 D506--D515,
  2019.

\bibitem[Cormen et~al.(2009)Cormen, Leiserson, Rivest, and Stein]{cormen}
Thomas~H. Cormen, Charles~E. Leiserson, Ronald~L. Rivest, and Clifford Stein.
\newblock \emph{Introduction to Algorithms, 3rd Edition}.
\newblock {MIT} Press, 2009.
\newblock ISBN 978-0-262-03384-8.
\newblock URL \url{http://mitpress.mit.edu/books/introduction-algorithms}.

\bibitem[Dai et~al.(2019)Dai, Yang, Yang, Cohen, Carbonell, Le, and
  Salakhutdinov]{transformerxl}
Zihang Dai, Zhilin Yang, Yiming Yang, William~W. Cohen, Jaime Carbonell,
  Quoc~V. Le, and Ruslan Salakhutdinov.
\newblock Transformer-{XL}: Language modeling with longer-term dependency,
  2019.
\newblock URL \url{https://openreview.net/forum?id=HJePno0cYm}.

\bibitem[Dehghani et~al.(2019)Dehghani, Gouws, Vinyals, Uszkoreit, and
  Kaiser]{universal_t}
Mostafa Dehghani, Stephan Gouws, Oriol Vinyals, Jakob Uszkoreit, and Lukasz
  Kaiser.
\newblock Universal transformers.
\newblock In \emph{7th International Conference on Learning Representations,
  {ICLR} 2019, New Orleans, LA, USA, May 6-9, 2019}. OpenReview.net, 2019.
\newblock URL \url{https://openreview.net/forum?id=HyzdRiR9Y7}.

\bibitem[Devlin et~al.(2018)Devlin, Chang, Lee, and Toutanova]{bert}
Jacob Devlin, Ming{-}Wei Chang, Kenton Lee, and Kristina Toutanova.
\newblock {BERT:} pre-training of deep bidirectional transformers for language
  understanding.
\newblock \emph{CoRR}, abs/1810.04805, 2018.
\newblock URL \url{http://arxiv.org/abs/1810.04805}.

\bibitem[Du et~al.(2020)Du, Meier, Ma, Fergus, and Rives]{du2020energy}
Yilun Du, Joshua Meier, Jerry Ma, Rob Fergus, and Alexander Rives.
\newblock Energy-based models for atomic-resolution protein conformations.
\newblock \emph{arXiv preprint arXiv:2004.13167}, 2020.

\bibitem[Elnaggar et~al.(2019)Elnaggar, Heinzinger, Dallago, and
  Rost]{elnaggar2019end}
Ahmed Elnaggar, Michael Heinzinger, Christian Dallago, and Burkhard Rost.
\newblock End-to-end multitask learning, from protein language to protein
  features without alignments.
\newblock \emph{bioRxiv}, pp.\  864405, 2019.

\bibitem[Frostig et~al.(2018)Frostig, Johnson, and Leary]{jax}
Roy Frostig, Matthew Johnson, and Chris Leary.
\newblock Compiling machine learning programs via high-level tracing.
\newblock In \emph{Conference on Machine Learning and Systems 2018}, 2018.
\newblock URL \url{http://www.sysml.cc/doc/2018/146.pdf}.

\bibitem[Fu et~al.(2019)Fu, Liu, Tian, Li, Bao, Fang, and Lu]{attention_cvpr}
Jun Fu, Jing Liu, Haijie Tian, Yong Li, Yongjun Bao, Zhiwei Fang, and Hanqing
  Lu.
\newblock Dual attention network for scene segmentation.
\newblock In \emph{{IEEE} Conference on Computer Vision and Pattern
  Recognition, {CVPR} 2019, Long Beach, CA, USA, June 16-20, 2019}, pp.\
  3146--3154, 2019.

\bibitem[Gulati et~al.(2020)Gulati, Qin, Chiu, Parmar, Zhang, Yu, Han, Wang,
  Zhang, Wu, and Pang]{conformer}
Anmol Gulati, James Qin, Chung-Cheng Chiu, Niki Parmar, Yu~Zhang, Jiahui Yu,
  Wei Han, Shibo Wang, Zhengdong Zhang, Yonghui Wu, and Ruoming Pang.
\newblock Conformer: Convolution-augmented transformer for speech recognition,
  2020.

\bibitem[Huang et~al.(2019)Huang, Vaswani, Uszkoreit, Simon, Hawthorne,
  Shazeer, Dai, Hoffman, Dinculescu, and Eck]{simon}
Cheng{-}Zhi~Anna Huang, Ashish Vaswani, Jakob Uszkoreit, Ian Simon, Curtis
  Hawthorne, Noam Shazeer, Andrew~M. Dai, Matthew~D. Hoffman, Monica
  Dinculescu, and Douglas Eck.
\newblock Music transformer: Generating music with long-term structure.
\newblock In \emph{7th International Conference on Learning Representations,
  {ICLR} 2019, New Orleans, LA, USA, May 6-9, 2019}. OpenReview.net, 2019.
\newblock URL \url{https://openreview.net/forum?id=rJe4ShAcF7}.

\bibitem[Ingraham et~al.(2019)Ingraham, Garg, Barzilay, and
  Jaakkola]{ingraham2019generative}
John Ingraham, Vikas Garg, Regina Barzilay, and Tommi Jaakkola.
\newblock Generative models for graph-based protein design.
\newblock In \emph{Advances in Neural Information Processing Systems}, pp.\
  15794--15805, 2019.

\bibitem[Katharopoulos et~al.(2020)Katharopoulos, Vyas, Pappas, and
  Fleuret]{trans-rnns}
Angelos Katharopoulos, Apoorv Vyas, Nikolaos Pappas, and Fran{\c{c}}ois
  Fleuret.
\newblock Transformers are rnns: Fast autoregressive transformers with linear
  attention.
\newblock \emph{CoRR}, abs/2006.16236, 2020.
\newblock URL \url{https://arxiv.org/abs/2006.16236}.

\bibitem[Kitaev et~al.(2020)Kitaev, Kaiser, and Levskaya]{reformer}
Nikita Kitaev, Lukasz Kaiser, and Anselm Levskaya.
\newblock Reformer: The efficient transformer.
\newblock In \emph{8th International Conference on Learning Representations,
  {ICLR} 2020, Addis Ababa, Ethiopia, April 26-30, 2020}. OpenReview.net, 2020.
\newblock URL \url{https://openreview.net/forum?id=rkgNKkHtvB}.

\bibitem[Kovaleva et~al.(2019)Kovaleva, Romanov, Rogers, and
  Rumshisky]{kovaleva2019revealing}
Olga Kovaleva, Alexey Romanov, Anna Rogers, and Anna Rumshisky.
\newblock Revealing the dark secrets of bert.
\newblock \emph{arXiv preprint arXiv:1908.08593}, 2019.

\bibitem[Kudo \& Richardson(2018)Kudo and Richardson]{sentencepiece}
Taku Kudo and John Richardson.
\newblock Sentencepiece: {A} simple and language independent subword tokenizer
  and detokenizer for neural text processing.
\newblock \emph{CoRR}, abs/1808.06226, 2018.
\newblock URL \url{http://arxiv.org/abs/1808.06226}.

\bibitem[Ladner \& Fischer(1980)Ladner and Fischer]{cumsum}
Richard~E. Ladner and Michael~J. Fischer.
\newblock Parallel prefix computation.
\newblock \emph{J. ACM}, 27\penalty0 (4):\penalty0 831–838, October 1980.
\newblock ISSN 0004-5411.
\newblock \doi{10.1145/322217.322232}.
\newblock URL \url{https://doi.org/10.1145/322217.322232}.

\bibitem[Lin et~al.(2020)Lin, Chen, Zhang, Laroche, and
  Choromanski]{Lin2020DemystifyingOM}
Han Lin, Haoxian Chen, Tianyi Zhang, Cl{\'e}ment Laroche, and Krzysztof
  Choromanski.
\newblock Demystifying orthogonal {Monte Carlo} and beyond.
\newblock \emph{CoRR}, abs/2005.13590, 2020.

\bibitem[Luo et~al.(2020)Luo, Zhang, Lei, and Xie]{luo}
Haoneng Luo, Shiliang Zhang, Ming Lei, and Lei Xie.
\newblock Simplified self-attention for transformer-based end-to-end speech
  recognition.
\newblock \emph{CoRR}, abs/2005.10463, 2020.
\newblock URL \url{https://arxiv.org/abs/2005.10463}.

\bibitem[Madani et~al.(2020)Madani, McCann, Naik, Keskar, Anand, Eguchi, Huang,
  and Socher]{progen}
Ali Madani, Bryan McCann, Nikhil Naik, Nitish~Shirish Keskar, Namrata Anand,
  Raphael~R. Eguchi, Po{-}Ssu Huang, and Richard Socher.
\newblock Progen: Language modeling for protein generation.
\newblock \emph{CoRR}, abs/2004.03497, 2020.
\newblock URL \url{https://arxiv.org/abs/2004.03497}.

\bibitem[Marcandalli et~al.(2019)Marcandalli, Fiala, Ols, Perotti, de~van~der
  Schueren, Snijder, Hodge, Benhaim, Ravichandran, Carter,
  et~al.]{marcandalli2019induction}
Jessica Marcandalli, Brooke Fiala, Sebastian Ols, Michela Perotti, Willem
  de~van~der Schueren, Joost Snijder, Edgar Hodge, Mark Benhaim, Rashmi
  Ravichandran, Lauren Carter, et~al.
\newblock Induction of potent neutralizing antibody responses by a designed
  protein nanoparticle vaccine for respiratory syncytial virus.
\newblock \emph{Cell}, 176\penalty0 (6):\penalty0 1420--1431, 2019.

\bibitem[Nangia \& Bowman(2018)Nangia and Bowman]{listops}
Nikita Nangia and Samuel~R. Bowman.
\newblock Listops: {A} diagnostic dataset for latent tree learning.
\newblock In \emph{Proceedings of the 2018 Conference of the North American
  Chapter of the Association for Computational Linguistics, {NAACL-HLT} 2018,
  New Orleans, Louisiana, USA, June 2-4, 2018, Student Research Workshop}, pp.\
   92--99, 2018.
\newblock \doi{10.18653/v1/n18-4013}.
\newblock URL \url{https://doi.org/10.18653/v1/n18-4013}.

\bibitem[Parmar et~al.(2018)Parmar, Vaswani, Uszkoreit, Kaiser, Shazeer, Ku,
  and Tran]{parmar}
Niki Parmar, Ashish Vaswani, Jakob Uszkoreit, Lukasz Kaiser, Noam Shazeer,
  Alexander Ku, and Dustin Tran.
\newblock Image transformer.
\newblock In \emph{Proceedings of the 35th International Conference on Machine
  Learning, {ICML} 2018, Stockholmsm{\"{a}}ssan, Stockholm, Sweden, July 10-15,
  2018}, volume~80 of \emph{Proceedings of Machine Learning Research}, pp.\
  4052--4061. {PMLR}, 2018.
\newblock URL \url{http://proceedings.mlr.press/v80/parmar18a.html}.

\bibitem[Rae et~al.(2020)Rae, Potapenko, Jayakumar, Hillier, and
  Lillicrap]{compr}
Jack~W. Rae, Anna Potapenko, Siddhant~M. Jayakumar, Chloe Hillier, and
  Timothy~P. Lillicrap.
\newblock Compressive transformers for long-range sequence modelling.
\newblock In \emph{International Conference on Learning Representations}, 2020.
\newblock URL \url{https://openreview.net/forum?id=SylKikSYDH}.

\bibitem[Rahimi \& Recht(2007)Rahimi and Recht]{fourierapprox}
Ali Rahimi and Benjamin Recht.
\newblock Random features for large-scale kernel machines.
\newblock In \emph{Advances in Neural Information Processing Systems 20,
  Proceedings of the Twenty-First Annual Conference on Neural Information
  Processing Systems, Vancouver, British Columbia, Canada, December 3-6, 2007},
  pp.\  1177--1184. Curran Associates, Inc., 2007.
\newblock URL
  \url{http://papers.nips.cc/paper/3182-random-features-for-large-scale-kernel-machines}.

\bibitem[Rives et~al.(2019)Rives, Goyal, Meier, Guo, Ott, Zitnick, Ma, and
  Fergus]{rives}
Alexander Rives, Siddharth Goyal, Joshua Meier, Demi Guo, Myle Ott, C.~Zitnick,
  Jerry Ma, and Rob Fergus.
\newblock Biological structure and function emerge from scaling unsupervised
  learning to 250 million protein sequences.
\newblock \emph{bioArxiv}, 04 2019.
\newblock \doi{10.1101/622803}.

\bibitem[Rowland et~al.(2019)Rowland, Hron, Tang, Choromanski, Sarl{\'{o}}s,
  and Weller]{hron}
Mark Rowland, Jiri Hron, Yunhao Tang, Krzysztof Choromanski, Tam{\'{a}}s
  Sarl{\'{o}}s, and Adrian Weller.
\newblock Orthogonal estimation of {W}asserstein distances.
\newblock In \emph{The 22nd International Conference on Artificial Intelligence
  and Statistics, {AISTATS} 2019, 16-18 April 2019, Naha, Okinawa, Japan},
  volume~89 of \emph{Proceedings of Machine Learning Research}, pp.\  186--195.
  {PMLR}, 2019.
\newblock URL \url{http://proceedings.mlr.press/v89/rowland19a.html}.

\bibitem[Roy et~al.(2020)Roy, Saffar, Vaswani, and Grangier]{routing_t}
Aurko Roy, Mohammad Saffar, Ashish Vaswani, and David Grangier.
\newblock Efficient content-based sparse attention with routing transformers.
\newblock \emph{CoRR}, abs/2003.05997, 2020.
\newblock URL \url{https://arxiv.org/abs/2003.05997}.

\bibitem[Shen et~al.(2018)Shen, Zhang, Yi, Yan, and Zhao]{shen}
Zhuoran Shen, Mingyuan Zhang, Shuai Yi, Junjie Yan, and Haiyu Zhao.
\newblock Factorized attention: Self-attention with linear complexities.
\newblock \emph{CoRR}, abs/1812.01243, 2018.
\newblock URL \url{http://arxiv.org/abs/1812.01243}.

\bibitem[Strubell et~al.(2019)Strubell, Ganesh, and McCallum]{co2}
Emma Strubell, Ananya Ganesh, and Andrew McCallum.
\newblock Energy and policy considerations for deep learning in {NLP}.
\newblock \emph{CoRR}, abs/1906.02243, 2019.
\newblock URL \url{http://arxiv.org/abs/1906.02243}.

\bibitem[Tang et~al.(2020)Tang, Nguyen, and Ha]{tang}
Yujin Tang, Duong Nguyen, and David Ha.
\newblock Neuroevolution of self-interpretable agents.
\newblock \emph{CoRR}, abs/2003.08165, 2020.
\newblock URL \url{https://arxiv.org/abs/2003.08165}.

\bibitem[Tay et~al.(2021)Tay, Dehghani, Abnar, Shen, Bahri, Pham, Rao, Yang,
  Ruder, and Metzler]{lra}
Yi~Tay, Mostafa Dehghani, Samira Abnar, Yikang Shen, Dara Bahri, Philip Pham,
  Jinfeng Rao, Liu Yang, Sebastian Ruder, and Donald Metzler.
\newblock Long range arena: {A} benchmark for efficient transformers.
\newblock 2021.

\bibitem[Tsai et~al.(2019)Tsai, Bai, Yamada, Morency, and
  Salakhutdinov]{tsai2019transformer}
Yao-Hung~Hubert Tsai, Shaojie Bai, Makoto Yamada, Louis-Philippe Morency, and
  Ruslan Salakhutdinov.
\newblock Transformer dissection: An unified understanding for transformer’s
  attention via the lens of kernel.
\newblock In \emph{Proceedings of the 2019 Conference on Empirical Methods in
  Natural Language Processing and the 9th International Joint Conference on
  Natural Language Processing (EMNLP-IJCNLP)}, pp.\  4335--4344, 2019.

\bibitem[Vaswani et~al.(2017)Vaswani, Shazeer, Parmar, Uszkoreit, Jones, Gomez,
  Kaiser, and Polosukhin]{transformer}
Ashish Vaswani, Noam Shazeer, Niki Parmar, Jakob Uszkoreit, Llion Jones,
  Aidan~N Gomez, \L~ukasz Kaiser, and Illia Polosukhin.
\newblock Attention is all you need.
\newblock In \emph{Advances in Neural Information Processing Systems 30}, pp.\
  5998--6008. Curran Associates, Inc., 2017.
\newblock URL
  \url{http://papers.nips.cc/paper/7181-attention-is-all-you-need.pdf}.

\bibitem[Velickovic et~al.(2018)Velickovic, Cucurull, Casanova, Romero,
  Li{\`{o}}, and Bengio]{gran}
Petar Velickovic, Guillem Cucurull, Arantxa Casanova, Adriana Romero, Pietro
  Li{\`{o}}, and Yoshua Bengio.
\newblock Graph attention networks.
\newblock In \emph{6th International Conference on Learning Representations,
  {ICLR} 2018, Vancouver, BC, Canada, April 30 - May 3, 2018, Conference Track
  Proceedings}. OpenReview.net, 2018.
\newblock URL \url{https://openreview.net/forum?id=rJXMpikCZ}.

\bibitem[Vig(2019)]{vig2019multiscale}
Jesse Vig.
\newblock A multiscale visualization of attention in the transformer model.
\newblock \emph{arXiv preprint arXiv:1906.05714}, 2019.

\bibitem[Vig \& Belinkov(2019)Vig and Belinkov]{analyzing_attention}
Jesse Vig and Yonatan Belinkov.
\newblock Analyzing the structure of attention in a transformer language model.
\newblock \emph{CoRR}, abs/1906.04284, 2019.
\newblock URL \url{http://arxiv.org/abs/1906.04284}.

\bibitem[Vig et~al.(2020)Vig, Madani, Varshney, Xiong, Socher, and
  Rajani]{bertology}
Jesse Vig, Ali Madani, Lav~R. Varshney, Caiming Xiong, Richard Socher, and
  Nazneen~Fatema Rajani.
\newblock Bertology meets biology: Interpreting attention in protein language
  models.
\newblock \emph{CoRR}, abs/2006.15222, 2020.
\newblock URL \url{https://arxiv.org/abs/2006.15222}.

\bibitem[Vinyals et~al.(2015)Vinyals, Fortunato, and Jaitly]{pointer}
Oriol Vinyals, Meire Fortunato, and Navdeep Jaitly.
\newblock Pointer networks.
\newblock In \emph{Advances in Neural Information Processing Systems 28: Annual
  Conference on Neural Information Processing Systems 2015, December 7-12,
  2015, Montreal, Quebec, Canada}, pp.\  2692--2700, 2015.

\bibitem[Wang et~al.(2020)Wang, Li, Khabsa, Fang, and Ma]{linformer}
Sinong Wang, Belinda~Z. Li, Madian Khabsa, Han Fang, and Hao Ma.
\newblock Linformer: Self-attention with linear complexity.
\newblock \emph{CoRR}, abs/2006.04768, 2020.
\newblock URL \url{https://arxiv.org/abs/2006.04768}.

\bibitem[Xiao et~al.(2019)Xiao, Li, Zhu, Yu, and Liu]{shared_weights}
Tong Xiao, Yinqiao Li, Jingbo Zhu, Zhengtao Yu, and Tongran Liu.
\newblock Sharing attention weights for fast transformer.
\newblock In \emph{Proceedings of the Twenty-Eighth International Joint
  Conference on Artificial Intelligence, {IJCAI} 2019, Macao, China, August
  10-16, 2019}, pp.\  5292--5298. ijcai.org, 2019.
\newblock \doi{10.24963/ijcai.2019/735}.
\newblock URL \url{https://doi.org/10.24963/ijcai.2019/735}.

\bibitem[Yang et~al.(2016)Yang, Yang, Dyer, He, Smola, and Hovy]{hans}
Zichao Yang, Diyi Yang, Chris Dyer, Xiaodong He, Alexander~J. Smola, and
  Eduard~H. Hovy.
\newblock Hierarchical attention networks for document classification.
\newblock In \emph{{NAACL} {HLT} 2016, The 2016 Conference of the North
  American Chapter of the Association for Computational Linguistics: Human
  Language Technologies, San Diego California, USA, June 12-17, 2016}, pp.\
  1480--1489. The Association for Computational Linguistics, 2016.
\newblock \doi{10.18653/v1/n16-1174}.
\newblock URL \url{https://doi.org/10.18653/v1/n16-1174}.

\bibitem[You et~al.(2020)You, Li, Xu, Fu, Wang, Chen, Baraniuk, Wang, and
  Lin]{energy}
Haoran You, Chaojian Li, Pengfei Xu, Yonggan Fu, Yue Wang, Xiaohan Chen,
  Richard~G. Baraniuk, Zhangyang Wang, and Yingyan Lin.
\newblock Drawing early-bird tickets: Toward more efficient training of deep
  networks.
\newblock In \emph{International Conference on Learning Representations}, 2020.
\newblock URL \url{https://openreview.net/forum?id=BJxsrgStvr}.

\bibitem[Yu et~al.(2016)Yu, Suresh, Choromanski, Holtmann{-}Rice, and
  Kumar]{ort}
Felix~X. Yu, Ananda~Theertha Suresh, Krzysztof~Marcin Choromanski, Daniel~N.
  Holtmann{-}Rice, and Sanjiv Kumar.
\newblock Orthogonal random features.
\newblock In \emph{Advances in Neural Information Processing Systems 29: Annual
  Conference on Neural Information Processing Systems 2016, December 5-10,
  2016, Barcelona, Spain}, pp.\  1975--1983, 2016.

\bibitem[Zambaldi et~al.(2019)Zambaldi, Raposo, Santoro, Bapst, Li, Babuschkin,
  Tuyls, Reichert, Lillicrap, Lockhart, Shanahan, Langston, Pascanu, Botvinick,
  Vinyals, and Battaglia]{relational}
Vin{\'{\i}}cius~Flores Zambaldi, David Raposo, Adam Santoro, Victor Bapst,
  Yujia Li, Igor Babuschkin, Karl Tuyls, David~P. Reichert, Timothy~P.
  Lillicrap, Edward Lockhart, Murray Shanahan, Victoria Langston, Razvan
  Pascanu, Matthew Botvinick, Oriol Vinyals, and Peter~W. Battaglia.
\newblock Deep reinforcement learning with relational inductive biases.
\newblock In \emph{7th International Conference on Learning Representations,
  {ICLR} 2019, New Orleans, LA, USA, May 6-9, 2019}, 2019.

\bibitem[Zhu et~al.(2015)Zhu, Kiros, Zemel, Salakhutdinov, Urtasun, Torralba,
  and Fidler]{bookcorpus}
Yukun Zhu, Ryan Kiros, Richard~S. Zemel, Ruslan Salakhutdinov, Raquel Urtasun,
  Antonio Torralba, and Sanja Fidler.
\newblock Aligning books and movies: Towards story-like visual explanations by
  watching movies and reading books.
\newblock In \emph{2015 {IEEE} International Conference on Computer Vision,
  {ICCV} 2015, Santiago, Chile, December 7-13, 2015}, pp.\  19--27, 2015.
\newblock \doi{10.1109/ICCV.2015.11}.
\newblock URL \url{https://doi.org/10.1109/ICCV.2015.11}.

\end{thebibliography}

\newpage

\section*{APPENDIX: Rethinking Attention with Performers}
\setcounter{section}{0}
\renewcommand{\thesection}{\Alph{section}}
\label{sec:appendix}

\section{Hyperparameters for experiments}
\label{appendix:hyperparameters}
This optimal setting (including comparisons to approximate softmax) we use for the Performer is specified in the Generalized Attention (Subsec. \ref{subsec:generalized_default}), and \textbf{unless specifically mentioned (e.g. using name "Performer-SOFTMAX"), "Performer" refers to using this generalized attention setting.}

\subsection{Metrics}
We report the following evaluation metrics:
\vspace{-3mm}
\begin{enumerate}[itemsep=0.2mm]
\item \textbf{Accuracy}: For unidirectional models, we measure the accuracy on next-token prediction, averaged across all sequence positions in the dataset. For bidirectional models, we mask each token with $15\%$ probability (same as \citep{bert}) and measure accuracy across the masked positions.
\item \textbf{Perplexity}: For unidirectional models, we measure perplexity across all sequence positions in the dataset. For bidirectional models, similar to the accuracy case, we measure perplexity across the masked positions.
\item \textbf{Bits Per Dimension/Character (BPD/BPC)}: This calculated by loss divided by $\ln(2)$.
\end{enumerate}
\vspace{-3mm}
We used the full evaluation dataset for TrEMBL in the plots in the main section, while for other datasets such as ImageNet64 and PG-19 which have very large evaluation dataset sizes, we used random batches (>2048 samples) for plotting curves. 

\subsubsection{PG-19 Preprocessing}

The PG-19 dataset \citep{compr} is presented as a challenging long range text modeling task. It consists of out-of-copyright Project Gutenberg books published before 1919. It does not have a fixed vocabulary size, instead opting for any tokenization which can model an arbitrary string of text. We use a unigram SentencePiece vocabulary \citep{sentencepiece} with 32768 tokens, which maintains whitespace and is completely invertible to the original book text. Perplexities are calculated as the average log-likelihood per token, multiplied by the ratio of the sentencepiece tokenization to number of tokens in the original dataset. The original dataset token count per split is: train=1973136207, validation=3007061, test=6966499. Our sentencepiece tokenization yields the following token counts per split: train=3084760726, valid=4656945, and test=10699704. This gives log likelihood multipliers of train=1.5634, valid=1.5487, test=1.5359 per split before computing perplexity, which is equal to $\exp(\text{log likelihood multiplier} * \text{loss})$.

Preprocessing for TrEMBL is extensively explained in Appendix \ref{appendix:protein_extended}.

\subsection{Training Hyperparameters} 
Unless specifically stated, all Performer + Transformer runs by default used $0.5$ grad clip, $0.1$ weight decay, $0.1$ dropout, $10^{-3}$ fixed learning rate with Adam hyperparameters $(\beta_{1} = 0.9, \beta_{2} = 0.98, \epsilon = 10^{-9})$, with batch size maximized (until TPU memory overload) for a specific model. 

All 36-layer protein experiments used the same amount of compute (i.e. 16x16 TPU-v2, 8GB per chip). For concatenated experiments, 16x16 TPU-v2's were also used for the Performer, while 8x8's were used for the 1-3 layer $(d = 256)$ Transformer models (using 16x16 did not make a difference in accuracy).

\textbf{Note that Performers are using the same training hyperparameters as Transformers, yet achieving competitive results} - this shows that FAVOR can act as a simple drop-in without needing much tuning.

\subsection{Approximate Softmax Attention Default Values} The optimal values, set to default parameters\footnote{\url{https://github.com/google-research/google-research/blob/master/performer/fast_attention}}
, are: renormalize\_attention = True, numerical stabilizer = $10^{-6}$, number of features = 256, ortho\_features = True, ortho\_scaling = 0.0.

\subsection{Generalized Attention Default Values} \label{subsec:generalized_default} The optimal values, set to default parameters\footnote{\url{https://github.com/google-research/google-research/blob/master/performer/fast_attention}}
, are: renormalize\_attention = True, numerical stabilizer = 0.0, number of features = 256, kernel = ReLU, kernel\_epsilon = $10^{-3}$. 

\subsection{Reformer Default Values} \label{subsec:reformer_default}
For the Reformer, we used the same hyperparameters as mentioned for protein experiments, without gradient clipping, while using the defaults\footnote{\url{https://github.com/google/trax/blob/master/trax/supervised/configs/reformer_imagenet64.gin}} (which instead use learning rate decay) for ImageNet-64. In both cases, the Reformer used the same default LSH attention parameters.

\subsection{Linformer Default Values} \label{subsec:linformer_default}
Using our standard pipeline as mentioned above, we replaced the attention function with the Linformer variant via Jax, with $\delta = 10^{-6}, k=600$ (same notation used in the paper \citep{linformer}), where $\delta$ is the exponent in a renormalization procedure using $e^{-\delta}$ as a multiplier in order to approximate softmax, while $k$ is the dimension of the projections of the $\mathbf{Q}$ and $\mathbf{K}$ matrices. As a sanity check, we found that our Linformer implementation in Jax correctly approximated exact softmax's output within $0.02$ error for all entries.

Note that for rigorous comparisons, our Linformer hyperparameters are even stronger than the defaults found in \citep{linformer}, as:

\begin{itemize}
\item We use $k=600$, which is more than twice than the default $k=256$ from the paper, and also twice than our default $m = 256$ number of features.
\item We also use redrawing, which avoids "unlucky" projections on $\mathbf{Q}$ and $\mathbf{K}$. 
\end{itemize}

\newpage

\section{Main Algorithm: FAVOR+}
\label{appendix:main_algorithm}
We outline the main algorithm for FAVOR+ formally: 

\begin{algorithm}[h]
\SetAlgoLined
\SetKwInOut{Input}{Input}
\Input{ $\mathbf{Q}, \mathbf{K}, \mathbf{V} \in \mathbb{R}^{L \times d}$, $\mathrm{isBidirectional}$ - binary flag.}
\KwResult{$\widehat{\mathrm{Att}}_\leftrightarrow (\mathbf{Q}, \mathbf{K}, \mathbf{V}) \in \mathbb{R}^{L \times L}$ if $\mathrm{isBidirectional}$, $\widehat{\mathrm{Att}}_\to (\mathbf{Q}, \mathbf{K}, \mathbf{V}) \in \mathbb{R}^{L \times L}$ otherwise.}
Compute $\mathbf{Q}^{\prime}$ and $\mathbf{K}^{\prime}$ as described in Section \ref{sec:gka} and Section \ref{howto} and take $\mathbf{C} := \begin{bmatrix} \mathbf{V} & \mathbf{1}_L \end{bmatrix}$\;
\eIf{$\mathrm{isBidirectional}$}{
    $\mathrm{Buf}_1 := (\mathbf{K}^{\prime})^\top \mathbf{C} \in \mathbb{R}^{M \times (d + 1)}, \quad \mathrm{Buf}_2 := \mathbf{Q}^{\prime} \mathrm{Buf}_1 \in \mathbb{R}^{L \times (d + 1)}$\;
}{
    Compute $\mathbf{G}$ and its prefix-sum tensor $\mathbf{G}^\mathrm{PS}$ according to (\ref{eq:cumsum})\;
    $\mathrm{Buf}_2 := \begin{bmatrix} \mathbf{G}^\mathrm{PS}_{1,:,:} \mathbf{Q}^{\prime}_1 & \dots & \mathbf{G}^\mathrm{PS}_{L,:,:} \mathbf{Q}^{\prime}_L \end{bmatrix}^\top \in \mathbb{R}^{L \times (d + 1)}$\;
}
$\begin{bmatrix} \mathrm{Buf}_3 & \mathrm{buf}_4 \end{bmatrix} := \mathrm{Buf}_2, \quad \mathrm{Buf}_3 \in \mathbb{R}^{L \times d}, \quad \mathrm{buf}_4 \in \mathbb{R}^L$\;
\Return $\mathrm{diag} (\mathrm{buf}_4)^{-1} \mathrm{Buf}_3$\;
\caption{FAVOR+ (bidirectional or unidirectional).}
\label{alg:1}
\end{algorithm}

\subsection{Unidirectional Case and Prefix Sums}
\label{subsec:unidirectional}
We explain how our analysis from Section \ref{sec:gka} can be extended to the unidirectional mechanism in this section.
Notice that this time attention matrix $\mathbf{A}$ is masked, i.e. all its entries not in the lower-triangular part (which contains the diagonal) are zeroed (see also Fig. \ref{fig:prefix_sum}).

\begin{figure}[h]
  \centering
  \includegraphics[width=0.99\textwidth]{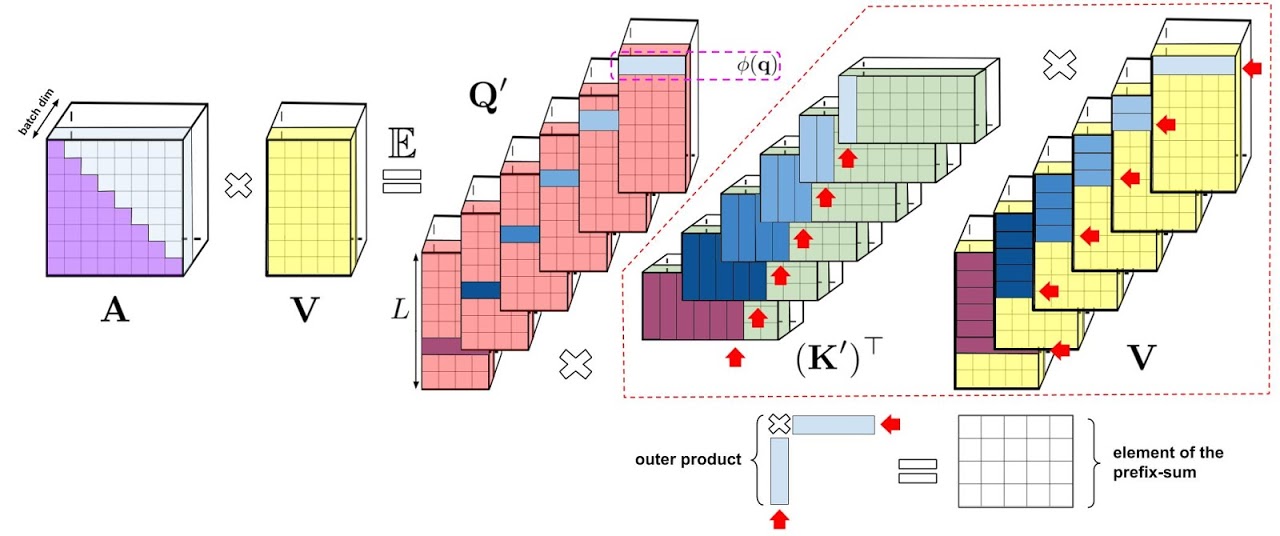}
  \caption{\small{Visual representation of the prefix-sum algorithm for unidirectional attention. For clarity, we omit attention normalization in this visualization. The algorithm keeps the prefix-sum which is a matrix obtained by summing the outer products of random features corresponding to keys with value-vectors. At each given iteration of the prefix-sum algorithm, a random feature vector corresponding to a query is multiplied by the most recent prefix-sum (obtained by summing all outer-products corresponding to preceding tokens) to obtain a new row of the matrix $\mathbf{AV}$ which is output by the attention mechanism.\\}}
  \label{fig:prefix_sum}
\end{figure}

For the unidirectional case, our analysis is similar as for the bidirectional case, but this time our goal is to compute $\mathrm{tril}(\mathbf{Q}^{\prime} (\mathbf{K}^{\prime})^\top) \mathbf{C}$ without constructing and storing the $L \times L$-sized matrix $\mathrm{tril}(\mathbf{Q}^{\prime} (\mathbf{K}^{\prime})^\top)$ explicitly, where
$\mathbf{C}~=~\begin{bmatrix} V & \mathbf{1}_L \end{bmatrix} \in \mathbb{R}^{L \times (d + 1)}$. In order to do so, observe that $\forall 1 \leq i \leq L$:
\begin{equation} \label{eq:cumsum}
    [\mathrm{tril}(\mathbf{Q}^{\prime} (\mathbf{K}^{\prime})^\top) \mathbf{C}]_i = \mathbf{G}^\mathrm{PS}_{i,:,:} \times \mathbf{Q}^{\prime}_i, \quad \mathbf{G}^\mathrm{PS}_{i,:,:} = \sum_{j = 1}^i \mathbf{G}_{j,:,:}, \quad \mathbf{G}_{j,:,:} = \mathbf{K}^{\prime}_j \mathbf{C}_j^\top \in \mathbb{R}^{M \times (d + 1)}
\end{equation}
where $\mathbf{G}, \mathbf{G}^\mathrm{PS} \in \mathbb{R}^{L \times M \times (d + 1)}$ are 3d-tensors. Each slice $\mathbf{G}^\mathrm{PS}_{:,l,p}$ is therefore a result of a prefix-sum (or cumulative-sum) operation applied to $\mathbf{G}_{:,l,p}$: $\mathbf{G}^\mathrm{PS}_{i,l,p} = \sum_{j = 1}^i \mathbf{G}_{i,l,p}$. An efficient algorithm to compute the prefix-sum of $L$ elements takes $O(L)$ total steps and $O(\log L)$ time when computed in parallel \citep{cumsum, cormen}. 
See Algorithm \ref{alg:1} for the whole approach.

\subsection{Orthogonal Random Features - Extensions}
\label{subsec:extensions}

As mentioned in the main text, for isotropic $\Omega$ (true for most practical applications, including regular attention), instead of sampling $\omega_{i}$ independently, we can use \emph{orthogonal random features} (ORF) \citep{ort, unreas, geom}: these maintain the marginal distributions of samples $\omega_{i}$ while enforcing that different samples are orthogonal. If we need $m>d$, ORFs still can be used locally within each $d \times d$ block of $\mathbf{W}$ \citep{ort}.

ORFs were introduced to reduce the variance of Monte Carlo estimators \citep{ort, unreas, geom, kama, hron,psrnn, uni} and we showed in the theoretical and experimental sections from the main body that they do indeed lead to more accurate approximations and substantially better downstream results. There exist several variants of the ORF-mechanism and in the main body we discussed only the base one (that we refer to here as \textit{regular}). Below we briefly review the most efficient ORF mechanisms (based on their strengths and costs) to present the most complete picture.

\textbf{(1) Regular ORFs [R-ORFs]:} Applies Gaussian orthogonal matrices \citep{ort}. Encodes matrix $\mathbf{W}$ of $\omega$-samples (with different rows corresponding to different samples) in $O(md)$ space. Provides algorithm for computing $\mathbf{Wx}$ in $O(md)$ time for any $\mathbf{x} \in \mathbb{R}^{d}$. Gives unbiased estimation. Requires one-time $O(md^{2})$ preprocessing (Gram-Schmidt orthogonalization).

\textbf{(2) Hadamard/Givens ORFs [H/G-ORFs]:} Applies random Hadamard \citep{unreas} or Givens matrices \citep{uni}. Encodes matrix $\mathbf{W}$ in $O(m)$ or $O(m\log(d))$ space. Provides algorithm for computing $\mathbf{Wx}$ in $O(m\log(d))$ time for any $\mathbf{x} \in \mathbb{R}^{d}$. Gives small bias (tending to $0$ with $d \rightarrow \infty$).

\subsection{Time and Space Complexity - Detailed Analysis}

We see that a variant of bidirectional FAVOR+ using iid samples or R-ORFs has $O(md+Ld+mL)$ space complexity as opposed to $\Theta(L^{2} + Ld)$ space complexity of the baseline. Unidirectional FAVOR+ using fast prefix-sum pre-computation in parallel \citep{cumsum, cormen} has $O(m L d)$ space complexity to store $\mathbf{G}^\textrm{PS}$ which can be reduced to $O(md+Ld+mL)$ by running a simple (though non-parallel in $L$) aggregation of $\mathbf{G}^\textrm{PS}_{i,:,:}$ without storing the whole tensor $\mathbf{G}^{\textrm{PS}}$ in memory. From Subsec. \ref{subsec:extensions}, we know that if instead we use G-ORFs, then space complexity is reduced to $O(m\log(d) + Ld + mL)$ and if the H-ORFs mechanism is used, then space is further reduced to $O(m + Ld + mL)=O(Ld+mL)$. Thus for $m,d \ll L$ all our variants provide substantial space complexity improvements since they do not need to store the attention matrix explicitly.

The time complexity of Algorithm \ref{alg:1} is $O(L m d)$ (note that constructing $\mathbf{Q}^{\prime}$ and $\mathbf{K}^{\prime}$ can be done in time $O(Lmd)$). Note that the time complexity of our method is much lower than $O(L^2 d)$ of the baseline for $L \gg m$.

As explained in Subsec. \ref{subsec:extensions}, the R-ORF mechanism incurs an extra one-time $O(md^{2})$ cost (negligible compared to the $O(Lmd)$ term for $L \gg d$). H-ORFs or G-ORFs do not have this cost, and when FAVOR+ uses them, computing $\mathbf{Q}^{\prime}$ and $\mathbf{K}^{\prime}$ can be conducted in time $O(L\log(m)d)$ as opposed to $O(Lmd)$ (see: Subsec. \ref{subsec:extensions}). Thus even though H/G-ORFs do not change the asymptotic time complexity, they improve the constant factor from the leading term. This might play an important role in training very large models. 

The number of random features $m$ allows a trade-off between computational complexity and the level of approximation: bigger $m$ results in higher computation costs, but also in a lower variance of the estimate of $\mathbf{A}$. In the theoretical section from the main body we showed that in practice we can take $M=\Theta(d\log(d))$.

Observe that the FAVOR+ algorithm is highly-parallelizable, and benefits from fast matrix multiplication and broadcasted operations on GPUs or TPUs.

\newpage

\section{Experimental Details for Protein Modeling Tasks}
\label{appendix:protein_extended}
\subsection{TrEMBL Dataset}

\begin{table}[h]
\centering
\renewcommand{\arraystretch}{1.4}
\resizebox{0.9\textwidth}{!}{%
\begin{tabular}{ |c|c|c|c|c|c|c|c|c| } 
\hline
\multirow{2}{*}{\bf Dataset} & \multirow{2}{*}{\bf Set Name} & \multirow{2}{*}{\bf Count} & \multicolumn{5}{|c|}{\bf Length Statistics} \\
\cline{4-8}
 & &  & \textbf{Min} & \textbf{Max}& \textbf{Mean}& \textbf{STD} & \textbf{Median} \\
\hline
\multirow{4}{*}{TrEMBL}
& Train & 104,863,744 & 2 & 74,488 & 353.09 & 311.16 & 289.00 \\ 
& Valid & 102,400 & 7 & 11,274 & 353.62 & 307.42 & 289.00 \\ 
\cline{2-8}
& Test & 1,033,216 & 8 & 32,278 & 353.96 & 312.23 & 289.00 \\ 
& OOD & 29,696 & 24 & 4,208 & 330.96 & 269.86 & 200.00 \\ 
\hline
& & & & & & & \vspace{-4.5mm}\\
\hline
\multirow{2}{*}{\begin{tabular}{c}TrEMBL\\ (concat)\end{tabular}} & Train & 4,532,224 & \multirow{2}{*}{8,192} & \multirow{2}{*}{8,192} & \multirow{2}{*}{8,192} & \multirow{2}{*}{0} & \multirow{2}{*}{8,192}\\
& Valid & 4,096 & & & & & \\ 
\hline
\end{tabular}
}
\vspace{2mm}
\caption{Statistics for the TrEMBL single sequence and the long sequence task.}
\label{table-trembl-statistics}
\end{table}

We used the TrEMBL dataset\footnote{\url{https://www.uniprot.org/statistics/TrEMBL}}, which contains 139,394,261 sequences of which 106,030,080 are unique. While the training dataset appears smaller than the one used in Madani et al. \citep{progen}, we argue that it includes most of the relevant sequences. Specifically, the TrEMBL dataset consists of the subset of UniProtKB sequences that have been computationally analyzed but not manually curated, and accounts for $\approx99.5\%$ of the total number of sequences in the UniProtKB dataset\footnote{\url{https://www.uniprot.org/uniprot/}}.

Following the methodology described in Madani et al. \citep{progen}, we used both an OOD-Test set, where a selected subset of Pfam families are held-out for valuation, and an IID split, where the remaining protein sequences are split randomly into train, valid, and test tests. We held-out the following protein families (PF18369, PF04680, PF17988, PF12325, PF03272, PF03938, PF17724, PF10696, PF11968, PF04153, PF06173, PF12378, PF04420, PF10841, PF06917, PF03492, PF06905, PF15340, PF17055, PF05318), which resulted in 29,696 OOD sequences. We note that, due to deduplication and potential TrEMBL version mismatch, our OOD-Test set does not match exactly the one in Madani et al. \citep{progen}. We also note that this OOD-Test selection methodology does not guarantee that the evaluation sequences are within a minimum distance from the sequences used during training. In future work, we will include rigorous distance based splits.

The statistics for the resulting dataset splits are reported in Table \ref{table-trembl-statistics}. In the standard sequence modeling task, given the length statistics that are reported in the table, we clip single sequences to maximum length $L=1024$, which results in few sequences being truncated significantly.

In the long sequence task, the training and validation sets are obtained by concatenating the sequences, separated by an end-of-sequence token, and grouping the resulting chain into non-overlapping sequences of length $L=8192$.


\subsection{Empirical Baseline}

\begin{figure}[h]
  \centering
  \includegraphics[width=0.99\linewidth]{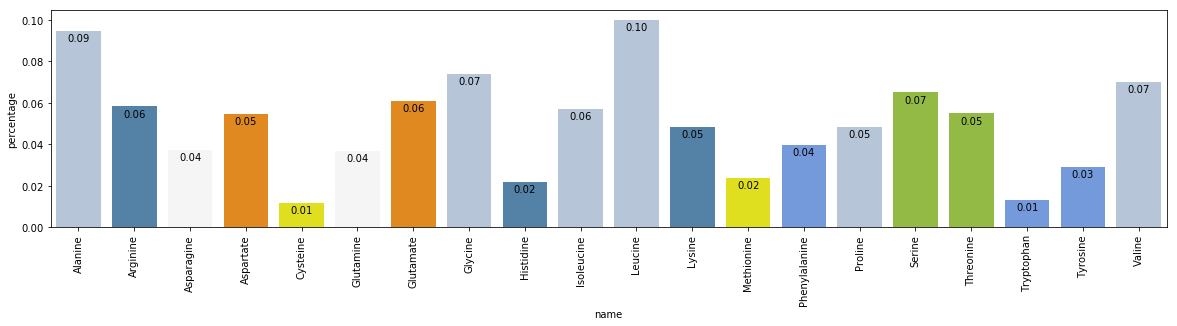}
  \caption{\small{Visualization of the estimated empirical distribution for the 20 standard amino acids, colored by their class. Note the consistency with the statistics on the TrEMBL web page.}}
  \vspace{2mm}
  \label{figure-empirical-baseline}
\end{figure}

A random baseline, with uniform probability across all the vocabulary tokens at every position, has accuracy $5\%$ (when including only the 20 standard amino acids) and $4\%$ (when also including the 5 anomalous amino acids \citep{uniprot2019uniprot}). However, the empirical frequencies of the various amino acids in our dataset may be far from uniform, so we also consider an \textit{empirical baseline} where the amino acid probabilities are proportional to their empirical frequencies in the training set.

Figure \ref{figure-empirical-baseline} shows the estimated empirical distribution. We use both the standard and anomalous amino acids, and we crop sequences to length 1024 to match the data processing performed for the Transformer models. The figure shows only the 20 standard amino acids, colored by their class, for comparison with the visualization on the TrEMBL web page\footnote{\url{https://www.uniprot.org/statistics/TrEMBL}}.

\subsection{Tabular Results}

Table \ref{table-trembl} contains the results on the single protein sequence modeling task ($L=1024$). We report accuracy and perplexity as defined in Appendix \ref{appendix:hyperparameters}:

\begin{table}[h]
\centering
\renewcommand{\arraystretch}{1.4}
\resizebox{0.9\textwidth}{!}{%
\begin{tabular}{ |c|c|c|c|c| } 
\hline
\textbf{Model Type} & \textbf{Set Name} & \textbf{Model} & \textbf{Accuracy} & \textbf{Perplexity} \\
\hline
\multirow{6}{*}{UNI} & \multirow{3}{*}{Test} & Empirical Baseline & 9.92 & 17.80 \\ 
& & Transformer & 30.80 & 9.37\\ 
& & Performer (generalized) & 31.58 & 9.17 \\ 
\cline{2-5}
& \multirow{3}{*}{OOD} & Empirical Baseline & ~9.07 & ~17.93 \\ 
& & Transformer & 19.70 & 13.20 \\ 
& & Performer (generalized) & 18.44 & 13.63 \\
\hline
& & & &\vspace{-4.5mm}\\
\hline
\multirow{6}{*}{BID} & \multirow{3}{*}{Test} & Transformer & 33.32 & 9.22 \\ 
& & Performer (generalized) & 36.09 &  8.36 \\ 
& & Performer (softmax) & 33.00 &  9.24 \\ 
\cline{2-5}
& \multirow{3}{*}{OOD} & Transformer & 25.07 & 12.09 \\ 
& & Performer (generalized) & 24.10 & 12.26 \\ 
& & Performer (softmax) & 23.48 & 12.41 \\ 
\hline
\end{tabular}
}
\vspace{2mm}
\caption{Results on single protein sequence modeling ($L=1024$). We note that the empirical baseline results are applicable to both the unidirectional (UNI) and bidirectional (BID) models.}\label{table-trembl}
\end{table}

\subsection{Attention Matrix Illustration}

In this section we illustrate the attention matrices produced by a Performer model. We focus on the bidirectional case and choose one Performer model trained on the standard single-sequence TrEMBL task for over 500K steps. The same analysis can be applied to unidirectional Performers as well.

We note that while the Transformer model instantiates the attention matrix in order to compute the attention output that incorporates the (queries $Q$, keys $K$, values $V$) triplet (see Eq.~\ref{eq:attnorm} in the main paper), the FAVOR mechanism returns the attention output directly (see Algorithm~\ref{alg:1}). To account for this discrepancy, we extract the attention matrices by applying each attention mechanism twice: once on each original $(Q, K, V)$ triple to obtain the attention output, and once on a modified $(Q, K, V^\circ)$ triple, where $V^\circ$ contains one-hot indicators for each position index, to obtain the attention matrix. The choice of $V^\circ$ ensures that the dimension of the attention output is equal to the sequence length, and that a non-zero output on a dimension $i$ can only arise from a non-zero attention weight to the $i^{th}$ sequence position. Indeed, in the Transformer case, when comparing the output of this procedure with the instantiated attention matrix, the outputs match.

\textbf{Attention matrix example.} We start by visualizing the attention matrix for an individual protein sequence. We use the BPT1\_BOVIN protein sequence\footnote{\url{https://www.uniprot.org/uniprot/P00974}}, one of the most extensively studied globular proteins, which contains 100 amino acids. In Figure \ref{fig:attention_matrices}, we show the attention matrices for the first 4 layers. Note that many heads show a \textit{diagonal} pattern, where each node attends to its neighbors, and some heads show a \textit{vertical} pattern, where each head attends to the same fixed positions. These patterns are consistent with the patterns found in Transformer models trained on natural language \citep{kovaleva2019revealing}. In Figure \ref{fig:model_view} we highlight these attention patterns by focusing on the first 25 tokens, and in Figure \ref{fig:attention_heads}, we illustrate in more detail two attention heads.

\textbf{Amino acid similarity.} Furthermore, we analyze the amino-acid similarity matrix estimated from the attention matrices produced by the Performer model, as described in Vig et al. \citep{bertology}. We aggregate the attention matrix across 800 sequences. The resulting similarity matrix is illustrated in Figure \ref{figure:amino-acid-similarity}. Note that the Performer recognizes highly similar amino acid pairs such as (D, E) and (F, Y).

\begin{figure}[h]
  \centering
  \includegraphics[width=0.99\textwidth]{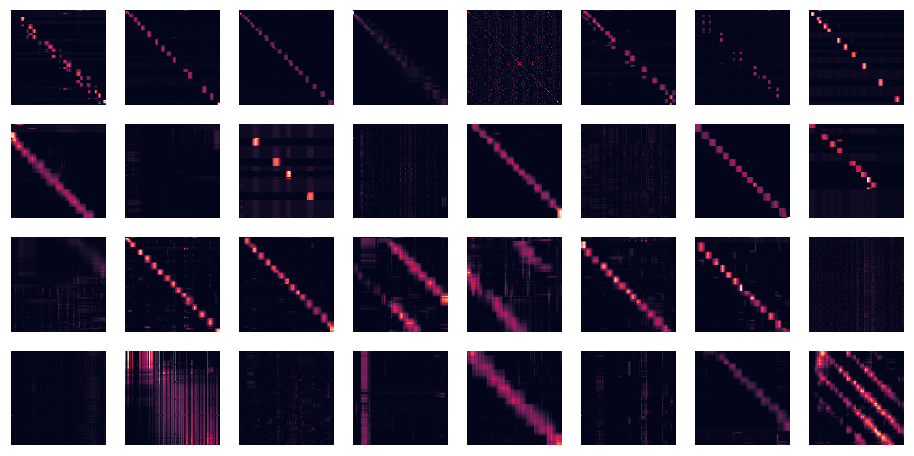}
  \vspace{2mm}
  \caption{\small{We show the attention matrices for the first 4 layers and all 8 heads (each row is a layer, each column is head index, each cell contains the attention matrix across the entire BPT1\_BOVIN protein sequence). Note that many heads show a \textit{diagonal} pattern, where each node attends to its neighbors, and some heads show a \textit{vertical} pattern, where each head attends to the same fixed positions.}}
  \label{fig:attention_matrices}
\end{figure}

\begin{figure}[h]
  \centering
  \includegraphics[width=0.49\textwidth]{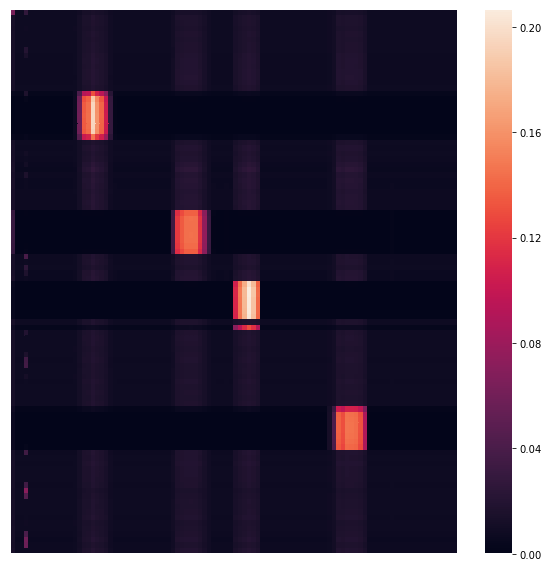}
  \includegraphics[width=0.49\textwidth]{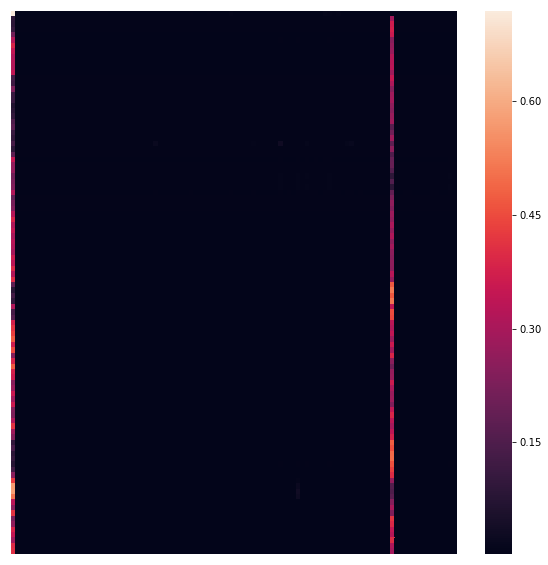}
  \vspace{2mm}
  \caption{\small{We illustrate in more detail two attention heads. The sub-figures correspond respectively to: \textbf{(1)} Head 1-2 (second layer, third head), \textbf{(2)} Head 4-1 (fifth layer, second head). Note the block attention in Head 1-2 and the vertical attention (to the start token (`M') and the 85th token (`C')) in Head 4-1.}}
  \label{fig:attention_heads}
\end{figure}

\newpage

\begin{figure}[ht]
    \centering
  \includegraphics[height = 0.6\linewidth, width=0.9\linewidth]{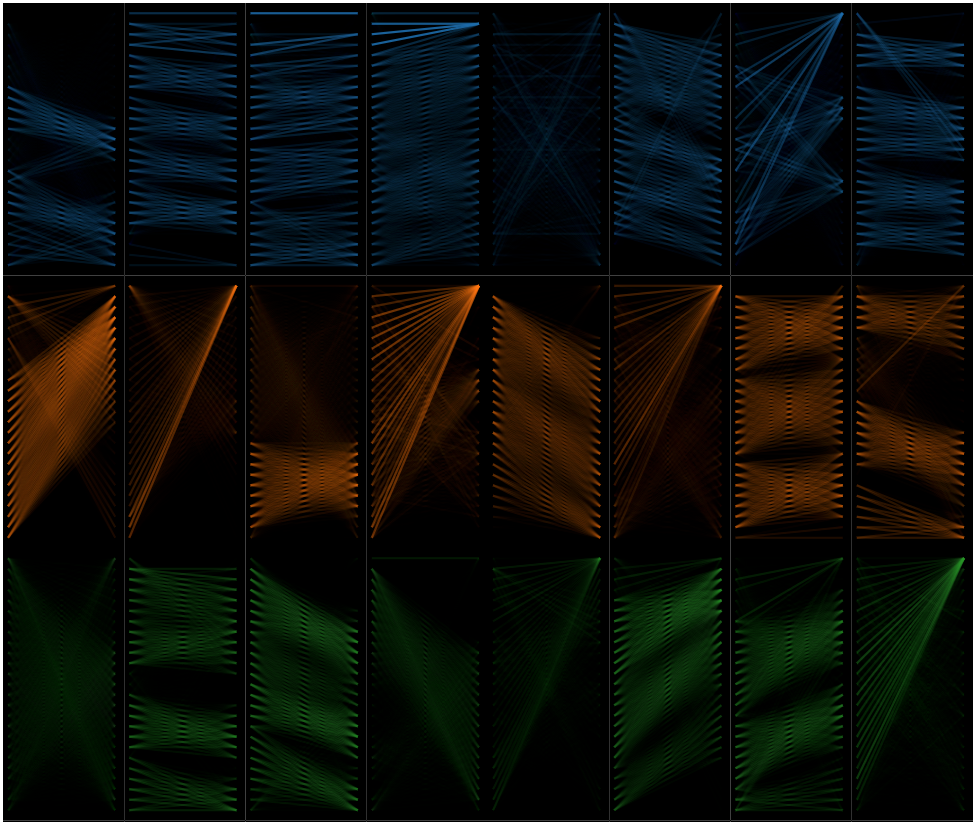}
  \vspace{2mm}
  \caption{\small{We highlight the attention patterns by restricting our attention to the first 25 tokens (note that we do not renormalize the attention to these tokens). The illustration is based on Vig et al. \protect\citep{vig2019multiscale, analyzing_attention}. Note that, similar to prior work on protein Transformers \protect\citep{progen}, the attention matrices include both local and global patterns.}}
  \label{fig:model_view}
\end{figure}

\begin{figure}[ht]
  \centering
  \includegraphics[width=0.49\textwidth]{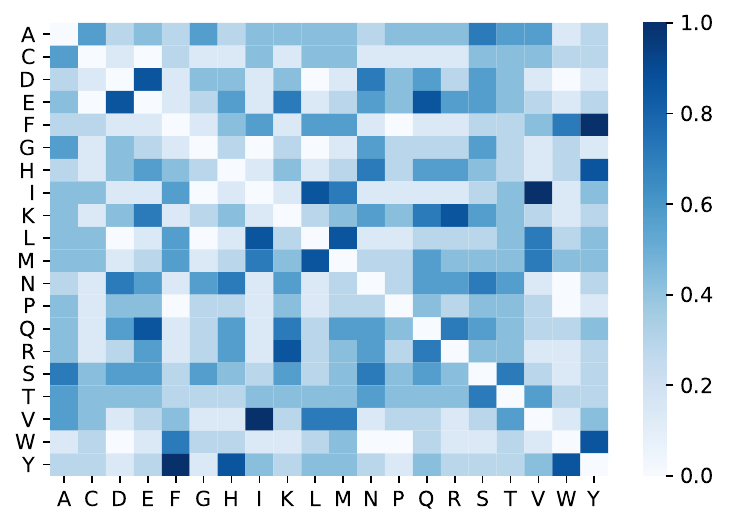}
  \includegraphics[width=0.49\textwidth]{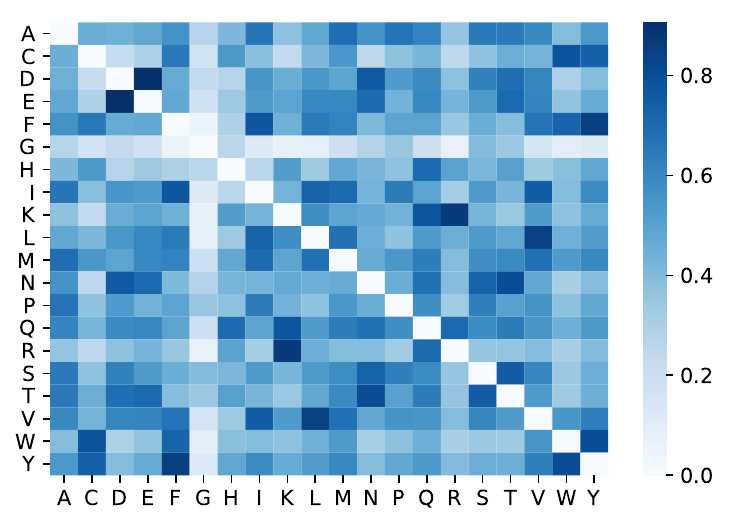}
  \vspace{1mm}
  \caption{\small{Amino acid similarity matrix estimated from attention matrices aggregated across a small subset of sequences, as described in Vig et al. \protect\citep{bertology}. The sub-figures correspond respectively to: \textbf{(1)} the normalized BLOSUM matrix, \textbf{(2)} the amino acid similarity estimated via a trained Performer model. Note that the Performer recognizes highly similar amino acid pairs such as (D, E) and (F, Y).}}
  \label{figure:amino-acid-similarity}
\end{figure}

\clearpage

\section{Extended approximation and comparison results}
\label{appendix:extended_approx}

\subsection{Backwards Compatibility - Error Propagation}
Although mentioned previously (Sec. \ref{subsec:approx_error_compatibility}) that the Performer with additional finetuning is backwards compatible with the Transformer, we demonstrate below in Fig. \ref{fig:appendix_approx} that error propagation due to non-attention components of the Transformer is one of the primary reasons that pretrained Transformer weights cannot be immediately used for inference on the corresponding Performer.

\begin{figure}[h]
  \centering
  \includegraphics[width=0.75\textwidth]{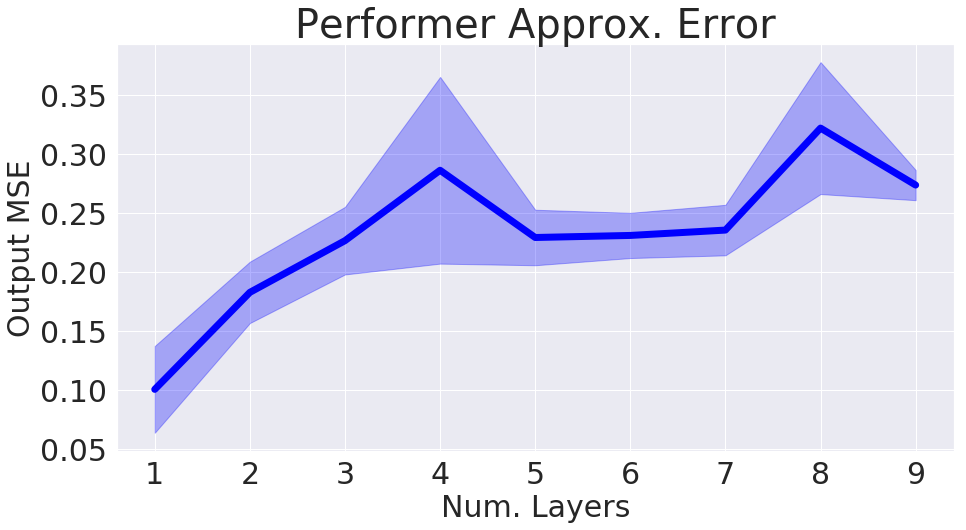}
  \caption{Output approximation errors between a vanilla Transformer and a Performer (with orthogonal features) for varying numbers of layers.}
  \label{fig:appendix_approx}
\end{figure}

\subsection{Approximate Softmax - Extended Properties}
\label{subsec:redraw}
\label{subsec:softmax_unidirectional}
\label{subsec:unstable_trig}
We show the following properties of our softmax approximation, in Fig. \ref{fig:appendix_redraw}:

\textbf{Redrawing:} While the benefits of redrawing features was shown in Subsec. \ref{subsec:softmax_approx_transformer} of the main body of the paper, we also demonstrate its benefits when there are multiple layers with large scale (16x16 TPU-v2) training.

\textbf{Unidirectional:} While we have shown on TrEMBL that Performer with generalized ReLU attention outperforms softmax, we also show that approximate softmax attention can still be a solid choice, for example on ImageNet64 (U). After 100K steps of training, the Performer-ReLU, Performer-Softmax, and Performer-Softmax (SMREG) variants achieve respectively, 3.67, 3.69, 3.67 BPD.

\textbf{Instability of Trigonometric Features:} We see the full view of the unstable training curve when using Trigonometric softmax.

\vspace{2mm}

\begin{figure}[h]
  \centering
  \includegraphics[width=0.99\textwidth]{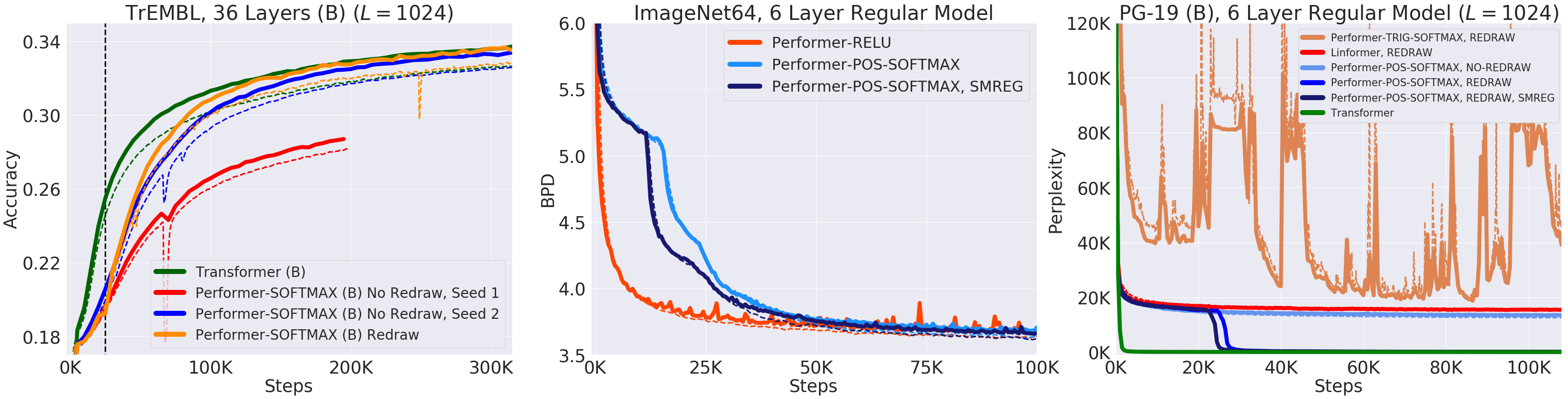}
  \caption{Best viewed zoomed in. \textbf{Left:} The importance of redrawing features. If redrawing is not used, an "unlucky" set of random features may cause training degradation, shown by the early-stopped curve with Seed 1, while a `lucky' set of random features may cause no issue, shown by the curve with Seed 2. Redrawing allows the training to correct itself, as seen at the black vertical line. \textbf{Middle:} Using the same 8x8 TPU-v2 compute and same 6-layer standard model, approximate softmax with positive features achieves the same result as generalized ReLU attention. \textbf{Right:} Zoomed out view of right subfigure of Fig. \ref{fig:backward_compatibility}, showing that Trigonometric softmax causes very unstable training behaviors.
  }
  \label{fig:appendix_redraw}
  \label{fig:im64_relu_softmax}
  \label{fig:pg19_full_softmax}
\end{figure}


\subsection{Generalized Attention}
\label{subsec:appendix_generalized_attention}
We investigated Generalized Attention mechanisms (mentioned in Sec. \ref{sec:gka}) on TrEMBL when $L=512$ for various kernel functions. This is similar to \citep{tsai2019transformer} which also experiments with various attention kernels for natural language. Using hyperparameter sweeps across multiple variables in FAVOR, we compared several kernels and also renormalization on/off (Fig. \ref{fig:attention_comparisons_2x2} and Fig. \ref{fig:attention_comparisons_4x4}), where $\mathrm{Renormalize}$ corresponds to applying $\mathbf{D}^{-1}$ operator in attention, as for the standard mechanism, though we noticed that disabling it does not necessarily hurt accuracy) to produce the best training configuration for the Performer. We note that the effective batch size slightly affects the rankings (as shown by the difference between 2x2 and 4x4 TPU runs) - we by default use the generalized ReLU kernel with other default hyperparameters shown in Appendix \ref{appendix:hyperparameters}, as we observed that they are empirically optimal for large batch size runs (i.e. 8x8 or 16x16 TPU's).

\begin{figure}[h]
  \centering
  \includegraphics[width=0.99\textwidth]{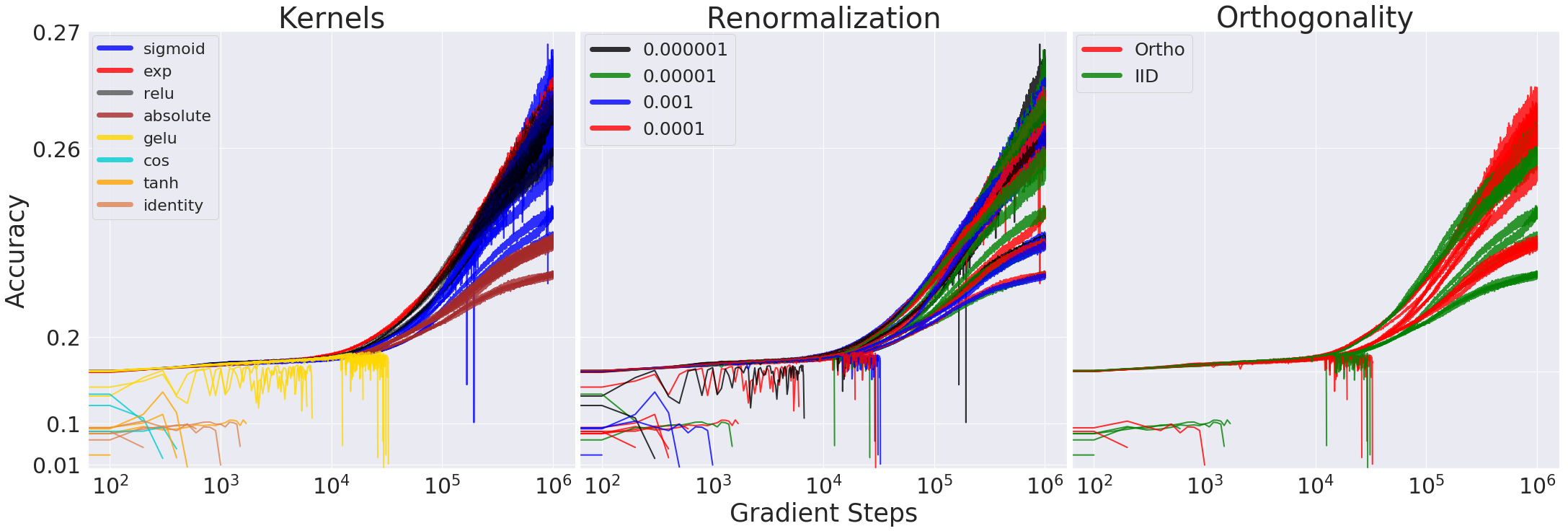}
  \caption{To emphasize the highest accuracy runs but also show the NaN issues with certain kernels which caused runs to stop early, we set both x and y axes to be log-scale. We tested kernels defined by different functions $f$ (see: Sec. \ref{sec:gka}): sigmoid, exponential, ReLU, absolute, gelu, cosine (original softmax approximation), tanh, and identity. All training runs were performed on 2x2 TPU-v2's, 128 batch size per device.}
  \label{fig:attention_comparisons_2x2}
\end{figure}

\begin{figure}[h]
  \centering
  \includegraphics[width=0.99\textwidth]{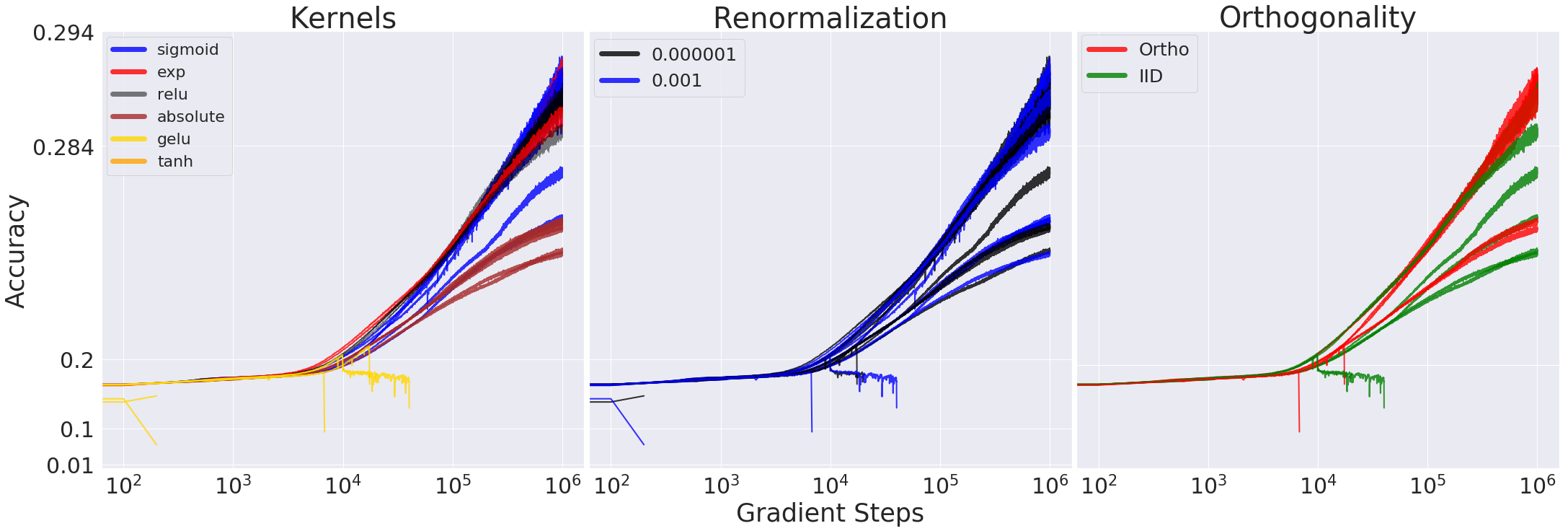}
  \caption{We also performed a similar setup as Fig. \ref{fig:attention_comparisons_2x2} for 4x4 TPU-v2's.}
  \label{fig:attention_comparisons_4x4}
\end{figure}

\subsection{Comparison with Linear Transformer}
\label{appendix:linear_transformer}

We use the attention implementation of the Linear Transformer from \citep{trans-rnns}, which mainly involves setting our feature map $\phi(x) = \text{elu}(x) + 1$, where $\text{elu}(x)$ is the shifted-eLU function from \citep{elu}. 

\begin{figure}[h]
  \centering
  \includegraphics[width=0.99\textwidth]{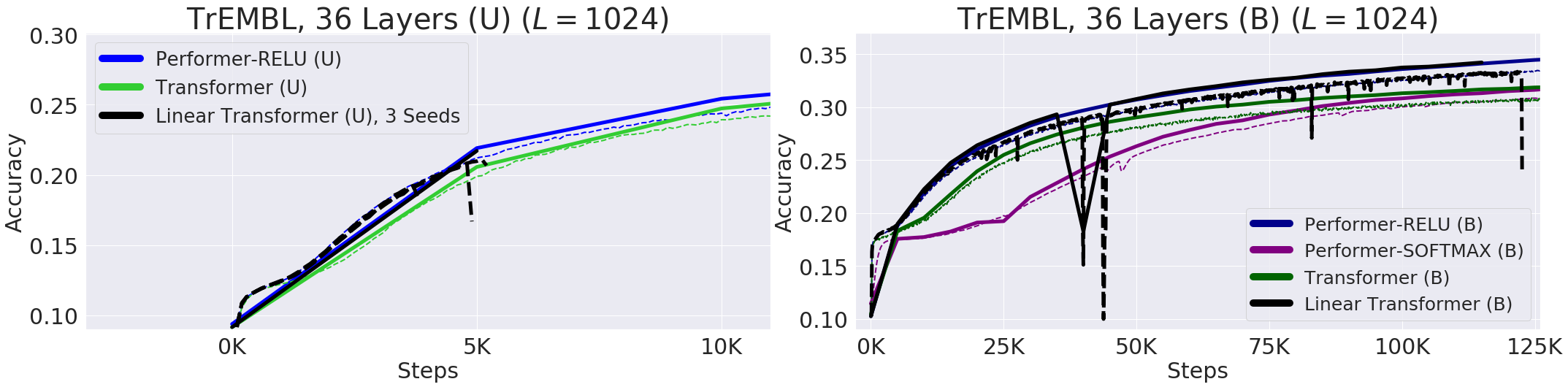}
  \caption{\textbf{Left:} In the unidirectional 36-ProGen setting, we ran 3 seeds of the Linear Transformer, and found that all 3 seeds produced exploding gradients very early on, stopping the training run. \textbf{Right:} The Linear Transformer in the bidirectional setting also produced an exploding gradient in the middle of training, near 125K steps. Exploding gradients can be evidenced by the sharp drop in train accuracy right before a NaN error.}
  \label{fig:linear_transformer}
\end{figure}

For the sake of fairness and to prevent confounding results, while \citep{trans-rnns} also uses the GeLU nonlinearity for the MLPs in the Linear Transformer, we instead use the original ReLU nonlinearity. We also used the exact same training hyperparameters as Performer-ReLU on our exact ProGen setting from Fig. \ref{fig:big_benchmarking}. Ultimately, we empirically found that the Linear Transformer possessed numerical instability during training via unstable training curves, \textbf{ultimately stopping training by producing exploding gradients (NaNs)} (Fig. \ref{fig:linear_transformer}).

\subsection{Long Range Arena}
Performers are compared against many additional (scalable and not scalable) methods not included in our paper: \textit{Local Attention}, \textit{Sparse Attention}, \textit{Longformer}, \textit{Sinkhorn Transformer}, \textit{Synthesizer}, \textit{Big Bird} and the aforementioned \textit{Linear Transformer} on challenging long range context tasks in the Long Range Arena \citep{lra}, with Fig. \ref{fig:lra_figure} displaying the original paper's results. Performers obtain the largest LRA (Long Range Arena) score among all tested \textbf{scalable} Transformers methods (which we define by having speed of > 100 examples/sec). 

Tasks used for comparison include: \textbf{(1)} a longer variation of the standard ListOps task proposed in \citep{listops}, \textbf{(2)} byte-level text classification using real-world data, \textbf{(3)} byte-level document retrieval, \textbf{(4)} image classification on sequences of pixels, and \textbf{(5)} Pathfinder task (long-range spatial dependency problem). In the Long Range Arena paper, the authors found that all models do not learn anything on Path-X task (denoted by FAIL), contrary to the Pathfinder task, which shows that increasing the sequence length can cause seriously difficulties for model training.
\vspace{0.2cm}
\begin{figure}[h]
  \centering
  \includegraphics[width=0.99\textwidth]{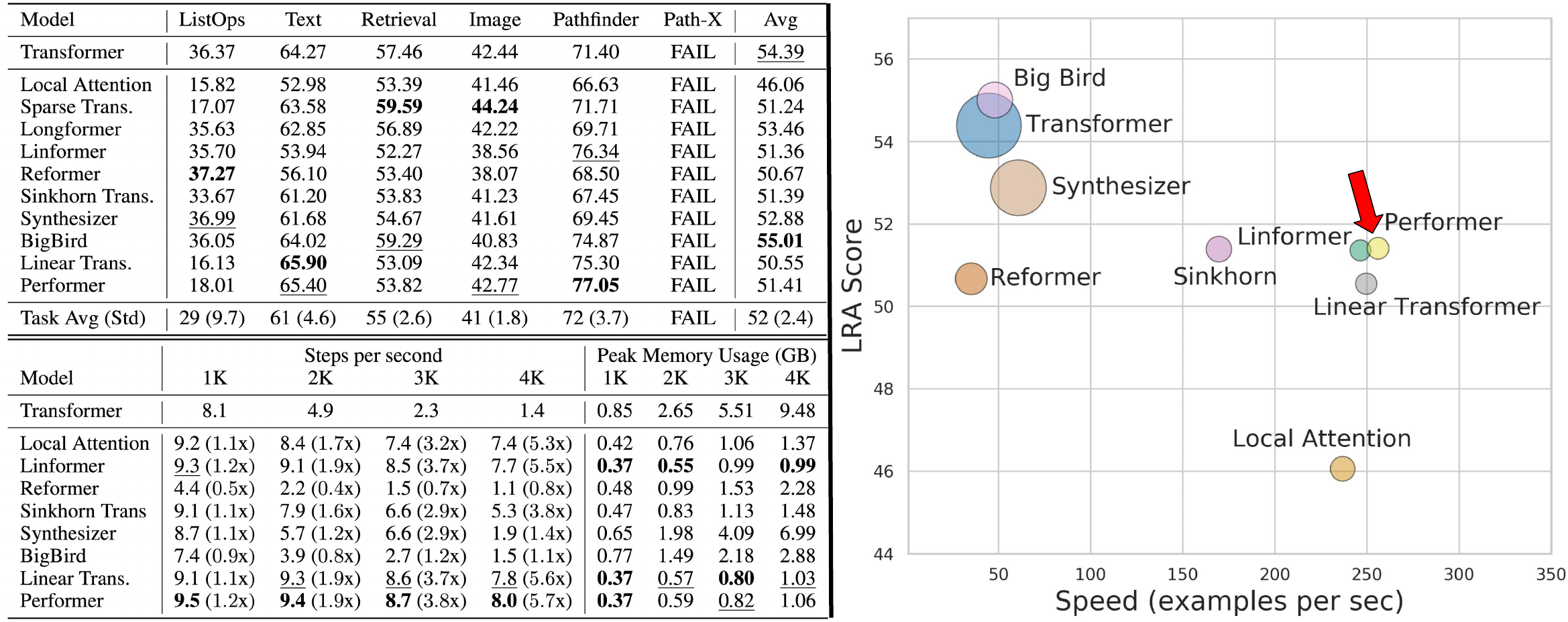}
  \vspace{0.15cm}
  \caption{\textbf{Upper Table:} Results on Long-Range Arena benchmark. Best model is in boldface and second best is underlined. \textbf{Lower Table:} Benchmark results of all X-former models with a consistent batch size of 32 across all models. The authors report relative speed increase/decrease in comparison with the vanilla Transformer in brackets besides the steps per second. Memory usage refers to per device memory usage across each TPU device. Benchmarks are run on 4x4 TPU-v3 chips. \textbf{Right Fig:} Performance (y-axis), speed (x-axis), and memory footprint (size of the circles) of different models.} 
  \label{fig:lra_figure}
\end{figure}

\newpage 

\section{Computation costs - Extended results}

In this subsection, we empirically measure computational costs in terms wall clock time on forward and backward passes for three scenarios in Fig. \ref{fig:appendix_runtime}:
\begin{enumerate}
\item Performer, with varying number of layers. We show that our method can scale up to (but not necessarily limited to) even 20 layers.
\item Attention time complexities when comparing standard attention (from Transformer) and FAVOR (from Performer). Note that the maximum memory size here is not reflective of the maximum memory size in an actual model (shown below), as this benchmark requires computing explicit tensors (causing memory increases) in Jax, while a model does not.
\item Time complexities when comparing the Transformer and Performer models. "X" (OPT) denotes the maximum possible speedup achievable, when attention simply returns the $\mathbf{V}$-vector, showing that the Performer is nearly optimal. We see that the maximum possible power of 2 length allowed on a V100 GPU (16GB) is $2^{15} = 32768$ using regular dimensions.
\end{enumerate}

Since some of the computational bottleneck in the Transformer may originate from the extra feed-forward layers \citep{reformer}, we also benchmark the ``Small" version, i.e. $(n_{heads}, n_{layers}, d_{ff}, d) = (1,6,64,64)$ as well, when the attention component is the dominant source of computation and memory. We remind the reader that the ``Regular" version consists of  $(n_{heads}, n_{layers}, d_{ff}, d) = (8,6,2048,512)$. 

\vspace{0.1cm}
\label{appendix:computation_costs_bidirectional}

\begin{figure}[h]
  \includegraphics[width=1.0\textwidth]{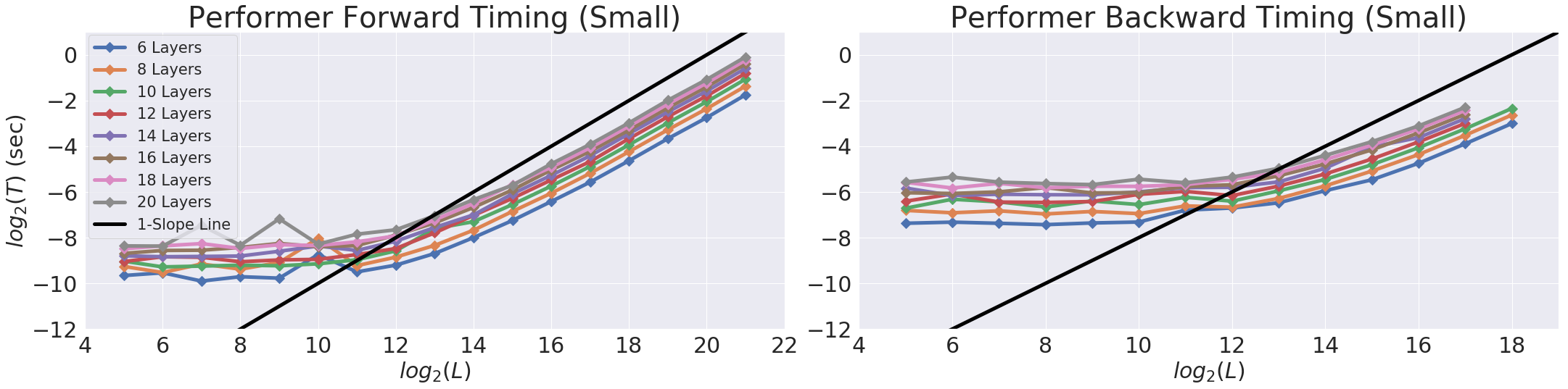}
  \includegraphics[width=1.0\textwidth]{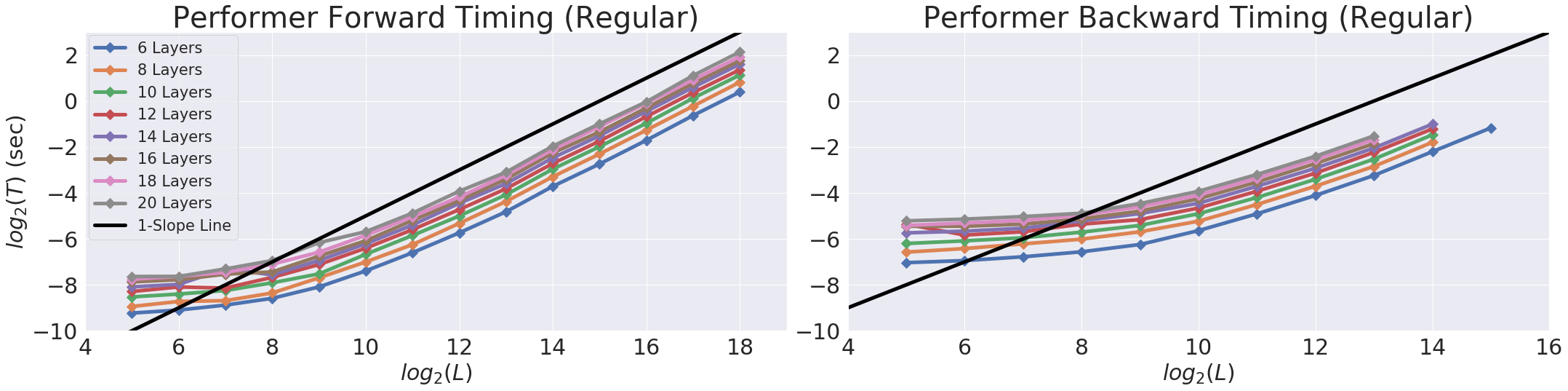}
  \includegraphics[width=1.0\textwidth]{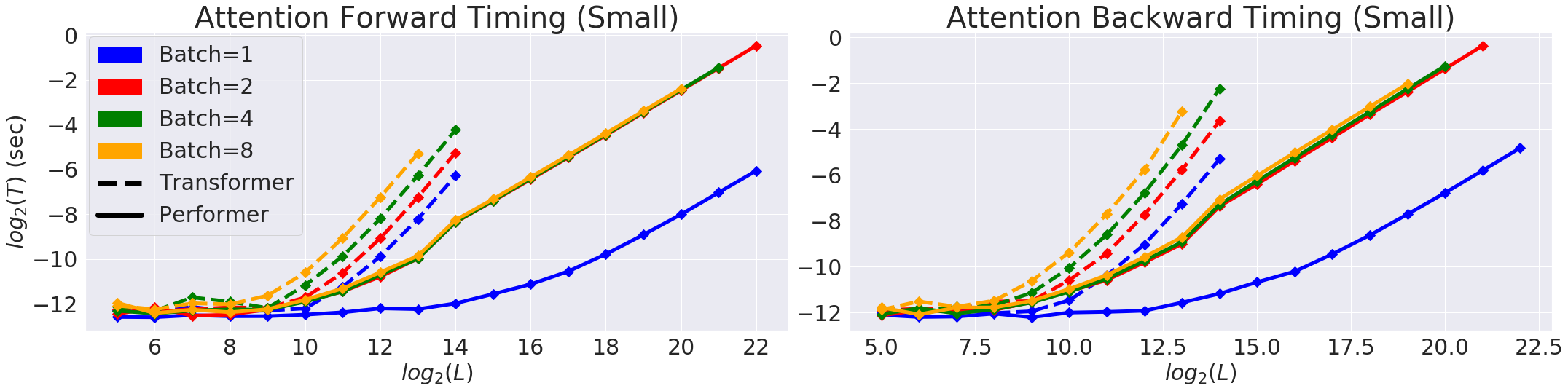}
  \includegraphics[width=1.0\textwidth]{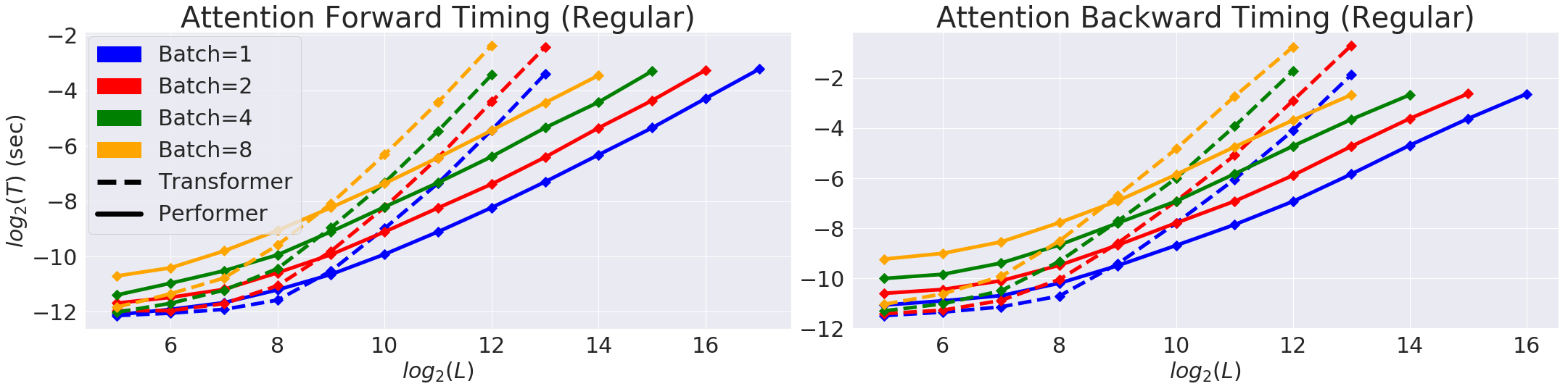}

  \caption{Captions (1) and (2) for each 2x2 subfigure mentioned above.}
  \label{fig:appendix_runtime}
\end{figure}

\clearpage
\begin{figure}[ht]
  \includegraphics[width=1.0\textwidth]{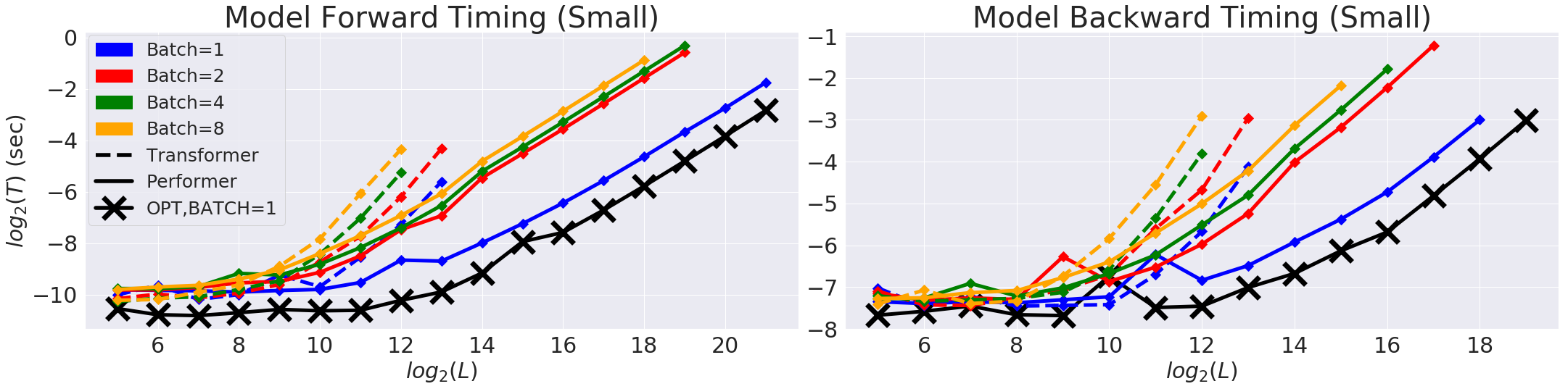}
  \includegraphics[width=1.0\textwidth]{img/model_regular_opt.png}

  \caption{Caption (3) for this 2x2 subfigure mentioned above.}
  \label{fig:appendix_runtime_2}
\end{figure}


\clearpage

\section{Theoretical results}
\label{appendix:theoretical_results}
We provide here the proofs of all theoretical results presented in the paper.

\subsection{Proof of Lemma \ref{pos_random_features_lemma}}

\begin{proof}
We first deduce that for any $\va, \vb \in \R^d$
\begin{equation*}
\mathrm{SM}(\mathbf{x},\mathbf{y})=
    \exp (\vx^\top \vy) = \exp (- \| \vx \|^2 / 2) \cdot \exp (\| \vx + \vy \|^2 / 2) \cdot \exp (- \| \vy \|^2 / 2) .
\end{equation*}
Next, let $\vw \in \R^d$. We use the fact that
\begin{equation*}
    (2 \pi)^{-d/2} \int \exp (- \| \vw - \vc \|_2^2 / 2) d \vw = 1
\end{equation*}
for any $\vc \in \mathbb{R}^d$ and derive:
\begin{align*}
    \exp (\| \vx &+ \vy \|^2 / 2) = (2 \pi)^{- d / 2} \exp (\| \vx + \vy \|^2 / 2) \int  \exp (- \| \vw - (\vx + \vy) \|^2 / 2) d \vw \\
    &= (2 \pi)^{- d / 2} \int  \exp (- \| \vw \|^2 / 2 + \vw^\top (\vx + \vy) - \| \vx + \vy \|^2 / 2 + \| \vx + \vy \|^2 / 2) d \vw \\
    &= (2 \pi)^{- d / 2} \int  \exp (- \| \vw \|^2 / 2 + \vw^\top (\vx + \vy) ) d \vw \\
    &= (2 \pi)^{- d / 2} \int  \exp (- \| \vw \|^2 / 2 ) \cdot \exp (\vw^\top \vx) \cdot \exp (\vw^\top \vy) d \vw \\
    &= \E_{\omega \sim \mathcal{N} (\mathbf{0}_d, \mathbf{I}_{d})} [\exp (\omega^\top \vx) \cdot \exp (\omega^\top \vy)] .
\end{align*}
That completes the proof of the first part of the lemma. An identity involving hyperbolic cosine function is implied by the fact that for every $\mathbf{u} \in \mathbb{R}^{d}$ and $\omega \sim \mathcal{N}(0, \mathbf{I}_{d})$ the following is true:
\begin{equation}
\mathbb{E}[\exp(\omega^{\top}\mathbf{u})] = \sum_{i=0}^{\infty} \frac{\mathbb{E}[(\omega^{\top}\mathbf{u})^{2i}]}{(2i)!} = 
\frac{1}{2} \sum_{i=0}^{\infty}
\frac{\mathbb{E}[(\omega^{\top}\mathbf{u})^{2i}] + \mathbb{E}[(-\omega^{\top}\mathbf{u})^{2i}]}{(2i)!}.
\end{equation}
The cancellation of the odd moments $\mathbb{E}[(\omega^{\top}\mathbf{u})^{2i+1}]$ follows directly from the fact that $\omega$ is taken from the isotropic distribution (i.e. distribution with pdf function constant on each sphere).
That completes the proof.
\end{proof}

\subsection{Proof of Lemma \ref{mse-lemma}}
\begin{proof}
Denote: $\mathbf{z}=\mathbf{x}+\mathbf{y}$ and $\Delta=\mathbf{x}-\mathbf{y}$.
Note that by using standard trigonometric identities (and the fact that the variance of the sum of independent random variables is the sum of variances of those random variables), we can get the following for $\omega \sim \mathcal{N}(0, \mathbf{I}_{d})$:
\begin{equation}
\mathrm{MSE}(\widehat{\mathrm{SM}}^{\mathrm{trig}}_{m}(\mathbf{x}, \mathbf{y})) = \frac{1}{m}\exp(\|\mathbf{x}\|^{2}+\|\mathbf{y}\|^{2}) \mathrm{Var}(\cos(\omega^{\top}\Delta)).     
\end{equation}
Using the fact that (see: Lemma 1 in \citep{ort}; note that in that lemma they use notation: $z$ for what we denote as: $\|\Delta\|$):
\begin{equation}
\mathrm{Var}(\cos(\omega^{\top}\Delta))=\frac{1}{2}(1-\exp(-\|\Delta\|^{2}))^{2},    
\end{equation}
we obtain:
\begin{align}
\begin{split}
\mathrm{MSE}(\widehat{\mathrm{SM}}^{\mathrm{trig}}_{m}(\mathbf{x}, \mathbf{y})) = \frac{1}{2m}\exp(\|\mathbf{x}\|^{2}+\|\mathbf{y}\|^{2})(1-\exp(-\|\Delta\|^{2}))^{2} = \\
\frac{1}{2m} \exp(\|\mathbf{z}\|^{2})\mathrm{SM}^{-2}(\mathbf{x},\mathbf{y})(1-\exp(-\|\Delta\|^{2}))^{2},
\end{split}
\end{align}
which completes the first part of the proof.
To obtain the formula for: $\mathrm{MSE}(\widehat{\mathrm{SM}}^{\mathrm{+}}_{m}(\mathbf{x}, \mathbf{y}))$
notice first that:
\begin{equation}
\label{neat-fact}
\mathbb{E}_{\omega \sim \mathcal{N}(0, \mathbf{I}_{d})}[\exp(\omega^{\top}\mathbf{z})] = \exp(\frac{\|\mathbf{z}\|^{2}}{2}).    
\end{equation}
The above immediately follows from the fact that positive random feature maps provide unbiased estimation of the softmax-kernel, thus the following is true:
\begin{equation}
\mathrm{SM}(\mathbf{x}, \mathbf{y}) = \exp(-\frac{\|\mathbf{x}\|^{2}+\|\mathbf{y}\|^{2}}{2})\mathbb{E}_{\omega \sim \mathcal{N}(0, \mathbf{I}_{d})}[\exp(\omega^{\top}\mathbf{z})].    
\end{equation}

Therefore we obtain:
\begin{align}
\begin{split}
\mathrm{MSE}(\widehat{\mathrm{SM}}^{\mathrm{+}}_{m}(\mathbf{x}, \mathbf{y})) = \frac{1}{m} \exp(-(\|\mathbf{x}\|^{2} + \|\mathbf{y}\|^{2}))\mathrm{Var}(\exp(\omega^{\top}\mathbf{z})) = \\
\frac{1}{m} \exp(-(\|\mathbf{x}\|^{2} + \|\mathbf{y}\|^{2}))
\left(\mathbb{E}[\exp(2\omega^{\top}\mathbf{z})] - (\mathbb{E}[\exp(\omega^{\top}\mathbf{z})])^{2}\right) = \\ 
\frac{1}{m} \exp(-(\|\mathbf{x}\|^{2} + \|\mathbf{y}\|^{2}))
(\exp(2\|\mathbf{z}\|^{2}) - \exp(\|\mathbf{z}\|^{2})),
\end{split}  
\end{align}
where the last equality follows from Equation \ref{neat-fact}.
Therefore we have:
\begin{align}
\begin{split}
\mathrm{MSE}(\widehat{\mathrm{SM}}^{\mathrm{+}}_{m}(\mathbf{x}, \mathbf{y})) = \frac{1}{m} \exp(-(\|\mathbf{x}\|^{2} + \|\mathbf{y}\|^{2})) \exp(\|\mathbf{z}\|^{2})(\exp(\|\mathbf{z}\|^{2}) - 1) = \\  
\frac{1}{m}\exp(\|\mathbf{z}\|^{2})\mathrm{SM}^{2}(\mathbf{x}, \mathbf{y})(1-\exp(-\|\mathbf{z}\|^{2})).
\end{split}
\end{align}

Finally, 
\begin{align}
\begin{split}
\mathrm{MSE}(\widehat{\mathrm{SM}}_{m}^{\mathrm{hyp+}}(\mathbf{x}, \mathbf{y})) = \frac{1}{4m}\exp(-\frac{\|\mathbf{x}\|^{2}+\|\mathbf{y}\|^{2}}{2})^{2}(\mathrm{Var}(\exp(\omega^{\top}\mathbf{z}))
+\mathrm{Var}(\exp(-\omega^{\top}\mathbf{z})) + \\
2\mathrm{Cov}(\exp(\omega^{\top}\mathbf{z})), \exp(-\omega^{\top}\mathbf{z})))) = 
\frac{1}{4m}\exp(-\frac{\|\mathbf{x}\|^{2}+\|\mathbf{y}\|^{2}}{2})^{2}(2\mathrm{Var}(\exp(\omega^{\top}\mathbf{z}))+\\2\mathrm{Cov}(\exp(\omega^{\top}\mathbf{z})), \exp(-\omega^{\top}\mathbf{z}))))) = 
\frac{1}{2m}\exp(-(\|\mathbf{x}\|^{2}+\|\mathbf{y}\|^{2}))\\
(\mathrm{Var}(\exp(\omega^{\top}\mathbf{z}))+ 1 - (\mathbb{E}[\exp(\omega^{\top}\mathbf{z})])^{2}) = 
\frac{1}{2m}\exp(-(\|\mathbf{x}\|^{2}+\|\mathbf{y}\|^{2}))\\
(\exp(2\|\mathbf{z}\|^{2})-\exp(\|\mathbf{z}\|^{2})+1-\exp(\|\mathbf{z}\|^{2})) = 
\frac{1}{2m}\exp(-(\|\mathbf{x}\|^{2}+\|\mathbf{y}\|^{2}))(\exp(\|\mathbf{z}\|^{2})-1)^{2} \\
=\frac{1}{2}(1-\exp(-\|\mathbf{z}\|^{2}))\mathrm{MSE}(\widehat{\mathrm{SM}}^{\mathrm{+}}_{m}(\mathbf{x}, \mathbf{y})).
\end{split}
\end{align}
In the chain of equalities above we used the fact that random variables $\mathrm{exp}(\omega^{\top}\mathbf{z})$ and 
$\mathrm{exp}(-\omega^{\top}\mathbf{z})$ have the same distribution. This is true since $\omega$ and $-\omega$ have the same distribution ($\omega$ is Gaussian).
That completes the proof.
\end{proof}

\subsection{Proof of Theorem \ref{reg-theorem}}
\begin{proof}
Let $\mathbf{x}, \mathbf{y} \in \mathbb{R}^{d}$ be respectively a query/key.
Note that from the definition of $\mathrm{SMREG}(\mathbf{x}, \mathbf{y})$ we have
for $\mathbf{z} = \mathbf{x}+\mathbf{y}$:
\begin{equation}
\mathrm{SMREG}(\mathbf{x}, \mathbf{y}) = 
\exp(-\frac{\|\mathbf{x}\|^{2}+\|\mathbf{y}\|^{2}}{2})
\sum_{k=0}^{\infty}\frac{1}{(2k)!}\|\mathbf{z}\|^{2k}d^{k}\mathbb{E}_{\omega \sim \mathcal{N}(0, \mathbf{I}_{d})}[(\frac{\omega}{\|\omega\|_{2}}\mathbf{e}_{1})^{2k}],
\end{equation}
where $\mathbf{e}_{1} \overset{\mathrm{def}}{=} (1,0,...,0)^{\top} \in \mathbb{R}^{d}$. To obtain the above we used the fact that $\mathcal{N}(0, \mathbf{I}_{d})$ is isotropic (that in particular implies zeroing of the even terms in the Taylor expansion).

Let us denote: $A(k, d) \overset{\mathrm{def}}{=} \mathbb{E}_{\omega \sim \mathcal{N}(0, \mathbf{I}_{d})}[(\frac{\omega}{\|\omega\|_{2}}\mathbf{e}_{1})^{2k}]$. It turns out that:
\begin{equation}
A(2k, d) = \frac{(2k-1)!!}{(d+2k-2)(d+2k-4) \cdot ... \cdot d}.    
\end{equation}
The proof of that fact can be found in the supplement of \citep{geom}, yet we provide it below for completeness and the convenience of the Reader:
\begin{lemma}
Expression $A(2k, d)$ satisfies the following for $k \in \mathbb{N}$ :
\begin{equation}
A(2k, d) = \frac{(2k-1)!!}{(d+2k-2)(d+2k-4) \cdot ... \cdot d}.     
\end{equation}
\end{lemma}
\begin{proof}
Note first that for $d \geq 2$ the density function $p_{d}(\theta)$ of the angle between a vector $\mathbf{r} \in \mathbb{R}^{d}$ chosen uniformly at random from the unit sphere and $\mathbf{e}_{1}$ is given by the following formula:
\begin{equation}
p_{d}(\theta) = \frac{\sin^{d-2}(\theta)}{\int_{0}^{\pi}\sin^{d-2(\theta)}d\theta}.    
\end{equation}
Let us denote: $F(k, d) \overset{\mathrm{def}}{=} \int_{0}^{\pi}\cos^{k}(\theta)\sin^{d}(\theta)d\theta$.
Using partial integration, we get:
\begin{align}
\begin{split}
\int_{0}^{\pi}\cos^{k}(\theta)\sin^{d}(\theta)d\theta=
\int_{0}^{\pi}\cos^{k-1}(\theta)\sin^{d}(\theta)(\sin(\theta))^{\prime}d\theta= \\
\cos^{k-1}(\theta)\sin^{d+1}(\theta)|^{\pi}_{0} - 
\int_{0}^{\pi}\sin(\theta)
(
(k-1)\cos^{k-2}(\theta)(-\sin(\theta))\sin^{d}(\theta)+\\
d\cos^{k}(\theta)\sin^{d-1}(\theta)
)d\theta.
\end{split}
\end{align}
Thus we conclude that: $F(k, d) = \frac{k-1}{d+1}F(k-2, d+2)$.
Therefore we have: 
\begin{equation}
F(2k, d) = \frac{(2k-1)!!}{(d+1)(d+3)\cdot...\cdot(d+2k-1)}
\int_{0}^{\pi}\sin^{d+2k}(\theta) d\theta.
\end{equation}
We again conduct partial integration and get:
\begin{align}
\begin{split}
\int_{0}^{\pi}\sin^{d}(\theta)d\theta = -\frac{1}{d}\sin^{d-1}(\theta)\cos(\theta)|^{\pi}_{0} + \\
\frac{d-1}{d}\int_{0}^{\pi}\sin^{d-2}(\theta)d\theta = 
\frac{d-1}{d}\int_{0}^{\pi}\sin^{d-2}(\theta)d\theta.
\end{split}
\end{align}
Therefore we conclude that:
\begin{align}
\begin{split}
A(2k, d) = \frac{1}{\frac{d-3}{d-2}\frac{d-5}{d-4}\cdot ...}
\frac{(2k-1)!!}{(d-1)(d+1)\cdot...\cdot(d+2k-3)}
\frac{d+2k-3}{d+2k-2}\frac{d+2k-5}{d+2k-4} \cdot ....
= \\ \frac{(2k-1)!!}{(d+2k-2)(d+2k-4) \cdot ... \cdot d},
\end{split}
\end{align}
which completes the proof.
\end{proof}
Applying the above lemma, we get:
\begin{align}
\begin{split}
\mathrm{SMREG}(\mathbf{x}, \mathbf{y}) = 
\exp(-\frac{\|\mathbf{x}\|^{2}+\|\mathbf{y}\|^{2}}{2})
\sum_{k=0}^{\infty}\frac{1}{(2k)!}
\|\mathbf{z}\|^{2k}d^{k}\frac{(2k-1)!!}{(d+2k-2)(d+2k-4)\cdot...\cdot d} \\ = 
\exp(-\frac{\|\mathbf{x}\|^{2}+\|\mathbf{y}\|^{2}}{2})
\sum_{k=0}^{\infty} \frac{w^{k}}{k!}f(k, d),
\end{split}
\end{align}
where $w = \frac{\|\mathbf{z}\|^{2}}{2}$
and $f(k, d) = \frac{d^{k}}{(d+2k-2)(d+2k-4)\cdot...\cdot d}$.

Thus we obtain:
\begin{equation}
\frac{\mathrm{SMREG}(\mathbf{x}, \mathbf{y})}{\mathrm{SM}(\mathbf{x}, \mathbf{y})} = 
e^{-w}\sum_{k=0}^{\infty} \frac{w^{k}}{k!}f(k, d).
\end{equation}
Note first that for $k \geq 1$ we have: $f(k, d) \leq 1$, thus:
\begin{equation}
\mathrm{SMREG}(\mathbf{x}, \mathbf{y}) \leq \mathrm{SM}(\mathbf{x}, \mathbf{y}).    
\end{equation}
We also have for $l=d^{\frac{1}{3}}$:
\begin{align}
\begin{split}
\frac{\mathrm{SMREG}(\mathbf{x}, \mathbf{y})}{\mathrm{SM}(\mathbf{x}, \mathbf{y})} = 
e^{-w}\sum_{k=0}^{l} \frac{w^{k}}{k!}f(k, d) +  
e^{-w}\sum_{k=l+1}^{\infty} \frac{w^{k}}{k!}f(k, d) \geq \\
f(l, d) e^{-w}\sum_{k=0}^{l} \frac{w^{k}}{k!} +
e^{-w}\sum_{k=l+1}^{\infty} \frac{w^{k}}{k!}f(k, d) \geq
f(l, d)(1 - e^{-w}\sum_{k=l+1}^{\infty} \frac{w^{k}}{k!}) = \\
f(l, d)(1-\mathbb{P}[\mathrm{Po}(w) > l]),
\end{split}
\end{align}
where $\mathrm{Po}(w)$ stands for the random variable of Poisson distribution with parameter $w$.
Therefore we get for $t = \ln(\frac{l}{w})$:
\begin{align}
\begin{split}
\frac{\mathrm{SMREG}(\mathbf{x}, \mathbf{y})}{\mathrm{SM}(\mathbf{x}, \mathbf{y})} \geq 
(1-\frac{2l-2}{d})^{l}(1-\mathbb{P}[\mathrm{Po}(w) > l]) \geq \\
\exp(l \ln(1-\frac{2l-2}{d}))(1-\mathbb{P}[t\mathrm{Po}(w) \geq tl]) = \\
\exp\left(l \sum_{i=1}^{\infty}(-1)^{i}\frac{(\frac{2l-2}{d})^{i}}{i}\right)(1-\mathbb{P}[\exp(t\mathrm{Po}(w)-tl) \geq 1]) \geq \\
\exp(-\frac{2}{d^{\frac{1}{3}}}+o(\frac{1}{d^{\frac{1}{3}}}))
(1-\exp(-tl)\mathbb{E}[\exp(t\mathrm{Po}(w))])
= \\
\exp(-\frac{2}{d^{\frac{1}{3}}}+o(\frac{1}{d^{\frac{1}{3}}}))
(1-\exp(-w-l(t-1))),
\end{split} 
\end{align}
where the last equality is implied by the formula for the Laplace Transform for the Poisson random variable:
\begin{equation}
\mathbb{E}[\exp(t\mathrm{Po}(w))] = \exp(w(\exp(t)-1)).    
\end{equation}
Notice that: 
$w = \frac{\|\mathbf{z}\|^{2}}{2} = \frac{\ln(\mathrm{SM}(\mathbf{x},\mathbf{x}))+\ln(\mathrm{SM}(\mathbf{y},\mathbf{y})) + 2\ln(\mathrm{SM}(\mathbf{x}, \mathbf{y}))}{2} \leq 2\ln(C)$.
We conclude that:
\begin{equation}
\frac{\mathrm{SMREG}(\mathbf{x}, \mathbf{y})}{\mathrm{SM}(\mathbf{x}, \mathbf{y})} \geq  
(1-\frac{2}{d^{\frac{1}{3}}}+o(\frac{1}{d^{\frac{1}{3}}}))(1-C^{-2}(\frac{d^{\frac{1}{3}}}{2e \cdot \ln(C)})^{-d^{\frac{1}{3}}})
=1 - \frac{2}{d^{\frac{1}{3}}} + o(\frac{1}{d^{\frac{1}{3}}}).
\end{equation}
That completes the proof.
\end{proof}

\subsection{Proofs of Theorem \ref{var-theorem},Theorem \ref{ort-theorem} \& Beautiful Functions}

We will provide here much more general theoretical results which will imply  Theorem \ref{ort-theorem} and Theorem \ref{var-theorem}. We need the following definition:

\begin{definition}
\label{beauty-definition}
We say that function $F:\mathbb{R}^{n} \rightarrow \mathbb{R}$ is beautiful if $F$ can be expressed as:
\begin{equation}
F_{\Omega,g}(\mathbf{z}) = 
\mathbb{E}_{\omega \sim \Omega}[g(\omega^{\top}\mathbf{z})], 
\end{equation}
for a probabilistic isotropic distribution $\Omega$, and where $g:\mathbb{R} \rightarrow \mathbb{R}$ is an entire function
with non-negative power-series coefficients 
(i.e. $g(x) = \sum_{i=0}^{\infty} a_{i}x^{i}$ for every $x \in \mathbb{R}$ and with $a_{i} \geq 0$ for $i=0,1,...$).
In the formula above we assume that the expectation on the RHS exists.
\end{definition}

Interestingly, beautiful functions can be used to define softmax and consequently, Gaussian kernels (both standard and regularized), leading to our PRF mechanism presented in the main body of the paper, as we explain below.

\begin{remark}
\label{imp_remark}
If one takes $\Omega = \mathcal{N}(0,\mathbf{I}_{d})$(note that $\mathcal{N}(0, \mathbf{I}_{d})$ is isotropic) and $g: x \rightarrow \exp(x)$ (such $g$ is clearly entire with nonnegative power-series coefficient) then the following is true for $\mathbf{z} = \mathbf{x}+\mathbf{y}$:
\begin{equation}
\label{beauty_to_sm}
\mathrm{SM}(\mathbf{x}, \mathbf{y}) = \exp(-\frac{\|\mathbf{x}\|^{2}+\|\mathbf{y}\|^{2}}{2})F_{\Omega,g}(\mathbf{z}).   
\end{equation}
Similarly: $\mathrm{SMREG}(\mathbf{x}, \mathbf{y}) = \exp(-\frac{\|\mathbf{x}\|^{2}+\|\mathbf{y}\|^{2}}{2})F_{\Omega_{\mathrm{reg}},g}(\mathbf{z})$, where $\Omega_{\mathrm{reg}}$ stands for the distribution corresponding to Haar measure on the sphere of radius $\sqrt{d}$ (which is clearly isotropic). 
Therefore general concentration results for Monte Carlo estimators of beautiful functions immediately imply corresponding results for the (standard and regularized) softmax (and thus also Gaussian) kernel.
\end{remark}

We will consider two estimators of the beautiful functions from Definition \ref{beauty-definition} that directly lead (through Remark \ref{imp_remark}) to: PRF-based approximation of the softmax-kernel and its enhanced version with orthogonal features. Standard Monte Carlo estimator samples independently $\omega_{1}^{\mathrm{iid}},...,\omega_{m}^{\mathrm{iid}} \overset{\mathrm{iid}}{\sim} \Omega$, where $m$ stands for the number of samples and then computes:
\begin{equation}
\widehat{F}^{\mathrm{iid}}_{m}(\mathbf{z}) \overset{\mathrm{def}}{=}
\frac{1}{m} \sum_{i=1}^{m} g((\omega_{i}^{\mathrm{iid}})^{\top}\mathbf{z}).
\end{equation}
Orthogonal Monte Carlo estimator samples $\omega_{1}^{\mathrm{ort}},...,\omega_{m}^{\mathrm{ort}}$ ($m \leq d$) in such a way that marginally we have: $\omega_{i}^{\mathrm{ort}} \sim \Omega$, but $(\omega_{i}^{\mathrm{ort}})^{\top}\omega_{j}^{\mathrm{ort}}=0$ for $i \neq j$ (such an orthogonal ensemble can be always created if $\Omega$ is isotropic, as we already mentioned in the main body of the paper). We define:
\begin{equation}
\widehat{F}^{\mathrm{ort}}_{m}(\mathbf{z}) \overset{\mathrm{def}}{=}
\frac{1}{m} \sum_{i=1}^{m} g((\omega_{i}^{\mathrm{ort}})^{\top}\mathbf{z}).
\end{equation}

\subsubsection{Orthogonality universally improves concentration}

Denote by $M_{Z}(\theta) = \mathbb{E}[e^{\theta Z}]$ a moment generating function of the random variable $Z$.
Note first that estimators of beautiful functions based on standard Monte Carlo procedure using independent vectors $\omega_{i}^{\mathrm{iid}}$ guarantee strong concentration bounds since 
independent $\omega_{i}$s provide a way to obtain exponentially small upper bounds on failure probabilities through moment generating functions.
We summarize this classic observation which is a standard application of Markov's Inequality below. \\

\begin{lemma}
\label{iid-lemma}
Consider an estimator $\widehat{F}^{\mathrm{iid}}_{m}(\mathbf{z})$ of the beautiful function $F$ evaluated at $\mathbf{z}$. Then the following holds for any $a > F(\mathbf{z})$, $\theta > 0$:
\begin{equation}
\mathbb{P}[\widehat{F}^{\mathrm{iid}}_{m}(\mathbf{z}) > a] \leq \exp(\theta m a) M_X (\theta)^m,
\end{equation}
where $X = g(\mathbf{w}^{\top}\mathbf{z})$, $\mathbf{w} \sim \mathcal{D}$.
\end{lemma}

The above result provides us with exponentially small (in Legendre Transform) upper bounds on tail probabilities for the standard estimator.
Below we provide our two main theoretical results. 

\begin{theorem}[orthogonality provides smaller tails]
\label{general-ort-theorem}
If $F_{\Omega, g}$ is a beautiful function then the following holds for $m \leq d$, $X$ as in Lemma \ref{iid-lemma} and any $a>F(\mathbf{z})$, $\theta > 0$: 
\begin{equation}
\mathbb{P}[\widehat{F}^{\mathrm{ort}}_{m}(\mathbf{z})) > a] \leq \exp(-\theta m a) \left( M_X (\theta)^m - \frac{\theta^4 m (m - 1)}{4 d^2 (d + 2)} a_0^{M - 2} a_1^2 \| \*z \|^4 (\mathbb{E} \| \omega \|^2)^2 \right).
\end{equation}
\end{theorem}

This result shows that features obtained from the ensembles of pairwise orthogonal random vectors provide exponentially small bounds on tail probabilities and that these bounds are strictly better than for estimators using unstructured features. Furthermore, the result is \textbf{universal}, i.e. holds for any dimensionality $d$, not just asymptotically for $d$ large enough.

We also obtain similar result regarding mean squared errors (MSEs) of the considered estimators:

\begin{theorem}
\label{general-var-theorem}
If $F_{\Omega, g}$ is a beautiful function then the following holds for $m \leq d$: 
\begin{equation}
\mathrm{MSE}(\widehat{F}^{\mathrm{ort}}_{m}(\mathbf{z})) \leq 
\mathrm{MSE}(\widehat{F}^{\mathrm{iid}}_{m}(\mathbf{z})) - (1-\frac{1}{m})\frac{2}{d+2}\left(F_{\Omega,g}(\mathbf{z}) - a_0 \right)^2.
\end{equation}
\end{theorem}
As before, an orthogonal estimator leads to better concentration results and as before, this is the case for any $d>0$, not only asymptotically for large enough $d$.

\textbf{Note that from what we have said above, Theorem \ref{var-theorem} and Theorem \ref{ort-theorem} follow immediately from Theorem \ref{general-var-theorem} and Theorem \ref{general-ort-theorem} respectively.} 

Thus in the remainder of this section we will prove Theorem \ref{general-var-theorem} and Theorem \ref{general-ort-theorem}.

\subsubsection{Proof of Theorem \ref{general-ort-theorem}}

\begin{proof}
Note that by the analogous application of Markov's Inequality as in Lemma \ref{iid-lemma}, we get:
\begin{align}
\begin{split}
\mathbb{P}[\widehat{F}^{\mathrm{ort}}_{m}(\mathbf{z})) > a] \leq   
\frac{\mathbb{E}[e^{\theta (X_{1}^{\mathrm{ort}}+...+X_{m}^{\mathrm{ort}})}]}{e^{\theta ma}},
\end{split}    
\end{align}
where we have:
$X_{i}^{\mathrm{ort}} = g((\omega_{i}^{\mathrm{ort}})^{\top}\mathbf{z})$.
We see that it suffices to show that for any $\theta > 0$ the following holds: 
$\mathbb{E}[e^{\theta (X_{1}^{\mathrm{ort}}+...+X_{m}^{\mathrm{ort}})}] < \mathbb{E}[e^{\theta (X_{1}^{\mathrm{iid}}+...+X_{m}^{\mathrm{iid}})}]$.
We have: 
\begin{align}
\begin{split}
\mathbb{E}[e^{\theta (X_{1}^{\mathrm{ort}}+...+X_{m}^{\mathrm{ort}})}] = \mathbb{E}[\sum_{j=0}^{\infty} \frac{(\theta \sum_{i=1}^{m} X_{i}^{\mathrm{ort}})^{j}}{j!}] 
= \mathbb{E}[\sum_{j=0}^{\infty}\frac{\theta^{j}}{j!}(\sum_{i=1}^{m}X^{\mathrm{ort}}_{i})^{j}]=\\
\sum_{j=0}^{\infty}\frac{\theta^{j}}{j!} \mathbb{E}[(\sum_{i=1}^{m} X^{\mathrm{ort}}_{i})^{j}]=
\sum_{j=0}^{\infty}\frac{\theta^{j}}{j!}
\mathbb{E}[\sum_{(j_{1},...,j_{m}) \in \mathcal{S}_{j}} \binom{j}{j_{1},\dots,j_{m}} (X_{1}^{\mathrm{ort}})^{j_{1}} \cdot ... \cdot (X_{m}^{\mathrm{ort}})^{j_{m}}],
\end{split}
\end{align}
where $\mathcal{S}_{j} = \{(j_{1},...,j_{m}) \in \mathbb{N} \times ...\times \mathbb{N}:j_{1},...,j_{m} \geq 0, j_{1}+...+j_{m}=j\}$.

Thus we have:
\begin{equation}
\mathbb{E}[e^{\theta (X_{1}^{\mathrm{ort}}+...+X_{m}^{\mathrm{ort}})}] = \sum_{j=0}^{\infty} \frac{\theta^{j}}{j!} \sum_{(j_{1},...,j_{m}) \in \mathcal{S}_{j}} \binom{j}{j_{1},\dots,j_{m}} \mathbb{E}[(X_{1}^{\mathrm{ort}})^{j_{1}} \cdot ... \cdot (X_{m}^{\mathrm{ort}})^{j_{m}}].  
\end{equation}

Similarly, we get:
\begin{equation}
\mathbb{E}[e^{\theta (X_{1}^{\mathrm{iid}}+...+X_{m}^{\mathrm{iid}})}] = \sum_{j=0}^{\infty} \frac{\theta^{j}}{j!} \sum_{(j_{1},...,j_{m}) \in \mathcal{S}_{j}} \binom{j}{j_{1},\dots,j_{m}} \mathbb{E}[(X_{1}^{\mathrm{iid}})^{j_{1}} \cdot ... \cdot (X_{m}^{\mathrm{iid}})^{j_{m}}].    
\end{equation}

Therefore we get:
\begin{align}
\begin{split}
\Delta = \mathbb{E}[e^{\theta (X_{1}^{\mathrm{iid}}+...+X_{m}^{\mathrm{iid}})}] - \mathbb{E}[e^{\theta (X_{1}^{\mathrm{ort}}+...+X_{m}^{\mathrm{ort}})}] \\
=
\sum_{j=0}^{\infty} \frac{\theta^{j}}{j!} \sum_{(j_{1},...,j_{m}) \in \mathcal{S}_{j}} \binom{j}{j_{1},\dots,j_{m}} \left(\mathbb{E}[(X_{1}^{\mathrm{iid}})^{j_{1}} \cdot ... \cdot (X_{m}^{\mathrm{iid}})^{j_{m}}] - \mathbb{E}[(X_{1}^{\mathrm{ort}})^{j_{1}} \cdot ... \cdot (X_{m}^{\mathrm{ort}})^{j_{m}}]\right)
\end{split}
\end{align}


Note first that using the fact that $f$ is entire, we can rewrite each $X_{i}^{\mathrm{ort}}$ as:
\begin{equation}
\label{x_ort_formula}
X_{i}^{\mathrm{ort}} = \sum_{s=0}^{\infty} a_{s}((\omega_{i}^{\mathrm{ort}})^{\top}\mathbf{z})^{s},
\end{equation}

where $f(x) = \sum_{s=0}^{\infty} a_{s}x^{s}$
and $a_{0},a_{1},... \geq 0$.
Similarly,
\begin{equation}
\label{x_iid_formula}
X_{i}^{\mathrm{iid}} = \sum_{s=0}^{\infty} a_{s}((\omega_{i}^{\mathrm{iid}})^{\top}\mathbf{z})^{s}.
\end{equation}

By plugging in the above formulae for $X_{i}^{\mathrm{ort}}$ and $X_{i}^{\mathrm{iid}}$ int the formula for $\Delta$ and expanding power-expressions, we obtain:
\begin{align}
\begin{split}
\Delta = \sum_{j=0}^{\infty} \frac{\theta^{j}}{j!} \sum_{(j_{1},...,j_{m}) \in \mathcal{S}_{j}} \binom{j}{j_{1},\dots,j_{m}} \sum_{(d_{1},...,d_{m}) \in \mathcal{D}(j_{1},...,j_{m})} \widehat{c}_{j_1, \dots, j_m}(d_{1},\dots,d_{m}) \widehat{\Delta}(d_{1},...,d_{m}),
\end{split}    
\end{align}
for some ordered subsets of indices (with potentially repeating entries) $\mathcal{D}(j_{1},...,j_{m})$ and some nonnegative $\widehat{c}_{j_1, \dots, j_m}(d_{1},\dots,d_{m})$ (exact formula for those can be given but we do not need it to complete the proof and since it is technical, it would unnecessarily complicate the proof so we skip it)
and $\widehat{\Delta}(d_{1},...,d_{m})$ defined as:
\begin{align}
\begin{split}
\label{imp-ineq}
\widehat{\Delta}(d_{1},...,d_{m}) = \mathbb{E}[((\omega_{1}^{\mathrm{iid}})^{\top}\mathbf{z})^{d_{1}} \cdot ... \cdot ((\omega_{m}^{\mathrm{iid}})^{\top}\mathbf{z})^{d_{m}}] - 
\mathbb{E}[((\omega_{1}^{\mathrm{ort}})^{\top}\mathbf{z})^{d_{1}} \cdot ... \cdot ((\omega_{m}^{\mathrm{ort}})^{\top}\mathbf{z})^{d_{m}}].
\end{split}
\end{align}

Our next goal is to re-write the formula for $\widehat{\Delta}(d_{1},...,d_{m})$. Denote:
\begin{equation}
Y = ((\omega_{1}^{\mathrm{ort}})^{\top}\mathbf{z})^{d_{1}} \cdot ... \cdot ((\omega_{m}^{\mathrm{ort}})^{\top}\mathbf{z})^{d_{m}}.    
\end{equation}

Observe that $Y$ has the same distribution as $Y^{\prime}$ defined as:

\begin{equation}
Y^{\prime} = (\mathbf{e}_{1}^{\top}\frac{\mathbf{g}}{\|\mathbf{g}\|_{2}}\|\mathbf{z}\|_{2})^{d_{1}} \cdot ... \cdot (\mathbf{e}_{m}^{\top}\frac{\mathbf{g}}{\|\mathbf{g}\|_{2}}\|\mathbf{z}\|_{2})^{d_{m}} \cdot
(\|\omega_{1}^{\mathrm{ort}}\|_{2})^{d_{1}} \cdot ... \cdot
(\|\omega_{m}^{\mathrm{ort}}\|_{2})^{d_{m}},
\end{equation}

where $\mathbf{g}$ is a Gaussian vector taken from the $\mathcal{N}(0,\mathbf{I}_{d})$ distribution, independently from: $\|\omega_{1}^{\mathrm{ort}}\|_{2},...,\|\omega_{m}^{\mathrm{ort}}\|_{2}$. 

This comes from the fact that for a fixed $\mathbf{z}$ one can think about the set:
$\frac{\omega_{1}^{\mathrm{ort}}}{\|\omega_{1}^{\mathrm{ort}}\|_{2}},...,\frac{\omega_{m}^{\mathrm{ort}}}{\|\omega_{m}^{\mathrm{ort}}\|_{2}}$ as a random rotation of the system of $m$ canonical basis vectors: $\mathbf{e}_{1},...,\mathbf{e}_{m}$.
Thus instead of applying a random rotation to: $\mathbf{e}_{1},...,\mathbf{e}_{m}$, one can equivalently randomly rotate vector $\mathbf{z}$. Randomly rotated vector $\mathbf{z}$ has the same distribution as: $\frac{\mathbf{g}}{\|\mathbf{g}\|_{2}}\|\mathbf{z}\|_{2}$. 


Now note that lengths of vectors $\omega_{1}^{\mathrm{ort}},...,\omega_{m}^{\mathrm{ort}}$ are chosen independently.

Therefore we obtain:
\begin{align}
\begin{split}
\mathbb{E}[((\omega_{1}^{\mathrm{ort}})^{\top}\mathbf{z})^{d_{1}} \cdot ... \cdot ((\omega_{m}^{\mathrm{ort}})^{\top}\mathbf{z})^{d_{m}}] = \\ \mathbb{E}[(\|\omega_{1}^{\mathrm{ort}}\|_{2})^{d_{1}}] \cdot ... \cdot \mathbb{E}[(\|\omega_{m}^{\mathrm{ort}}\|_{2})^{d_{m}}] \cdot \mathbb{E}[(\mathbf{e}_{1}^{\top}\mathbf{v})^{d_{1}} \cdot ... \cdot (\mathbf{e}_{m}^{\top}\mathbf{v})^{d_{m}}]
\|\mathbf{z}\|_{2}^{d_{1}+...+d_{m}},
\end{split}
\end{align}
where $\mathbf{v} \sim \frac{\mathbf{g}}{\|\mathbf{g}\|_{2}}$.

Denote $\mathbf{g}=(g_{1},...,g_{d})^{\top}$.
Thus we obtain:
\begin{align}
\begin{split}
\label{lhs}
\mathbb{E}[((\omega_{1}^{\mathrm{ort}})^{\top}\mathbf{z})^{d_{1}} \cdot ... \cdot ((\omega_{m}^{\mathrm{ort}})^{\top}\mathbf{z})^{d_{m}}] = \\ \mathbb{E}[(\|\omega_{1}^{\mathrm{ort}}\|_{2})^{d_{1}}] \cdot ... \cdot \mathbb{E}[(\|\omega_{m}^{\mathrm{ort}}\|_{2})^{d_{m}}] \cdot 
\|\mathbf{z}\|_{2}^{d_{1}+...+d_{m}} \mathbb{E}[\frac{g_{1}^{d_{1} \cdot ... \cdot}g_{m}^{d_{m}}}{\sqrt{g_{1}^{2}+...+g_{d}^{2}}^{d_{1}+...+d_{m}}}]
\end{split}
\end{align}

Now let us focus on the second expression from the formula on $\widehat{\Delta}(d_{1},...,d_{m})$. We have:
\begin{align}
\begin{split}
\label{rhs}
\mathbb{E}[((\omega_{1}^{\mathrm{iid}})^{\top}\mathbf{z})^{d_{1}} \cdot ... \cdot ((\omega_{m}^{\mathrm{iid}})^{\top}\mathbf{z})^{d_{m}}] = \prod_{i=1}^{m} \mathbb{E}[((\omega_{i}^{\mathrm{iid}})^{\top}\mathbf{z})^{d_{i}}]   
= \\ \mathbb{E}[(\|\omega_{1}^{\mathrm{iid}}\|_{2})^{d_{1}}] \cdot ... \cdot \mathbb{E}[(\|\omega_{m}^{\mathrm{iid}}\|_{2})^{d_{m}}] \cdot \|\mathbf{z}\|_{2}^{d_{1}+...+d_{m}} \cdot \prod_{i=1}^{m} \mathbb{E}[\frac{g_{i}^{d_{i}}}{\sqrt{g_{1}^{2}+...+g_{d}^{2}}^{d_{i}}}],
\end{split}
\end{align}

where the first equality comes from the fact that
different $\omega_{i}^{\mathrm{iid}}$s are independent and the second one is implied by the analogous analysis to the one conducted above.

We will need the following lemma:

\begin{lemma}
\label{useful-lemma}
For every $s \in \mathbb{N}_{+}$ such that $s \leq n$ and every $k_{1},...,k_{s} \in \mathbb{N}_{+}$ the following holds:
\begin{equation}
\mathbb{E}[\frac{g_{1}^{k_{1}} \cdot ... \cdot g_{s}^{k_{s}}}{\sqrt{g_{1}^{2}+...+g_{d}^{2}}^{k_{1}+...+k_{s}}}] = \frac{\prod_{i=1}^{s}\mathbb{E}[g_{i}^{k_{i}}]}{\mathbb{E}[\sqrt{g_{1}^{2}+...+g_{d}^{2}}^{k_{1}+...+k_{s}}]}.    
\end{equation}
\end{lemma}

\begin{proof}
Take $\mathbf{r} = \frac{\mathbf{g}}{\|\mathbf{g}\|_{2}}\|\tilde{\mathbf{g}}\|_{2}$, where $\tilde{\mathbf{g}}$ is an independent copy of $\mathbf{g}$. Note that $\mathbf{r} \sim \mathbf{g}$.
We have:
\begin{align}
\begin{split}
\mathbb{E}[r_{1}^{k_{1}}] \cdot ... \cdot     
\mathbb{E}[r_{s}^{k_{s}}] = 
\mathbb{E}[r_{1}^{k_{1}} \cdot ... \cdot r_{s}^{k_{s}}]
= \mathbb{E}[\frac{g_{1}^{k_{1}} \cdot ... \cdot g_{s}^{k_{s}}}{\sqrt{g_{1}^{2}+...+g_{d}^{2}}^{k_{1}+...+k_{s}}}]
\cdot \mathbb{E}[\|\tilde{\mathbf{g}}\|_{2}^{k_{1}+...+k_{s}}],
\end{split}    
\end{align}
where the first equality comes from the independence of different elements of $\mathbf{z}=(z_{1},...,z_{n})^{\top}$
and the second equality is implied by the fact that $\tilde{\mathbf{g}}$ is independent from $\mathbf{g}$.

Therefore we have:
\begin{equation}
 \mathbb{E}[\frac{g_{1}^{k_{1}} \cdot ... \cdot g_{s}^{k_{s}}}{\sqrt{g_{1}^{2}+...+g_{d}^{2}}^{k_{1}+...+k_{s}}}] = \frac{\mathbb{E}[r_{1}^{k_{1}}] \cdot ... \cdot     \mathbb{E}[r_{s}^{k_{s}}]}{\mathbb{E}[\|\tilde{\mathbf{g}}\|_{2}^{k_{1}+...+k_{s}}]}.   
\end{equation}
That completes the proof since $\mathbf{z} \sim \mathbf{g}$ and $\tilde{\mathbf{g}} \sim \mathbf{g}$.
\end{proof}

Note that by Lemma \ref{useful-lemma}, we can rewrite the right expression from the formula on 
$\widehat{\Delta}(d_1,..., d_m)$
as: 
\begin{equation}
\mathbb{E}[(\|\omega_{1}^{\mathrm{ort}}\|_{2})^{d_{1}}] \cdot ... \cdot \mathbb{E}[(\|\omega_{m}^{\mathrm{ort}}\|_{2})^{d_{m}}] \cdot \\
\|\mathbf{z}\|_{2}^{d_{1}+...+d_{m}}\frac{\prod_{i=1}^{m}\mathbb{E}[g_{i}^{d_{i}}]}{\mathbb{E}[\sqrt{g_{1}^{2}+...+g_{d}^{2}}^{d_{1}+...+d_{m}}]}.
\end{equation}
The left expression from the formula on 
$\widehat{\Delta}(d_1,..., d_m)$ can be rewritten as:
\begin{align}
\begin{split}
L(d_{1},...,d_{m}) = \mathbb{E}[(\|\omega_{1}^{\mathrm{iid}}\|_{2})^{d_{1}}] \cdot ... \cdot \mathbb{E}[(\|\omega_{m}^{\mathrm{iid}}\|_{2})^{d_{m}}] \cdot 
\|\mathbf{z}\|_{2}^{d_{1}+...+d_{m}} \\
\frac{\prod_{i=1}^{m}\mathbb{E}[g_{i}^{d_{i}}]}
{\mathbb{E}[\sqrt{g_{1}^{2}+...+g_{d}^{2}}^{d_{1}}] \cdot ...\cdot \mathbb{E}[\sqrt{g_{1}^{2}+...+g_{d}^{2}}^{d_{m}}]}.
\end{split}
\end{align}

Since marginal distributions of $\omega_{i}^{\mathrm{ort}}$ and $\omega_{i}^{\mathrm{iid}}$ are the same, we can rewrite $\widehat{\Delta}(d_{1},...,d_{n})$ as:
\begin{equation}
\widehat{\Delta}(d_{1},...,d_{m})=
L(d_{1},...,d_{m})(1 - \tau(d_{1},...,d_{m})),
\end{equation}
where $\tau(d_{1},...,d_{m})$ is defined as:
\begin{equation}
\tau(d_{1},...,d_{m}) = \frac{\mathbb{E}[\sqrt{g_{1}^{2}+...+g_{d}^{2}}^{d_{1}}] \cdot ...\cdot \mathbb{E}[\sqrt{g_{1}^{2}+...+g_{d}^{2}}^{d_{m}}]}
{\mathbb{E}[\sqrt{g_{1}^{2}+...+g_{d}^{2}}^{d_{1}+...+d_{m}}]}    \label{eq:taudef} 
\end{equation}
We need now few observations regarding $\widehat{\Delta}(d_{1},...,d_{m})$.
Note firsr that since odd moments of the Gaussian scalar distribution $\mathcal{N}(0, 1)$ are zero, $\widehat{\Delta}(d_{1},...,d_{m})$ is zero if at least of of $d_{i}$ is odd. Furthermore, $\widehat{\Delta(d_{1},...,d_{m})}$ is trivially zero if all but at most one $d_{i}$ are zero.

With our new notation, $\Delta$ can be rewritten as:
\begin{gather*}
\Delta = \sum_{j=0}^{\infty} \frac{\theta^{j}}{j!} \sum_{(j_{1},...,j_{m}) \in \mathcal{S}_{j}} \binom{j}{j_{1},\dots,j_{m}} \sum_{(d_{1},...,d_{m}) \in \mathcal{D}(j_{1},...,j_{m})} \widehat{c}_{j_1, \dots, j_m}(d_{1},\dots,d_{m}) \\
\times L(d_{1},...,d_{m})(1-\tau(d_{1},...,d_{m})),
\end{gather*}

Note also that we have:
\begin{gather*}
e^{\theta(X_{1}^{\mathrm{iid}}+...+X_{m}^{\mathrm{iid}})} = \sum_{j=0}^{\infty} \frac{\theta^{j}}{j!} \sum_{(j_{1},...,j_{m}) \in \mathcal{S}_{j}} \binom{j}{j_{1},\dots,j_{m}} \sum_{(d_{1},...,d_{m}) \in \mathcal{D}(j_{1},...,j_{m})} \widehat{c}_{j_1, \dots, j_m}(d_{1},\dots,d_{m}) \\
\times  L(d_{1},...,d_{m}).
\end{gather*}

Therefore (see: our observations on $\widehat{\Delta}(d_{1},...,d_{m})$) to complete the proof it suffices to show that: $\tau(d_{1},...,d_{m}) \leq \frac{d}{d+2}$ if at least two: $d_{i}$, $d_{j}$ for $i \neq j$ are nonzero and all $d_{i}$ are even.
\begin{lemma}
\label{tau-lemma}
The following holds if for some $i \neq j$ we have: $d_{i}, d_{j} > 0$ and all $d_{i}$ are even:
\begin{equation}
\tau(d_{1},...,d_{m}) \leq \frac{d}{d+2}.    
\end{equation}
\end{lemma}
\begin{proof}
Note that $\tau(d_{1},...,d_{m})$ can be rewritten as:
\begin{equation}
\label{multi-d}
\tau(d_{1},...,d_{m}) = \frac{\prod_{i=1}^{m} \mu_{d}(d_{i})}{\mu_{d}(\sum_{i=1}^{m} d_i)},    
\end{equation}
where $\mu_{d}(j)$ stands for the $j^{th}$ moment of the $\chi$-distribution with $d$ degrees of freedom.
Note that $\mu_{d}(j) = 2^{\frac{j}{2}}
\frac{\Gamma(\frac{d+j}{2})}{\Gamma(\frac{d}{2})}$,
where $\Gamma$ is the so-called \textit{Gamma-function}.

Using the fact that: $\Gamma(n) = (n-1)!$ and $\Gamma(n+\frac{1}{2})=\frac{(2n-1)!!}{2^{n}}\sqrt{\pi}$ for $n \in \mathbb{N}_{+}$, it is easy to see 
that for a fixed $d$, the RHS of the Equality \ref{multi-d} is maximized when $d_{i}=d_{j}=2$ and $d_{k}=0$ for some $i \neq j$ and $k \notin \{i,j\}$. Furthermore, straightforward calculations show that in that case the value of the RHS from Equality \ref{multi-d} is $\frac{d}{d+2}$. That completes the proof of the Lemma.
\end{proof}

By $\mathcal{D}' (j_1, \dots, j_m)$ denote a subset of $\mathcal{D} (j_1, \dots, j_m)$ formed by only keeping $d_1, \dots, d_m$ such that for some $i \neq j$, $d_i, d_j > 0$ and all $d_i$ are even. As we have shown above, $\widehat{\Delta}(d_{1},\dots,d_{m}) = 0$ when $(d_1, \dots, d_m) \notin \mathcal{D}' (j_1, \dots, j_m)$. Otherwise,
\begin{equation*}
    \widehat{\Delta}(d_{1},\dots,d_{m}) \geq \frac{2}{d + 2} \Lambda (d_{1},\dots,d_{m}) \geq 0 .
\end{equation*}
Hence, since all terms in the sum
\begin{gather}
\Delta = \sum_{j=0}^{\infty} \frac{\theta^{j}}{j!} \sum_{(j_{1},\dots,j_{m}) \in \mathcal{S}_{j}} \binom{j}{j_{1},\dots,j_{m}} \sum_{(d_{1},\dots,d_{m}) \in \mathcal{D}(j_{1},\dots,j_{m})} \widehat{c}_{j_1, \dots, j_m}(d_{1},\dots,d_{m}) \\
\times \widehat{\Delta}(d_{1},\dots,d_{m}). \label{eq:doublesum}
\end{gather}
are nonnegative, we'll get a lower bound on $\Delta$ by only taking a subset of these terms. For this subset, we take $j = 4$, a subset of $\mathcal{S}_4$ with only two nonzero $j_{k_1} = j_{k_2} = 2$ for some $k_1 \neq k_2$ (there are $\binom{m}{2}$ combinations of such $j_1, \dots, j_m$). Then, we take only those $d_1, \dots, d_m$ from $\mathcal{D} (j_1, \dots, j_m)$ which correspond to $s = 1$ in (\ref{x_iid_formula}) for $k_1, k_2$ and $s = 0$ for all other $k$'s. Hence, $d_{k_1} = d_{k_2} = 2$ and all other $d_k$'s are zero and the corresponding weight from the second sum in (\ref{eq:doublesum}) would be $a_1^2 a_0^{m - 2}$. For $d_1, \dots, d_m$ in such set, we'll have $\tau (d_1, \dots, d_m) \leq \frac{d}{d + 2}$ by Lemma \ref{tau-lemma} and, hence, $\widehat{\Delta}(d_{1},\dots,d_{m}) \geq \frac{2}{d + 2} \Lambda(d_{1},\dots,d_{m})$. As the result, we get the following lower bound on $\Delta$:
\begin{align*}
    \Delta &\geq \frac{2 \theta^4}{ 4! (d + 2)}  \binom{m}{2} \binom{4}{2, 2, 0, \dots, 0} a_1^2 a_0^{m - 2} \Lambda (2, 2, 0, \dots, 0) \\
    &= \frac{\theta^4 m (m - 1)}{4 (d + 2)} a_1^2 a_0^{m - 2} \Lambda(2, 2, 0, \dots, 0) \\
    &= \frac{\theta^4 m (m - 1)}{4 (d + 2)} a_1^2 a_0^{m - 2} \| \*z \|^4 \left( \mathbb{E} \| \bs{\omega} \|^2 \right)^2 \frac{(\mathbb{E} (\*g_1^2))^2}{(\mathbb{E} \| \*g \|^2)^2} .
\end{align*}
Since $\*g \sim \mathcal{N} (0, 1)^d$, $\mathbb{E} \*g_1^2 = 1$ and $\mathbb{E} \| \*g \|^2 = d \mathbb{E} \*g_1^2 = d$. This results in
\begin{equation}
    \Delta \geq \frac{\theta^4 m (m - 1)}{4 d^2 (d + 2)} a_1^2 a_0^{m - 2} \| \*z \|^4 \left( \mathbb{E} \| \omega \|^2 \right)^2 \label{eq:deltalb}
\end{equation}
which concludes the proof.

\end{proof}
\subsubsection{Proof of Theorem \ref{general-var-theorem}}
\begin{proof}
We will use the notation from the proof of Theorem \ref{general-ort-theorem}.
Since both estimators: $\widehat{F}^{\mathrm{ort}}_{m}(\mathbf{z})$ and
$\widehat{F}^{\mathrm{iid}}_{m}(\mathbf{z})$ are unbiased, we have:
$\mathrm{MSE}(\widehat{F}^{\mathrm{ort}}_{m}(\mathbf{z})) = \mathrm{Var}(\widehat{F}^{\mathrm{ort}}_{m}(\mathbf{z}))$ and
$\mathrm{MSE}(\widehat{F}^{\mathrm{iid}}_{m}(\mathbf{z})) = \mathrm{Var}(\widehat{F}^{\mathrm{iid}}_{m}(\mathbf{z}))$.
We have:
\begin{align}
\begin{split}
\mathrm{Var}(\widehat{F}^{\mathrm{iid}}_{m}(\mathbf{z})) = 
\mathbb{E}[(\widehat{F}^{\mathrm{iid}}_{m}(\mathbf{z})-\mathbb{E}[\widehat{F}^{\mathrm{iid}}_{m}(\mathbf{z})])^{2}] =
\mathbb{E}[(\widehat{F}^{\mathrm{iid}}_{m}(\mathbf{z}))^{2}]-
F^{2}(\mathbf{z}).
\end{split}
\end{align}

Similarly,
\begin{align}
\begin{split}
\mathrm{Var}(\widehat{F}^{\mathrm{ort}}_{m}(\mathbf{z})) = 
\mathbb{E}[(\widehat{F}^{\mathrm{ort}}_{m}(\mathbf{z}))^{2}]-
F^{2}(\mathbf{z}).
\end{split}
\end{align}

We have: 
\begin{align}
\begin{split}
\mathbb{E}[(\widehat{F}^{\mathrm{iid}}_{m}(\mathbf{z}))^{2}]
= \frac{1}{m^{2}}\sum_{i=1}^{m}\mathbb{E}[(X_{i}^{\mathrm{iid}})^{2}]
+\frac{1}{m^{2}}\sum_{i \neq j} \mathbb{E}[X^{\mathrm{iid}}_{i}X^{\mathrm{iid}}_{j}].
\end{split}
\end{align}

Similarly, we get:
\begin{align}
\begin{split}
\mathbb{E}[(\widehat{F}^{\mathrm{ort}}_{m}(\mathbf{z}))^{2}]
= \frac{1}{m^{2}}\sum_{i=1}^{m}\mathbb{E}[(X_{i}^{\mathrm{ort}})^{2}]
+\frac{1}{m^{2}}\sum_{i \neq j} \mathbb{E}[X^{\mathrm{ort}}_{i}X^{\mathrm{ort}}_{j}].
\end{split}
\end{align}

Therefore, since marginal distributions of $X_{i}^{\mathrm{iid}}$ and $X_{i}^{\mathrm{ort}}$ are the same,  we have:
\begin{align}
\begin{split}
\label{mse-diff}
\mathrm{MSE}(\widehat{F}^{\mathrm{iid}}_{m}(\mathbf{z})) - \mathrm{MSE}(\widehat{F}^{\mathrm{ort}}_{m}(\mathbf{z})) = 
{m \choose 2} \cdot 2 \cdot \frac{1}{m^{2}}
(\mathbb{E}[X^{\mathrm{iid}}_{1}X^{\mathrm{iid}}_{2}]-
\mathbb{E}[X^{\mathrm{ort}}_{1}X^{\mathrm{ort}}_{2}])\\
=(1-\frac{1}{m})(\mathbb{E}[X^{\mathrm{iid}}_{1}X^{\mathrm{iid}}_{2}]-
\mathbb{E}[X^{\mathrm{ort}}_{1}X^{\mathrm{ort}}_{2}])
\end{split}
\end{align}
Plugging in the formula for $X^{\mathrm{ort}}_{i}$ and $X^{\mathrm{iid}}_{i}$ from Equation \ref{x_ort_formula} and Equation \ref{x_iid_formula}, and using our analysis from the proof of Theorem \ref{ort-theorem} we obtain:
\begin{align}
\begin{split}
 \mathrm{MSE}(\widehat{F}^{\mathrm{iid}}_{m}(\mathbf{z})) - \mathrm{MSE}(\widehat{F}^{\mathrm{ort}}_{m}(\mathbf{z})) =   
 (1-\frac{1}{m})\sum_{t,u=0}^{\infty}a_{t}a_{u}\|\mathbf{z}\|_{2}^{t+u}
 \mathbb{E}[\|\omega\|_{2}^{t}] \mathbb{E}[\|\omega\|_{2}^{u}] \cdot \\
\frac{\mathbb{E}[r^{t}]\mathbb{E}[r^{u}]}
{\mathbb{E}[\sqrt{g_{1}^{2}+...+g_{d}^{2}}^{t}]
\mathbb{E}[\sqrt{g_{1}^{2}+...+g_{d}^{2}}^{u}]}
(1-\tau(t, u)). \label{eq:vardec}
\end{split} 
\end{align}
for $\omega \sim \Omega$ and $r \sim \mathcal{N}(0, 1)$.

Based on the definition of $\tau$ (\ref{eq:taudef}), if $t = 0$ or $u = 0$, $\tau (t, u) = 1$ and the whole corresponding term in the sum (\ref{eq:vardec}) is zero. Also, if $t$ is odd, $\mathbb{E}(r^t) = 0$ and, again, the corresponding term in the sum (\ref{eq:vardec}) is zero. Same holds for $u$ from (\ref{eq:vardec}). Based on the analysis from Theorem \ref{general-ort-theorem}'s proof and $F_{\Omega,g}(\mathbf{z})$'s definition we have:
\begin{equation*}
    F_{\Omega,g}(\mathbf{z})  = \sum_{t=0}^{\infty}a_{t}\|\mathbf{z}\|_{2}^{t}
 \mathbb{E}[\|\omega\|_{2}^{t}] \cdot 
\frac{\mathbb{E}[r^{t}]}
{\mathbb{E}[\sqrt{g_{1}^{2}+...+g_{d}^{2}}^{t}]} = \sum_{t=0}^{\infty}a_{2 t}\|\mathbf{z}\|_{2}^{2 t}
 \mathbb{E}[\|\omega\|_{2}^{2 t}] \cdot 
\frac{\mathbb{E}[r^{2 t}]}
{\mathbb{E}[\sqrt{g_{1}^{2}+...+g_{d}^{2}}^{2 t}]}
\end{equation*}
where in the second transition we use the fact that $\mathbb{E}[r^{t}] = 0$ for odd $t$.

Hence, we can rewrite (\ref{eq:vardec}) by excluding terms which are definitely zero and using Lemma \ref{tau-lemma}:
\begin{align}
\begin{split}
 \mathrm{MSE}(\widehat{F}^{\mathrm{iid}}_{m}(\mathbf{z})) - \mathrm{MSE}(\widehat{F}^{\mathrm{ort}}_{m}(\mathbf{z})) \geq (1-\frac{1}{m})\frac{2}{d+2}\sum_{t,u=1}^{\infty}a_{2 t}a_{2 u}\|\mathbf{z}\|_{2}^{2 t+2u}
 \mathbb{E}[\|\omega\|_{2}^{2 t}] \mathbb{E}[\|\omega\|_{2}^{2 u}] \cdot \\
\frac{\mathbb{E}[r^{2 t}]\mathbb{E}[r^{2 u}]}
{\mathbb{E}[\sqrt{g_{1}^{2}+...+g_{d}^{2}}^{2 t}]
\mathbb{E}[\sqrt{g_{1}^{2}+...+g_{d}^{2}}^{2 u}]}\\
= (1-\frac{1}{m})\frac{2}{d+2}   
\left(\sum_{t=1}^{\infty}a_{2 t}\|\mathbf{z}\|_{2}^{2 t}
 \mathbb{E}[\|\omega\|_{2}^{2 t}] \cdot 
\frac{\mathbb{E}[r^{2 t}]}
{\mathbb{E}[\sqrt{g_{1}^{2}+...+g_{d}^{2}}^{2 t}]}\right)^{2} \\
= (1-\frac{1}{m})\frac{2}{d+2}\left(F_{\Omega,g}(\mathbf{z}) - a_0\right)^2.
\end{split}    
\end{align}
That completes the proof.
\end{proof}

\subsection{Proof of Theorem \ref{thm:uniform}}

We showed in the main body of the paper that in contrast to other methods approximating the attention matrix $\mathbf{A}$, our algorithm provides strong concentration guarantees. This is the case also for trigonometric random features, yet, as discussed in the main body of the paper, due to attention renormalization and higher variance of the estimation of small entries of the attention matrix, trigonometric mechanism is sub-optimal.
We show here that $m_{\mathrm{opt}}$, the optimal number of random projections for the trigonometric orthogonal mechanism for accurate estimation of the attention matrix does not depend on $L$ but only on $d$. In fact, we prove that if we take $m_{\mathrm{opt}} = \Theta(d\log(d))$, then with $O(L d^2\log(d))$-time, we can approximate $\mathbf{A}$ up to any precision, regardless of the number of tokens $L$. In order to provide those guarantees, we leverage recent research on the theory of negative dependence for ORFs \citep{Lin2020DemystifyingOM}.

We prove the more general version of Theorem \ref{thm:uniform} from the main body of the paper: 

\begin{theorem}[Uniform convergence for the trigonometric mechanism]
\label{uniform_convergence}
Define entries of the attention matrix $\mathbf{A}$ as follows: $\mathbf{A}_{i,j}=g(\mathbf{q}_{i}^{\top})\mathrm{K}(\frac{1}{d^{\frac{1}{4}}}\mathbf{q}_{i}^{\top},\frac{1}{d^{\frac{1}{4}}}\mathbf{k}_{j}^{\top})h(\mathbf{k}_{j}^{\top})$ for some
$g,h:\mathbb{R}^{d} \rightarrow \mathbb{R}$ and where $\mathrm{K}$ is a radial basis function (RBF) kernel \citep{geom} with corresponding spectral distribution $\Omega$ (e.g. Gaussian kernel for which $\Omega=\mathcal{N}(0,\mathbf{I}_{d})$). Assume that the rows of matrices $\mathbf{Q}$ and $\mathbf{K}$ are taken from a ball $B(R)$ of radius $R$, centered at $0$ (i.e. norms of queries and keys are upper-bounded by $R$). 
Define $l=Rd^{-\frac{1}{4}}$ and take $g^{*} = \max_{\mathbf{x} \in B(l)}|g(\mathbf{x})|$ and
$h^{*} = \max_{\mathbf{x} \in B(l)}|h(\mathbf{x})|$.
Then for any $\epsilon>0$, $\delta=\frac{\epsilon}{g^{*}h^{*}}$
and the number of random projections $m = \Omega(\frac{d}{\delta^{2}}\log(\frac{4\sigma R}{\delta d^{\frac{1}{4}}}))$ for $\sigma=\mathbb{E}_{\omega \sim \Omega}[\omega^{\top}\omega]$ the following holds:
$
\|\widehat{\mathbf{A}}-\mathbf{A}\|_{\infty} \leq \epsilon     
$
with any constant probability,
where $\widehat{\mathbf{A}}$ approximates generalized attention matrix via orthogonal trigonometric random features.
\end{theorem}

The result holds in particular for regular softmax-attention for which $\mathrm{K}$ is a Gaussian kernel and $g(\mathbf{x}) = h(\mathbf{x}) =  \exp(\frac{\|\mathbf{x}\|^{2}}{2})$. In that case $m_{\mathrm{opt}}=\Omega(\frac{d}{\delta^{2}}\log(\frac{4d^{\frac{3}{4}} R}{\delta}))$ since $\sigma = d$.

\begin{proof}
Let $\mathbf{D}_{\mathbf{Q}}$ be a diagonal matrix with entries of the form: $g(\mathbf{q}_{i}^{\top})$ and let $\mathbf{D}_{\mathbf{K}}$ be a diagonal matrix with entries of the form: $h(\mathbf{k}_{i}^{\top})$. Denote 
$\mathbf{B}=[\mathrm{K}(\frac{1}{d^{\frac{1}{4}}}\mathbf{q}_{i}^{\top},\frac{1}{d^{\frac{1}{4}}}\mathbf{k}_{j}^{\top})]_{i,j} \in \mathbb{R}^{L \times L}$. Denote by $\widehat{\mathbf{A}}$ and approximation of the attention matrix obtained from trigonometric orthogonal random features and by $\widehat{\mathbf{B}}$ an approximation of matrix $\mathbf{B}$ that those random features provide.
We rely on Theorem 3 from \citep{Lin2020DemystifyingOM}.
Note that we can apply it in our case, since for RBF kernels the corresponding functions $f_{i}$ satisfy $f_{1}(x) = \sin(x)$, $f_{2}(x) = \cos(x)$ (thus in particular are bounded). Also, it is not hard to observe (see for instance analysis in  Claim 1 from \citep{fourierapprox}) that we can take: $L_{f}=1$ (for $L_{f}$ as in Theorem 3 from \citep{Lin2020DemystifyingOM}). 
Using Theorem 3 from \citep{Lin2020DemystifyingOM}, we conclude that:
\begin{equation}
\|\widehat{\mathbf{B}}-\mathbf{B}\|_{\infty} \leq \delta    
\end{equation}
with any constant probability as long as
$m = \Omega(\frac{d}{\delta^{2}})\log(\frac{\sigma \cdot  \mathrm{diam}(\mathcal{M})}{\delta})$,
where $\sigma=\mathbb{E}[\omega^{\top}\omega]$ and $\mathcal{M}$ is the diameter of the smallest ball $\mathcal{M}$ containing all vectors of the form $\mathbf{z} = \frac{\mathbf{Q}_{i}}{d^{\frac{1}{4}}}-\frac{\mathbf{K}_{j}}{d^{\frac{1}{4}}}$. 
Since $\|\mathbf{Q}_{i}\|_{2}, \|\mathbf{K}_{j}\|_{2} \leq R$, we conclude that $\|\mathbf{z}\|_{2} \leq \frac{2R}{d^{\frac{1}{4}}}$ and thus one can take $\mathrm{diam}(\mathcal{M})=\frac{4R}{d^{\frac{1}{4}}}$.
We have:
\begin{equation}
\|\widehat{\mathbf{A}}-\mathbf{A}\|_{\infty} = \|\mathbf{D}_{\mathbf{Q}}(\widehat{\mathbf{B}}-\mathbf{B})\mathbf{D}_{\mathbf{K}}\|_{\infty} \leq 
\|\mathbf{D}_{\mathbf{Q}}\|_{\infty}
\|\widehat{\mathbf{B}}-\mathbf{B}\|_{\infty}
\|\mathbf{D}_{\mathbf{K}}\|_{\infty} \leq \delta g^{*}h^{*}
\end{equation}
Taking $\delta = \frac{\epsilon}{g^{*}h^{*}}$ completes the proof.
\end{proof}

\subsection{Discussion of Theorem \ref{thm:uniform}} \label{sec:ball}

As a consequence of Theorem~\ref{thm:uniform}, the number $m$ of random projections required to approximate the attention matrix within $\epsilon$ error is a function of data dimensionality $d$, the parameter $\epsilon$ and the radius $R$ of the ball within which the queries and keys live:
\[ m = \Psi(\epsilon, d, R). \]
The dependence on $d$ and $\epsilon$ is fairly easy to understand: with a larger dimensionality $d$ we need more random projeections (on the order of magnitude $d\log(d)$) to get an approximation within $\epsilon$ error. The dependence on $R$ means that the length of queries and keys cannot grow at a fixed $m$ if we want to retain the quality of the approximation.
In particular, this means that FAVOR cannot approximate hard attention on sequences of unlimited length with a fixed $m$. When the sequence length increases, even the standard attention requires longer and longer vectors to make the softmax concentrated enough to pick single elements. Nevertheless, as seen in our experiments, this limitation does not manifest itself in practice at the lengths we experimented with.

\end{document}